\documentclass[sigconf]{acmart}

\usepackage{multirow}
\usepackage{makecell}
\usepackage{xcolor}
\usepackage{colortbl}
\usepackage{subcaption}
\usepackage{float}

\usepackage{caption}  

\newcommand{\eg}{\emph{e.g., }}
\newcommand{\ie}{\emph{i.e., }}

\newcommand{\eat}[1]{}
\ifodd 1

\newcommand{\TODO}[1]{{\color{red}TODO:{#1}}}
\newcommand\beftext[1]{{\color[rgb]{0.5,0.5,0.5}{BEFORE:#1}}}
\else

\newcommand{\TODO}[1]{}
\newcommand{\beftext}[1]{}
\fi
\AtBeginDocument{%
  }

\setcopyright{acmlicensed}
\copyrightyear{2018}
\acmYear{2018}
\acmDOI{XXXXXXX.XXXXXXX}
\acmConference[Conference acronym 'XX]{Make sure to enter the correct
  conference title from your rights confirmation email}{June 03--05,
  2018}{Woodstock, NY}
\acmISBN{978-1-4503-XXXX-X/2018/06}

\begin{document}

\title{Physics-Informed Teleconnection-Aware Transformer for Global Subseasonal-to-Seasonal Forecasting}

\author{Tengfei Lyu}
\affiliation{%
  \institution{The Hong Kong University of Science and Technology (Guangzhou)}
  \city{Guangzhou}
  \country{China}}
\email{tlyu077@connect.hkust-gz.edu.cn}

\author{Weijia Zhang}
\affiliation{%
  \institution{The Hong Kong University of Science and Technology (Guangzhou)}
  \city{Guangzhou}
  \country{China}}
\email{wzhang411@connect.hkust-gz.edu.cn}

\author{Hao Liu}
\authornote{Corresponding author.}
\affiliation{%
  \institution{The Hong Kong University of Science and Technology (Guangzhou)}
  \institution{The Hong Kong University of Science and Technology}
  \city{Guangzhou \& Hong Kong}
  \country{China}}
\email{liuh@ust.hk}

\renewcommand{\shortauthors}{Tengfei Lyu, Jindong Han, and Hao Liu.}

\begin{abstract}
Subseasonal-to-seasonal (S2S) forecasting, which predicts climate conditions from several weeks to months in advance, represents a critical frontier for agricultural planning, energy management, and disaster preparedness. However, it remains one of the most challenging problems in atmospheric science, due to the chaotic dynamics of atmospheric systems and complex interactions across multiple scales. Current approaches often fail to explicitly model underlying physical processes and teleconnections that are crucial at S2S timescales. We introduce \textbf{TelePiT}, a novel deep learning architecture that enhances global S2S forecasting through integrated multi-scale physics and teleconnection awareness. Our approach consists of three key components: (1) Spherical Harmonic Embedding, which accurately encodes global atmospheric variables onto spherical geometry; (2) Multi-Scale Physics-Informed Neural ODE, which explicitly captures atmospheric physical processes across multiple learnable frequency bands; (3) Teleconnection-Aware Transformer, which models critical global climate interactions through explicitly modeling teleconnection patterns into the self-attention. 
Extensive experiments demonstrate that \textbf{TelePiT} significantly outperforms state-of-the-art data-driven baselines and operational numerical weather prediction systems across all forecast horizons, marking a significant advance toward reliable S2S forecasting. 
\end{abstract}



\keywords{Global Subseasonal-to-Seasonal Forecasting, Climate Prediction}

\maketitle

\section{Introduction}
\label{Introduction}

Subseasonal-to-seasonal (S2S) forecasting, predicting climate conditions from several weeks to months, occupies a critical frontier in atmospheric science~\cite{robertson2015improving,white2017potential,zhang2025unleashing}. 
By bridging the gap between short-range weather forecasts and seasonal climate projections, S2S outlooks are indispensable for proactive climate resilience and the efficient management of resources in agriculture, energy, and disaster preparedness~\cite{white2017potential,lyu2024autostf,zhang2024irregular,lyu2024nrformer}.
Compared to medium-range forecasting (up to 2 weeks)~\cite{bi2023accurate,lam2023learning}, S2S forecasting encounters significantly greater challenges due to the inherent chaotic dynamics of atmospheric systems and intricate interactions across various scales of Earth's climate system. 
Beyond initial atmospheric conditions, accurate S2S forecasting requires consideration of additional complex factors, such as more complicated and highly coupled physical processes, teleconnections, and boundary interactions. 
These complexities often render existing medium-range forecasting models inadequate for S2S forecasting, leading to a significant decline in predictive skill at extended lead times~\cite{nathaniel2024chaosbench}.

Recent models have delivered promising improvements in both regional and global S2S forecasting tasks. While regional models~\cite{mouatadid2023subseasonalclimateusa,he2022learning,hwang2019improving} have significantly advanced localized forecasting capabilities, they often overlook the crucial influence of long-range correlations in atmospheric conditions, which become increasingly significant at S2S timescales. To achieve global forecasting, Chen et al.\cite{chen2024machine} present FuXi-S2S, an encoder-decoder framework trained on global ERA5 reanalysis data~\cite{hersbach2020era5}, notably enhancing global precipitation forecast. Furthermore, Liu et al.~\cite{liucirt} introduce CirT, a geometry-inspired transformer that effectively models Earth’s spherical structure for S2S forecasting.

However, current methods still exhibit critical research gaps. Firstly, existing works overlook explicit modeling of the underlying physical processes that dictate multi-scale climate dynamics. This oversight stems from reliance on fixed-frequency modeling and insufficient integration of physics-informed mechanisms, which are essential for accurate modeling of fundamental atmospheric processes such as advection, diffusion, and external forcing~\cite{raissi2019physics,karniadakis2021physics,beucler2021enforcing}.
Secondly, many models fail to explicitly represent essential teleconnections (\ie~significant long-range correlations between distant regions), exemplified by climate modes such as North Atlantic Oscillation and Madden–Julian Oscillation. These teleconnections substantially influence global atmospheric conditions and play pivotal roles in determining regional and global climate variability on S2S scales~\cite{trenberth1998atmospheric,ham2019deep}.

To address these limitations, we propose \textbf{TelePiT}, a novel deep learning architecture specifically designed to enhance global S2S forecasting through integrated multi-scale physics and teleconnection awareness. Our approach includes three major components: (1) Spherical Harmonic Embedding, which accurately encodes global atmospheric variables onto spherical geometry, preserving the inherent spatial continuity and enabling the model to learn location-dependent atmospheric processes; (2) Multi-Scale Physics-Informed Ordinary Differential Equation (ODE), explicitly capturing atmospheric physical processes across multiple learnable atmospheric frequency bands by integrating physics-driven constraints within a neural ODE; (3) Teleconnection-Aware Transformer, which leverages a specially designed teleconnection-aware self-attention to model teleconnection patterns explicitly, effectively capturing critical global climate interactions.

Our main contributions are as follows: 
Firstly, we pioneer the explicit integration of physical constraints across multi-scale climate dynamics via Multi-Scale Physics-Informed ODE, simultaneously capturing atmospheric processes operating at different frequencies while preserving fundamental physical laws.
Secondly, we introduce the first explicit modeling of teleconnections through a novel Teleconnection-Aware Transformer, which captures the complex, non-local correlations between distant geographic regions that critically influence forecasting in S2S scales.
Finally, extensive experiments demonstrate that \textbf{TelePiT} significantly outperforms both state-of-the-art data-driven baselines and operational numerical weather prediction systems, consistently delivering enhanced forecasting accuracy and reliability at global scales.

\section{Preliminaries}
\label{sec:Preliminaries}
In this work, we represent the global weather condition on the day $t$ as $\mathbf{X}_t\in \mathbb{R}^{C\times H\times W}$, where $H$ and $W$ correspond to the spatial dimensions along latitude and longitude respectively, and $C$ denotes the number of weather variables (\eg temperature, geopotential height). The latitude-longitude grid is defined as $\mathcal{G}\in \mathbb{R}^{H\times W}$, where each grid point $\mathcal{G}_{h,w}=(\theta_h,\phi_w) \in [-90^{\circ},90^{\circ}]\times[-180^{\circ},180^{\circ}]$.

\textbf{S2S Forecasting.}
Following previous works \cite{chen2024machine,liucirt}, given the initial condition $\mathbf{X}_{t_1}$, the goal is to predict the mean fields of these variables over two future multi-week windows: weeks 3–4 (days 15–28 ) and weeks 5–6 (days 29–42) of the forecast period, as shown in the following:
\begin{equation}
\mathcal{F}_{\Theta}(\mathbf{X}_{t_1})
\rightarrow
(\hat{\mathbf{Y}}_{t_{15}:t_{28}}, \hat{\mathbf{Y}}_{t_{29}:t_{42}})   
\end{equation}
where  $\hat{\mathbf{Y}}_{t_{15}:t_{28}} \in \mathbb{R}^{C\times H\times W}$, and $\hat{\mathbf{Y}}_{t_{29}:t_{42}} \in \mathbb{R}^{C\times H\times W}$ denote the average value over weeks 3-4 and weeks 5-6 respectively. $\mathcal{F}_{\Theta}$ is forecasting model.
S2S forecasting transcends the deterministic chaos limit of weather forecasting (2 weeks), thereby requiring advanced neural architecture to bridge the fundamental gap between chaotic weather dynamics and more predictable climate signals.

\begin{figure*}[t]
    \centering
    \includegraphics[width=1\linewidth]{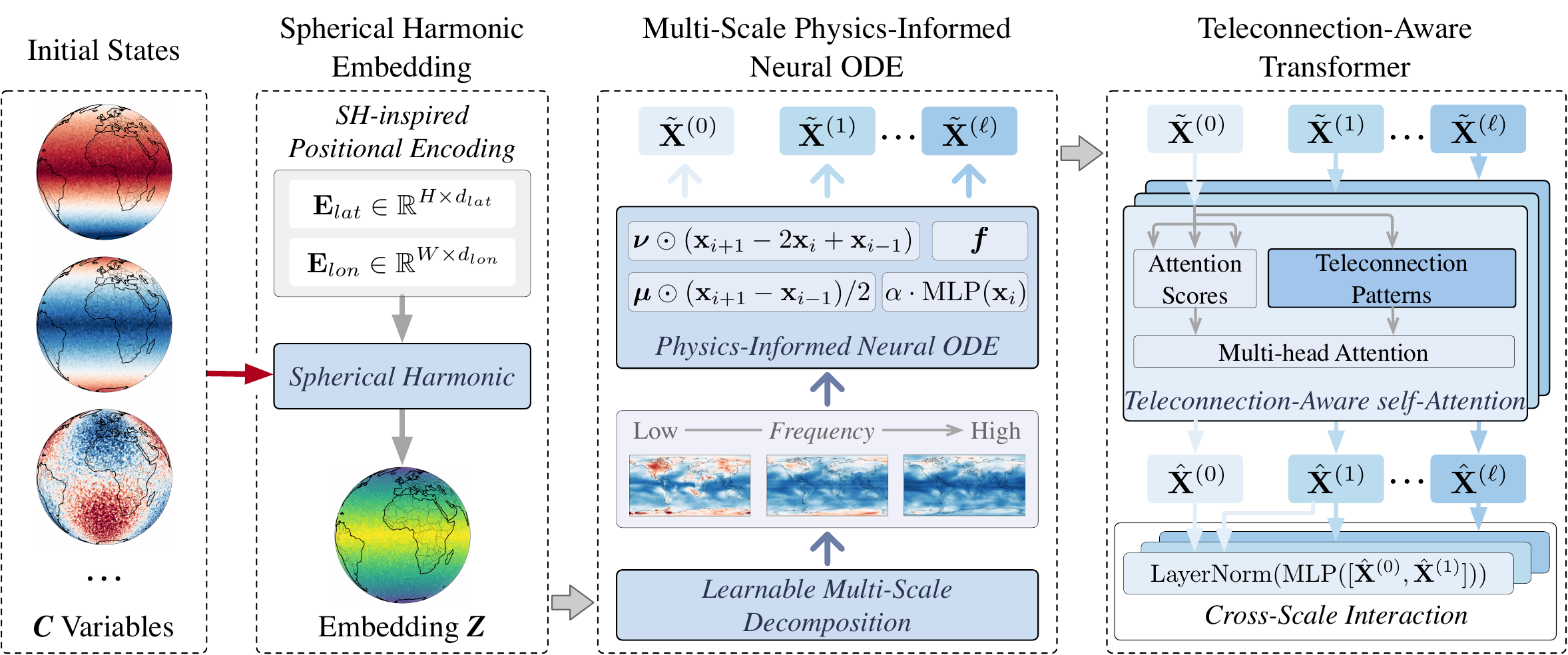}
    \vspace{-0.4cm}
    \caption{The framework of TelePiT for global subseasonal-to-seasonal forecasting.}
    \label{fig:framework}
    \vspace{-0.2cm}
\end{figure*}

\section{Methods}
\label{sec:methods}
Figure \ref{fig:framework} presents the overall architecture of \textbf{TelePiT}, our physics-informed teleconnection-aware transformer for global S2S forecasting.
The model consists of three key components designed to address the fundamental challenges of extending predictability beyond the deterministic chaos limit. 

\subsection{Spherical Harmonic Embedding}
\label{sec:embedding}

To accurately represent global atmospheric variables on the sphere, we introduce a learnable spherical harmonic-inspired positional embedding. 
For a grid point at latitude $\theta_i \in [-\frac{\pi}{2}, \frac{\pi}{2}]$ and longitude $\phi_j \in [-\pi, \pi]$, we construct positional encodings using frequency-scaled sinusoidal functions:
\begin{equation}
\mathbf{E}_{lat}(i,2k) = \sin\big((k+1)\theta_i\big), \quad
\mathbf{E}_{lat}(i,2k+1) = \cos\big((k+1)\theta_i\big),
\label{eq:lat-embed}
\end{equation}
\begin{equation}
\mathbf{E}_{lon}(j,2q) = \sin\big((q+1)\phi_j\big), \quad
\mathbf{E}_{lon}(j,2q+1) = \cos\big((q+1)\phi_j\big),
\label{eq:lon-embed}
\end{equation}
where $k \in \{0,1,...,\frac{d_{lat}}{2}-1\}$ and $q \in \{0,1,...,\frac{d_{lon}}{2}-1\}$. These encodings are initialized as learnable parameters $\mathbf{E}_{lat} \in \mathbb{R}^{H \times d_{lat}}$ and $\mathbf{E}_{lon} \in \mathbb{R}^{W \times d_{lon}}$ where $d_{lat} + d_{lon} = D_{emb}$. 
The sinusoidal functions provide a natural way to represent the periodic structure of spherical coordinates, particularly for longitude which wraps around the globe \cite{atkinson2012spherical,cohen2018spherical}.

We process the latitude dimension as a sequence of tokens while aggregating information along the longitude dimension. For each latitude $i$, we compute a zonal average:
$\mathbf{u}_i = \frac{1}{W}\sum_{j=1}^W X_{i,j} \in \mathbb{R}^{C}$.
This zonal averaging operation is physically justified as atmospheric patterns often exhibit stronger correlations along latitude bands, particularly for large-scale circulation features \cite{baldwin2003stratospheric, kidston2015stratospheric}. The features are then projected to a higher-dimensional latent space through a learnable transformation:
$\mathbf{h}_i = \mathbf{W}\mathbf{u}_i + \mathbf{b} \in \mathbb{R}^{D_{emb}}$,
where $\mathbf{W} \in \mathbb{R}^{D_{emb} \times C}$ and $\mathbf{b} \in \mathbb{R}^{D_{emb}}$ are learnable parameters.

To incorporate spherical harmonic-inspired positional embedding, we combine the latitude embedding with the global average of longitude embeddings:
\begin{equation}
\mathbf{p}_i^{lat} = \mathbf{E}_{lat}(i,:) \in \mathbb{R}^{d_{lat}}, \enspace
\mathbf{p}^{lon} = \frac{1}{W}\sum_{j=1}^W \mathbf{E}_{lon}(j,:) \in \mathbb{R}^{d_{lon}}, 
\end{equation}
\begin{equation}
\mathbf{p}_i = [\mathbf{p}_i^{lat}, \mathbf{p}^{lon}] \in \mathbb{R}^{D_{emb}}.
\end{equation}
The final embedding is computed with the complete embedded sequence obtained by stacking:
\begin{equation}
\mathbf{z}_i = \mathbf{W}_{proj}(\mathbf{h}_i + \mathbf{p}_i) + \mathbf{b}_{proj} \in \mathbb{R}^{D_{emb}},
\end{equation}
\begin{equation}
\mathbf{Z} = [\mathbf{z}_1, \mathbf{z}_2, \ldots, \mathbf{z}_H] \in \mathbb{R}^{H \times D_{emb}}.
\end{equation}
The theoretical foundation of this connects to spherical function representation theory. Mathematically, our zonal averaging and embedding method implements learnable approximation to spherical harmonic decomposition. Traditionally, a function $f(\theta,\phi)$ on sphere can be decomposed as:
\begin{equation}
f(\theta,\phi) = \sum_{\ell=0}^{\infty}\sum_{m=-\ell}^{\ell} a_{\ell m} Y_{\ell}^m(\theta,\phi),
\end{equation}
where $Y_{\ell}^m(\theta,\phi)$ are the spherical harmonic functions and $a_{\ell m}$ are expansion coefficients \cite{atkinson2012spherical}. In particular, when $m=0$, spherical harmonics reduce to zonal harmonics that depend solely on latitude $\theta$, which precisely corresponds to our zonal averaging operation. This approach extracts a sequence of length $H$ as input tokens for subsequent modules, preserving the inductive bias of spherical geometry while enabling the model to adaptively learn optimal representations for atmospheric variables.

\subsection{Multi-Scale Physics-Informed Neural ODE}
\label{sec:wavelet_ode}

\subsubsection{Learnable Multi-Scale Decomposition}
\label{sec:wavelet}
Atmospheric processes span multiple scales that interact in complex ways, creating a fundamental challenge for S2S forecasting. 
For example, when predicting Northern Hemisphere winter circulation, stratospheric polar vortex dynamics (planetary-scale phenomena) exhibit significant downward influence on tropospheric weather patterns (regional-scale outcomes) through scale-bridging mechanisms operating over 2-3 week \cite{baldwin2001stratospheric, kidston2015stratospheric}.
These cross-scale interactions, also documented in observational studies \cite{butler2019predictability,vitart2014evolution}, embody critical physical processes that must be explicitly resolved to overcome current limitations.

To explicitly capture this inherent multi-scale structure, we implement an $L$-level learnable wavelet decomposition $\mathcal{W}$ operating along the latitude dimension of $\mathbf{Z}$. At each level, the current signal is split into a coarse approximation and a detail component, mirroring the discrete wavelet transform but with learnable filters. Formally, let $\mathbf{A}_0 = \mathbf{Z} \in \mathbb{R}^{H\times D_{emb}}$ be the input at the finest scale. For each level $\ell=1,2,\cdots,L$, we apply a neural transform:
\begin{equation}
[\mathbf{A}_\ell, \mathbf{D}_\ell] = \textit{split}\Big(\text{MLP}_\ell(\mathbf{A}_{\ell-1}), \text{dim}=2\Big),
\end{equation}
where $\text{MLP}_\ell: \mathbb{R}^{D_{emb}} \rightarrow \mathbb{R}^{2D_{emb}}$ is a two-layer perceptron with GELU activation, the $\textit{split}(\cdot, \text{dim}=2)$ operation divides the output features into two equal halves. This yields an \emph{approximation} $\mathbf{A}_\ell \in \mathbb{R}^{H\times D_{emb}}$ and a \emph{detail} $\mathbf{D}_\ell \in \mathbb{R}^{H\times D_{emb}}$. The $\mathbf{A}_\ell$ carries the smoother (low-frequency) information, while $\mathbf{D}_\ell$ captures the variations (high-frequency details) removed at level $\ell$.
After $L$ levels, we obtain a final coarse approximation $\mathbf{A}_L$ and detail components $\{\mathbf{D}_\ell\}_{\ell=1}^L$. For convenience, we denote the set of multi-scale features as:
\begin{equation}
\mathcal{X}=\{\mathbf{X}^{(0)}, \mathbf{X}^{(1)}, \dots, \mathbf{X}^{(L)}\},
\end{equation}
where $\mathbf{X}^{(0)} := \mathbf{A}_L$ represents the lowest-frequency band, and $\mathbf{X}^{(\ell)} := \mathbf{D}_{L+1-\ell}$ for $\ell=1,\cdots,L$ (ordering from low to high frequency). Each $\mathbf{X}^{(\ell)} \in \mathbb{R}^{H\times D_{emb}}$ maintains the same dimensions as the original embedding.

This multi-scale decomposition aligns with the established understanding in atmospheric science that different physical processes dominate at different scales \cite{hoskins1985use}, enabling our learnable approach to adaptively determine the optimal decomposition basis for S2S forecasting.

\subsubsection{Physics-Informed Neural ODE}
\label{sec:ode}
To incorporate physical prior knowledge of atmospheric dynamics \cite{kashinath2021physics}, we evolve each frequency band using a physics-informed neural ODE. Drawing from fluid dynamics principles, we formulate an ODE in the latent space \cite{vermaclimode} that combines fundamental atmospheric processes with data-driven corrections to handle complex, unresolved dynamics.

For each multi-scale component $\mathbf{X}^{(\ell)} = [\mathbf{x}_1, \mathbf{x}_2, \cdots, \mathbf{x}_H] \in \mathbb{R}^{H\times D_{emb}}$, we define an ODE that evolves the latent state through time:
\begin{equation}\label{eq:ode}
\frac{d \mathbf{x}_i}{dt} = \gamma \cdot \tanh(\mathbf{R}_i)
\end{equation}
\begin{equation}\label{eq:composite}
{\small
\mathbf{R}_i = \underbrace{\boldsymbol{\nu}\odot(\mathbf{x}_{i+1} - 2\mathbf{x}_i + \mathbf{x}_{i-1})}_{\text{diffusion}} + 
\underbrace{\boldsymbol{\mu}\odot\frac{\mathbf{x}_{i+1} - \mathbf{x}_{i-1}}{2}}_{\text{advection}} + 
\underbrace{\boldsymbol{f}}_{\text{forcing}} + 
\underbrace{\alpha \cdot \text{MLP}(\mathbf{x}_i)}_{\text{neural correction}}
}
\end{equation}
This equation incorporates four physically meaningful components:
\begin{itemize}
    \item {\textbf{Diffusion term}}: $\boldsymbol{\nu}\odot(\mathbf{x}_{i+1} - 2\mathbf{x}_i + \mathbf{x}_{i-1})$ approximates the second-order diffusion operator along latitude, representing the meridional diffusion of atmospheric properties. The learnable parameter $\boldsymbol{\nu} \in \mathbb{R}^{D_{emb}}$ modulates diffusion strength for each feature dimension.
    \item \textbf{Advection term}: $\boldsymbol{\mu}\odot\frac{\mathbf{x}_{i+1} - \mathbf{x}_{i-1}}{2}$ approximates the first-order spatial derivative, representing meridional transport by mean flow. The parameter $\boldsymbol{\mu} \in \mathbb{R}^{D_{emb}}$ controls advection intensity.
    \item \textbf{Forcing term}: $\boldsymbol{f} \in \mathbb{R}^{D_{emb}}$ represents external forcing, analogous to radiative and diabatic processes in the climate system.
    \item \textbf{Neural correction}: $\alpha \cdot \text{MLP}(\mathbf{x}_i)$ learns complex, nonlinear dynamics not captured by simplified physical terms. Parameter $\alpha$ balances the physical terms with data-driven corrections.
\end{itemize}
For numerical stability and physical realism, we apply boundary conditions $\mathbf{x}_0 = \mathbf{x}_{H+1} = 0$ (zero-padding at the poles) and constrain the rate of change using a hyperbolic tangent function. $\gamma$ is a small scaling factor that moderates the rate of change to prevent numerical instabilities.
We employ Euler with a fixed step size $\Delta t$:
\begin{equation}
\mathbf{x}_i(t+\Delta t) = \mathbf{x}_i(t) + \Delta t \cdot \frac{d\mathbf{x}_i}{dt}\Big|_t.
\end{equation}
We evolve each frequency band $\mathbf{X}^{(\ell)}$ independently from time $t=0$ to $t=T$ (where $T$ represents one forecast step in latent space), yielding evolved representations $\tilde{\mathbf{X}}^{(\ell)}$. This scale-dependent evolution reflects physical reality that atmospheric processes operate differently across scales, diffusion dominates at smaller scales while advection plays a more significant role at larger scales \cite{vallis2017atmospheric}.

\subsection{Teleconnection-Aware Transformer}
\label{sec:transformer}

\subsubsection{Teleconnection-Aware Self-Attention}
Teleconnections, statistically significant correlations between weather and climate events at distant locations, such as the North Atlantic Oscillation (NAO) and Madden-Julian Oscillation (MJO), are fundamental for accurate S2S forecasting \cite{wallace1981teleconnections}. 
We encode these global patterns through a novel teleconnection-aware self-attention mechanism.

For each frequency band $\tilde{\mathbf{X}}^{(\ell)} \in \mathbb{R}^{H\times D_{emb}}$, we first compute standard queries, keys, and values: $\mathbf{Q} = \tilde{\mathbf{X}}^{(\ell)}\mathbf{W}^Q$,  $\mathbf{K} = \tilde{\mathbf{X}}^{(\ell)}\mathbf{W}^K$, $\mathbf{V} = \tilde{\mathbf{X}}^{(\ell)}\mathbf{W}^V$, where $\mathbf{W}^Q, \mathbf{W}^K, \mathbf{W}^V \in \mathbb{R}^{D_{emb} \times d_k}$ are learnable projections and $d_k = D_{emb}/N_h$ for $N_h$ heads.

Our key innovation is introducing the teleconnection patterns to explicitly modulate attention mechanisms. We compute a global teleconnection vector that captures dominant climate modes:
\begin{equation}
\mathbf{c} = \sum_{j=1}^{n_p} \omega_j \mathbf{P}_j, \quad \text{with} \quad \boldsymbol{\omega} = \text{softmax}(\mathbf{\bar{x}}\mathbf{W}^p) \in \mathbb{R}^{n_p},
\label{eq:global-pattern}
\end{equation}
where $\mathbf{\bar{x}} = \frac{1}{H}\sum_{i=1}^H \tilde{\mathbf{x}}_i$ is the global atmosphere state, $\mathbf{W}^p \in \mathbb{R}^{D_{emb}\times n_p}$ projects this state onto $n_p$ learnable coefficients, and $\mathbf{P}_1,\ldots,\mathbf{P}_{n_p} \in \mathbb{R}^{D_{emb}}$ are vectors that represent the teleconnection patterns, which typically can be obtained by manual extraction or automatic learning from data.
We project $\mathbf{c}$ into query space to create a teleconnection query $\mathbf{q}^{tel} = \mathbf{c}\mathbf{W}^Q \in \mathbb{R}^{d_k}$ and compute attention scores $b_j$ between this query and all latitude keys. These scores are added as a bias to the standard attention logits:
\begin{equation}
b_j = \frac{1}{\sqrt{d_k}}(\mathbf{q}^{tel} \cdot \mathbf{K}_j), \enspace \tilde{A}_{ij} = \frac{1}{\sqrt{d_k}}(\mathbf{Q}_i \cdot \mathbf{K}_j) + \lambda \cdot b_j,
\label{eq:tel-attn}
\end{equation}
where $\lambda$ is a hyperparameter controlling teleconnection influence.
This mechanism biases attention toward active global teleconnection patterns. For instance, during an MJO event, the model can emphasize connections between the tropical Indian Ocean/western Pacific and remote affected regions \cite{zhang2005madden}, empowering teleconnection-aware forecasting.

\subsubsection{Cross-Scale Interaction}
Atmospheric processes operate across multiple scales with complex interactions \cite{majda2007multiscale}. We enable cross-scale information flow through a hierarchical fusion mechanism. First, each frequency band is processed by a teleconnection-aware Transformer: $\hat{\mathbf{X}}^{(\ell)} = \textit{TA-Transformer}_\ell(\tilde{\mathbf{X}}^{(\ell)})$. 
We then implement cross-scale fusion from larger to smaller scales: 
\begin{equation}
\hat{\mathbf{X}}^{(\ell+1)} = \text{LayerNorm}(\text{MLP}([\hat{\mathbf{X}}^{(\ell)}, \hat{\mathbf{X}}^{(\ell+1)}])), 
\end{equation}
where $[\cdot,\cdot]$ denotes concatenation. 
Finally, we integrate information across all scales: 
\begin{equation}
    \mathbf{Z}_{final} = \frac{1}{L+1}\sum_{\ell=0}^{L} \hat{\mathbf{X}}^{(\ell)},
\end{equation}
yielding a representation that captures atmospheric dynamics across all scales.

\subsection{S2S Forecasting}
\label{sec:prediction}

For each latitude $i$, we process its feature vector  as follows:
\begin{equation}
\mathbf{z}^{final}_i = \mathbf{Z}_{final}(i,:) \in \mathbb{R}^{D_{emb}},
\end{equation}
\begin{equation}
\mathbf{y}_i = \mathbf{W}^{(2)}\text{GELU}(\mathbf{W}^{(1)}\mathbf{z}^{final}_i + \mathbf{b}^{(1)}) + \mathbf{b}^{(2)}, 
\end{equation}
where $\mathbf{W}^{(1)} \in \mathbb{R}^{D_{emb}\times 2D_{emb}}$ and $\mathbf{W}^{(2)} \in \mathbb{R}^{2D_{emb}\times (2CW)}$.
The output $\mathbf{y}_i \in \mathbb{R}^{2CW}$ is reshaped into a tensor representing predictions for both horizons at all longitudes for this latitude. Stacking these across all latitudes yields the final forecast fields $\hat{\mathbf{Y}}^{(1)}_{t_{15}:t_{28}}, \hat{\mathbf{Y}}^{(2)}_{t_{29}:t_{42}} \in \mathbb{R}^{C \times H \times W}$ for the two target periods.
We train the model using the loss function:
\begin{equation}
\mathcal{L}
= \frac{1}{2 C H W}
\Bigl(
    \left\lVert \hat{\mathbf{Y}}^{(1)}_{t_{15}:t_{28}} - \mathbf{Y}^{(1)}_{t_{15}:t_{28}} \right\rVert_2^{2}
    + 
    \left\lVert \hat{\mathbf{Y}}^{(2)}_{t_{29}:t_{42}} - \mathbf{Y}^{(2)}_{t_{29}:t_{42}} \right\rVert_2^{2}
\Bigr).
\end{equation}

To establish the mathematical foundations of our \textbf{TelePiT} and ensure its theoretical rigor and consistency, we provide a comprehensive theoretical analysis in Appendix \ref{appendix_sec:Theoretical Analysis}.

\begin{table*}[t] 
\let\oldeverydisplay\everydisplay
\let\oldeverymath\everymath
\everydisplay{}
\everymath{}
\definecolor{lightblue}{RGB}{255, 255, 255}  
\definecolor{lightred}{RGB}{255, 255, 255}   
\setlength{\tabcolsep}{5.3pt} 
\renewcommand{\arraystretch}{0.65} 
\caption{Performance comparison of \textbf{TelePiT} against data-driven models. 
Lower RMSE ($\downarrow$) indicates better accuracy, while higher ACC ($\uparrow$) indicates better correlation with ground truth. 
}
\vspace{-0.3cm}
\label{tab:main}
  \begin{tabular}{cc|cccccc|cccccc}
    \toprule
    \multirow{2}{*}{\textbf{}} &\multirow{2}{*}{\textbf{Variable}} & \multicolumn{6}{c|}{\textbf{RMSE ($\downarrow$)}} & \multicolumn{6}{c}{\textbf{ACC ($\uparrow$)}}\\ 
    &  &  FCNV2 &  GC & Pangu & ClimaX & CirT & \textbf{TelePiT} &  FCN-V2 &  GC & Pangu & ClimaX & CirT & \textbf{TelePiT}\\
    
    \midrule
    \multirow{7}{*}{\rotatebox{90}{\textbf{Weeks 3-4}}} & \multicolumn{1}{>{\columncolor{lightblue}}c|}{\text{ z500 ($gpm$) }}  
    &  \multicolumn{1}{>{\columncolor{lightblue}}c}{62.454}   &  \multicolumn{1}{>{\columncolor{lightblue}}c}{63.119}   & \multicolumn{1}{>{\columncolor{lightblue}}c}{65.944}                  
    & \multicolumn{1}{>{\columncolor{lightblue}}c}{63.243}    & \multicolumn{1}{>{\columncolor{lightblue}}c}{53.807}    & \multicolumn{1}{>{\columncolor{lightblue}}c|}{\textbf{48.671}}        
    & \multicolumn{1}{>{\columncolor{lightblue}}c}{0.968}     & \multicolumn{1}{>{\columncolor{lightblue}}c}{0.973}     & \multicolumn{1}{>{\columncolor{lightblue}}c}{0.971}                   
    & \multicolumn{1}{>{\columncolor{lightblue}}c}{0.975}     & \multicolumn{1}{>{\columncolor{lightblue}}c}{0.981}     & \multicolumn{1}{>{\columncolor{lightblue}}c}{\textbf{0.985}}  \\
    
    \multicolumn{1}{c}{} & \multicolumn{1}{>{\columncolor{lightred}}c|}{\text{ z850 ($gpm$) }}   
    & \multicolumn{1}{>{\columncolor{lightred}}c}{41.555}     & \multicolumn{1}{>{\columncolor{lightred}}c}{42.706}     & \multicolumn{1}{>{\columncolor{lightred}}c}{43.007}                   
    & \multicolumn{1}{>{\columncolor{lightred}}c}{39.679}     & \multicolumn{1}{>{\columncolor{lightred}}c}{34.006}     & \multicolumn{1}{>{\columncolor{lightred}}c|}{{\textbf{31.082}}}       
    & \multicolumn{1}{>{\columncolor{lightred}}c}{0.923}      & \multicolumn{1}{>{\columncolor{lightred}}c}{0.929}      & \multicolumn{1}{>{\columncolor{lightred}}c}{0.928}                    
    & \multicolumn{1}{>{\columncolor{lightred}}c}{0.942}      & \multicolumn{1}{>{\columncolor{lightred}}c}{0.955}      & \multicolumn{1}{>{\columncolor{lightred}}c}{\textbf{0.963}}  \\
    
    & \multicolumn{1}{>{\columncolor{lightblue}}c|}{\text{ t500 ($K$) }}          
    & \multicolumn{1}{>{\columncolor{lightblue}}c}{2.065}     & \multicolumn{1}{>{\columncolor{lightblue}}c}{2.147}     & \multicolumn{1}{>{\columncolor{lightblue}}c}{2.254}    
    & \multicolumn{1}{>{\columncolor{lightblue}}c}{2.208}     & \multicolumn{1}{>{\columncolor{lightblue}}c}{1.797}     & \multicolumn{1}{>{\columncolor{lightblue}}c|}{\textbf{1.684}}       
    & \multicolumn{1}{>{\columncolor{lightblue}}c}{0.979}     & \multicolumn{1}{>{\columncolor{lightblue}}c}{0.980}     & \multicolumn{1}{>{\columncolor{lightblue}}c}{0.979}   
    & \multicolumn{1}{>{\columncolor{lightblue}}c}{0.980}     & \multicolumn{1}{>{\columncolor{lightblue}}c}{0.986}     & \multicolumn{1}{>{\columncolor{lightblue}}c}{\textbf{0.988}}  \\
    
    \multicolumn{1}{c}{} & \multicolumn{1}{>{\columncolor{lightred}}c|}{\text{ t850 ($K$) }}          
    & \multicolumn{1}{>{\columncolor{lightred}}c}{2.429}      & \multicolumn{1}{>{\columncolor{lightred}}c}{2.414}      & \multicolumn{1}{>{\columncolor{lightred}}c}{2.630}    
    & \multicolumn{1}{>{\columncolor{lightred}}c}{2.759}      & \multicolumn{1}{>{\columncolor{lightred}}c}{2.051}      & \multicolumn{1}{>{\columncolor{lightred}}c|}{\textbf{1.873}}       
    & \multicolumn{1}{>{\columncolor{lightred}}c}{0.978}      & \multicolumn{1}{>{\columncolor{lightred}}c}{0.982}      & \multicolumn{1}{>{\columncolor{lightred}}c}{0.980}   
    & \multicolumn{1}{>{\columncolor{lightred}}c}{0.978}      & \multicolumn{1}{>{\columncolor{lightred}}c}{0.986}      & \multicolumn{1}{>{\columncolor{lightred}}c}{\textbf{0.990}}  \\
    
    & \multicolumn{1}{>{\columncolor{lightblue}}c|}{\text{ t2m ($K$) }}           
    & \multicolumn{1}{>{\columncolor{lightblue}}c}{--}        & \multicolumn{1}{>{\columncolor{lightblue}}c}{--}        & \multicolumn{1}{>{\columncolor{lightblue}}c}{--}       
    & \multicolumn{1}{>{\columncolor{lightblue}}c}{60.894}    & \multicolumn{1}{>{\columncolor{lightblue}}c}{28.526}    & \multicolumn{1}{>{\columncolor{lightblue}}c|}{\textbf{12.057}}      
    & \multicolumn{1}{>{\columncolor{lightblue}}c}{--}        & \multicolumn{1}{>{\columncolor{lightblue}}c}{--}        & \multicolumn{1}{>{\columncolor{lightblue}}c}{--}      
    & \multicolumn{1}{>{\columncolor{lightblue}}c}{0.889}     & \multicolumn{1}{>{\columncolor{lightblue}}c}{0.977}     & \multicolumn{1}{>{\columncolor{lightblue}}c}{\textbf{0.996}}  \\
    
    \multicolumn{1}{c}{}& \multicolumn{1}{>{\columncolor{lightred}}c|}{\text{ u10 (m/s) }}           
    & \multicolumn{1}{>{\columncolor{lightred}}c}{4.217}      & \multicolumn{1}{>{\columncolor{lightred}}c}{--}         & \multicolumn{1}{>{\columncolor{lightred}}c}{4.000}    
    & \multicolumn{1}{>{\columncolor{lightred}}c}{0.829}      & \multicolumn{1}{>{\columncolor{lightred}}c}{0.567}      & \multicolumn{1}{>{\columncolor{lightred}}c|}{{\textbf{0.491}}}       
    & \multicolumn{1}{>{\columncolor{lightred}}c}{0.322}      & \multicolumn{1}{>{\columncolor{lightred}}c}{--}         & \multicolumn{1}{>{\columncolor{lightred}}c}{0.335}   
    & \multicolumn{1}{>{\columncolor{lightred}}c}{0.784}      & \multicolumn{1}{>{\columncolor{lightred}}c}{0.907}      & \multicolumn{1}{>{\columncolor{lightred}}c}{\textbf{0.928}}  \\
    
    & \multicolumn{1}{>{\columncolor{lightblue}}c|}{\text{ v10 (m/s) }}           
    & \multicolumn{1}{>{\columncolor{lightblue}}c}{2.403}     & \multicolumn{1}{>{\columncolor{lightblue}}c}{--}        & \multicolumn{1}{>{\columncolor{lightblue}}c}{2.322}    
    & \multicolumn{1}{>{\columncolor{lightblue}}c}{0.773}     & \multicolumn{1}{>{\columncolor{lightblue}}c}{0.531}     & \multicolumn{1}{>{\columncolor{lightblue}}c|}{\textbf{0.464}}       
    & \multicolumn{1}{>{\columncolor{lightblue}}c}{0.486}     & \multicolumn{1}{>{\columncolor{lightblue}}c}{--}        & \multicolumn{1}{>{\columncolor{lightblue}}c}{0.488}  
    & \multicolumn{1}{>{\columncolor{lightblue}}c}{0.791}     & \multicolumn{1}{>{\columncolor{lightblue}}c}{0.916}     & \multicolumn{1}{>{\columncolor{lightblue}}c}{\textbf{0.934}}  \\
    
    \midrule
    
    \multirow{7}{*}{\rotatebox{90}{\textbf{Weeks 5-6}}}
    & \multicolumn{1}{>{\columncolor{lightblue}}c|}{\text{ z500 ($gpm$) }}  
    & \multicolumn{1}{>{\columncolor{lightblue}}c}{65.604}    & \multicolumn{1}{>{\columncolor{lightblue}}c}{--}        & \multicolumn{1}{>{\columncolor{lightblue}}c}{75.838}   
    & \multicolumn{1}{>{\columncolor{lightblue}}c}{64.433}    & \multicolumn{1}{>{\columncolor{lightblue}}c}{52.969}    & \multicolumn{1}{>{\columncolor{lightblue}}c|}{\textbf{47.761}}      
    & \multicolumn{1}{>{\columncolor{lightblue}}c}{0.966}     & \multicolumn{1}{>{\columncolor{lightblue}}c}{--}        & \multicolumn{1}{>{\columncolor{lightblue}}c}{0.965}   
    & \multicolumn{1}{>{\columncolor{lightblue}}c}{0.974}     & \multicolumn{1}{>{\columncolor{lightblue}}c}{0.981}     & \multicolumn{1}{>{\columncolor{lightblue}}c}{\textbf{0.985}}  \\
    
    \multicolumn{1}{c}{}& \multicolumn{1}{>{\columncolor{lightred}}c|}{\text{ z850 ($gpm$) }}  
    & \multicolumn{1}{>{\columncolor{lightred}}c}{43.327}     & \multicolumn{1}{>{\columncolor{lightred}}c}{--}         & \multicolumn{1}{>{\columncolor{lightred}}c}{47.317}   
    & \multicolumn{1}{>{\columncolor{lightred}}c}{39.831}     & \multicolumn{1}{>{\columncolor{lightred}}c}{33.502}     & \multicolumn{1}{>{\columncolor{lightred}}c|}{{\textbf{30.405}}}      
    & \multicolumn{1}{>{\columncolor{lightred}}c}{0.919}      & \multicolumn{1}{>{\columncolor{lightred}}c}{--}         & \multicolumn{1}{>{\columncolor{lightred}}c}{0.915}   
    & \multicolumn{1}{>{\columncolor{lightred}}c}{0.942}      & \multicolumn{1}{>{\columncolor{lightred}}c}{0.956}      & \multicolumn{1}{>{\columncolor{lightred}}c}{\textbf{0.964}}  \\
    
    & \multicolumn{1}{>{\columncolor{lightblue}}c|}{\text{ t500 ($K$) }}          
    & \multicolumn{1}{>{\columncolor{lightblue}}c}{2.196}     & \multicolumn{1}{>{\columncolor{lightblue}}c}{--}        & \multicolumn{1}{>{\columncolor{lightblue}}c}{2.784}    
    & \multicolumn{1}{>{\columncolor{lightblue}}c}{2.281}     & \multicolumn{1}{>{\columncolor{lightblue}}c}{1.779}     & \multicolumn{1}{>{\columncolor{lightblue}}c|}{\textbf{1.698}}       
    & \multicolumn{1}{>{\columncolor{lightblue}}c}{0.977}     & \multicolumn{1}{>{\columncolor{lightblue}}c}{--}        & \multicolumn{1}{>{\columncolor{lightblue}}c}{0.971}   
    & \multicolumn{1}{>{\columncolor{lightblue}}c}{0.979}     & \multicolumn{1}{>{\columncolor{lightblue}}c}{0.985}     & \multicolumn{1}{>{\columncolor{lightblue}}c}{\textbf{0.988}}  \\
    
    \multicolumn{1}{c}{}& \multicolumn{1}{>{\columncolor{lightred}}c|}{\text{ t850 ($K$) }}          
    & \multicolumn{1}{>{\columncolor{lightred}}c}{2.566}      & \multicolumn{1}{>{\columncolor{lightred}}c}{--}         & \multicolumn{1}{>{\columncolor{lightred}}c}{3.013}    
    & \multicolumn{1}{>{\columncolor{lightred}}c}{2.819}      & \multicolumn{1}{>{\columncolor{lightred}}c}{2.056}      & \multicolumn{1}{>{\columncolor{lightred}}c|}{{\textbf{1.894}}}       
    & \multicolumn{1}{>{\columncolor{lightred}}c}{0.957}      & \multicolumn{1}{>{\columncolor{lightred}}c}{--}         & \multicolumn{1}{>{\columncolor{lightred}}c}{0.974}   
    & \multicolumn{1}{>{\columncolor{lightred}}c}{0.977}      & \multicolumn{1}{>{\columncolor{lightred}}c}{0.986}      & \multicolumn{1}{>{\columncolor{lightred}}c}{\textbf{0.990}}  \\
    
    & \multicolumn{1}{>{\columncolor{lightblue}}c|}{\text{ t2m ($K$) }}           
    & \multicolumn{1}{>{\columncolor{lightblue}}c}{--}        & \multicolumn{1}{>{\columncolor{lightblue}}c}{--}        & \multicolumn{1}{>{\columncolor{lightblue}}c}{--}       
    & \multicolumn{1}{>{\columncolor{lightblue}}c}{60.915}    & \multicolumn{1}{>{\columncolor{lightblue}}c}{28.556}    & \multicolumn{1}{>{\columncolor{lightblue}}c|}{\textbf{12.063}}      
    & \multicolumn{1}{>{\columncolor{lightblue}}c}{--}        & \multicolumn{1}{>{\columncolor{lightblue}}c}{--}        & \multicolumn{1}{>{\columncolor{lightblue}}c}{--}      
    & \multicolumn{1}{>{\columncolor{lightblue}}c}{0.889}     & \multicolumn{1}{>{\columncolor{lightblue}}c}{0.977}     & \multicolumn{1}{>{\columncolor{lightblue}}c}{\textbf{0.996}}  \\
    
    \multicolumn{1}{c}{}& \multicolumn{1}{>{\columncolor{lightred}}c|}{\text{ u10 (m/s) }}           
    & \multicolumn{1}{>{\columncolor{lightred}}c}{4.232}      & \multicolumn{1}{>{\columncolor{lightred}}c}{--}         & \multicolumn{1}{>{\columncolor{lightred}}c}{4.074}    
    & \multicolumn{1}{>{\columncolor{lightred}}c}{0.832}      & \multicolumn{1}{>{\columncolor{lightred}}c}{0.569}      & \multicolumn{1}{>{\columncolor{lightred}}c|}{{\textbf{0.490}}}       
    & \multicolumn{1}{>{\columncolor{lightred}}c}{0.321}      & \multicolumn{1}{>{\columncolor{lightred}}c}{--}         & \multicolumn{1}{>{\columncolor{lightred}}c}{0.321}   
    & \multicolumn{1}{>{\columncolor{lightred}}c}{0.784}      & \multicolumn{1}{>{\columncolor{lightred}}c}{0.906}      & \multicolumn{1}{>{\columncolor{lightred}}c}{\textbf{0.928}}  \\
    
    & \multicolumn{1}{>{\columncolor{lightblue}}c|}{\text{ v10 (m/s) }}           
    & \multicolumn{1}{>{\columncolor{lightblue}}c}{2.428}     & \multicolumn{1}{>{\columncolor{lightblue}}c}{--}        & \multicolumn{1}{>{\columncolor{lightblue}}c}{2.314}    
    & \multicolumn{1}{>{\columncolor{lightblue}}c}{0.781}     & \multicolumn{1}{>{\columncolor{lightblue}}c}{0.527}     & \multicolumn{1}{>{\columncolor{lightblue}}c|}{\textbf{0.466}}       
    & \multicolumn{1}{>{\columncolor{lightblue}}c}{0.479}     & \multicolumn{1}{>{\columncolor{lightblue}}c}{--}        & \multicolumn{1}{>{\columncolor{lightblue}}c}{0.485}   
    & \multicolumn{1}{>{\columncolor{lightblue}}c}{0.785}     & \multicolumn{1}{>{\columncolor{lightblue}}c}{0.916}     & \multicolumn{1}{>{\columncolor{lightblue}}c}{\textbf{0.934}}  \\
    
    \bottomrule
  \end{tabular}
\let\everydisplay\oldeverydisplay
\let\everymath\oldeverymath
\end{table*}

\begin{table*}[t] 
\let\oldeverydisplay\everydisplay
\let\oldeverymath\everymath
\everydisplay{}
\everymath{}
\definecolor{lightblue}{RGB}{255, 255, 255}  
\definecolor{lightred}{RGB}{255,255,255}   
\definecolor{orange}{RGB}{229, 152, 102}       

\setlength{\tabcolsep}{6pt} 
\renewcommand{\arraystretch}{0.65} 
\caption{Comparison of physical consistency using SpecDiv metric (lower is better $\downarrow$).}
\vspace{-0.3cm}
\label{tab:main_SpecDiv}
  \begin{tabular}{c|cccccc|cccccc}
    \toprule
    \multirow{2}{*}{\textbf{Variable}} & \multicolumn{6}{c|}{\textbf{Weeks 3-4}} & \multicolumn{6}{c}{\textbf{Weeks 5-6}}\\ 
    &  FCNV2 &  GC & Pangu & ClimaX & CirT & \textbf{TelePiT} &  FCN-V2 &  GC & Pangu & ClimaX & CirT & \textbf{TelePiT}  \\
    
    \midrule
    \multicolumn{1}{>{\columncolor{lightblue}}c|}{\text{ z500 ($gpm$) }}  & \multicolumn{1}{>{\columncolor{lightblue}}c}{0.1374}    & \multicolumn{1}{>{\columncolor{lightblue}}c}{0.0742}   & \multicolumn{1}{>{\columncolor{lightblue}}c}{0.6081}    & \multicolumn{1}{>{\columncolor{lightblue}}c}{0.3335}   & \multicolumn{1}{>{\columncolor{lightblue}}c}{0.0286}   & \multicolumn{1}{>{\columncolor{lightblue}}c|}{\textbf{0.0180}}      & \multicolumn{1}{>{\columncolor{lightblue}}c}{0.1451}   & \multicolumn{1}{>{\columncolor{lightblue}}c}{--}     & \multicolumn{1}{>{\columncolor{lightblue}}c}{0.6793}   & \multicolumn{1}{>{\columncolor{lightblue}}c}{0.3567}   & \multicolumn{1}{>{\columncolor{lightblue}}c}{0.0664}   & \multicolumn{1}{>{\columncolor{lightblue}}c}{\textbf{0.0161}}  \\
    \multicolumn{1}{>{\columncolor{lightred}}c|}{\text{ z850 ($gpm$) }}  & \multicolumn{1}{>{\columncolor{lightred}}c}{0.2537}    & \multicolumn{1}{>{\columncolor{lightred}}c}{0.3337}   & \multicolumn{1}{>{\columncolor{lightred}}c}{0.1448}    & \multicolumn{1}{>{\columncolor{lightred}}c}{0.0196}   & \multicolumn{1}{>{\columncolor{lightred}}c}{0.1286}   & \multicolumn{1}{>{\columncolor{lightred}}c|}{\textbf{0.0175}}      & \multicolumn{1}{>{\columncolor{lightred}}c}{0.2833}   & \multicolumn{1}{>{\columncolor{lightred}}c}{--}     & \multicolumn{1}{>{\columncolor{lightred}}c}{0.1502}   & \multicolumn{1}{>{\columncolor{lightred}}c}{0.0568}   & \multicolumn{1}{>{\columncolor{lightred}}c}{0.2144}   & \multicolumn{1}{>{\columncolor{lightred}}c}{\textbf{0.0168}}  \\
    \multicolumn{1}{>{\columncolor{lightblue}}c|}{\text{ t500 ($K$) }}          & \multicolumn{1}{>{\columncolor{lightblue}}c}{0.4010}    & \multicolumn{1}{>{\columncolor{lightblue}}c}{0.0613}   & \multicolumn{1}{>{\columncolor{lightblue}}c}{0.2040}    & \multicolumn{1}{>{\columncolor{lightblue}}c}{0.3391}   & \multicolumn{1}{>{\columncolor{lightblue}}c}{0.0948}   & \multicolumn{1}{>{\columncolor{lightblue}}c|}{\textbf{0.0150}}     & \multicolumn{1}{>{\columncolor{lightblue}}c}{0.3923}   & \multicolumn{1}{>{\columncolor{lightblue}}c}{--}     & \multicolumn{1}{>{\columncolor{lightblue}}c}{0.2160}   & \multicolumn{1}{>{\columncolor{lightblue}}c}{0.3104}   & \multicolumn{1}{>{\columncolor{lightblue}}c}{0.2310}   & \multicolumn{1}{>{\columncolor{lightblue}}c}{\textbf{0.0184}} \\
    \multicolumn{1}{>{\columncolor{lightred}}c|}{\text{ t850 ($K$) }}          & \multicolumn{1}{>{\columncolor{lightred}}c}{0.2919}    & \multicolumn{1}{>{\columncolor{lightred}}c}{0.0684}   & \multicolumn{1}{>{\columncolor{lightred}}c}{0.2107}    & \multicolumn{1}{>{\columncolor{lightred}}c}{0.1561}   & \multicolumn{1}{>{\columncolor{lightred}}c}{0.0256}   & \multicolumn{1}{>{\columncolor{lightred}}c|}{\textbf{0.0045}}      & \multicolumn{1}{>{\columncolor{lightred}}c}{0.2901}   & \multicolumn{1}{>{\columncolor{lightred}}c}{--}     & \multicolumn{1}{>{\columncolor{lightred}}c}{0.2194}   & \multicolumn{1}{>{\columncolor{lightred}}c}{0.1582}   & \multicolumn{1}{>{\columncolor{lightred}}c}{0.0173}   & \multicolumn{1}{>{\columncolor{lightred}}c}{\textbf{0.0042}}  \\
    \multicolumn{1}{>{\columncolor{lightblue}}c|}{\text{ t2m ($K$) }}           & \multicolumn{1}{>{\columncolor{lightblue}}c}{--}        & \multicolumn{1}{>{\columncolor{lightblue}}c}{--}       & \multicolumn{1}{>{\columncolor{lightblue}}c}{--}        & \multicolumn{1}{>{\columncolor{lightblue}}c}{0.4706}   & \multicolumn{1}{>{\columncolor{lightblue}}c}{0.0983}   & \multicolumn{1}{>{\columncolor{lightblue}}c|}{\textbf{0.0018}}       & \multicolumn{1}{>{\columncolor{lightblue}}c}{--}       & \multicolumn{1}{>{\columncolor{lightblue}}c}{--}     & \multicolumn{1}{>{\columncolor{lightblue}}c}{--}       & \multicolumn{1}{>{\columncolor{lightblue}}c}{0.4744}   & \multicolumn{1}{>{\columncolor{lightblue}}c}{0.0926}   & \multicolumn{1}{>{\columncolor{lightblue}}c}{\textbf{0.0018}} \\
    \multicolumn{1}{>{\columncolor{lightred}}c|}{\text{ u10 (m/s) }}           & \multicolumn{1}{>{\columncolor{lightred}}c}{0.2229}    & \multicolumn{1}{>{\columncolor{lightred}}c}{--}       & \multicolumn{1}{>{\columncolor{lightred}}c}{0.0576}    & \multicolumn{1}{>{\columncolor{lightred}}c}{0.2614}   & \multicolumn{1}{>{\columncolor{lightred}}c}{0.0482}   & \multicolumn{1}{>{\columncolor{lightred}}c|}{\textbf{0.0020}}      & \multicolumn{1}{>{\columncolor{lightred}}c}{0.2128}   & \multicolumn{1}{>{\columncolor{lightred}}c}{--}     & \multicolumn{1}{>{\columncolor{lightred}}c}{0.0559}   & \multicolumn{1}{>{\columncolor{lightred}}c}{0.2796}   & \multicolumn{1}{>{\columncolor{lightred}}c}{0.0460}   & \multicolumn{1}{>{\columncolor{lightred}}c}{\textbf{0.0026}} \\
    \multicolumn{1}{>{\columncolor{lightblue}}c|}{\text{ v10 (m/s) }}           & \multicolumn{1}{>{\columncolor{lightblue}}c}{0.4131}    & \multicolumn{1}{>{\columncolor{lightblue}}c}{--}       & \multicolumn{1}{>{\columncolor{lightblue}}c}{0.2156}    & \multicolumn{1}{>{\columncolor{lightblue}}c}{0.1256}   & \multicolumn{1}{>{\columncolor{lightblue}}c}{0.0577}   & \multicolumn{1}{>{\columncolor{lightblue}}c|}{\textbf{0.0061}}     & \multicolumn{1}{>{\columncolor{lightblue}}c}{0.4098}   & \multicolumn{1}{>{\columncolor{lightblue}}c}{--}     & \multicolumn{1}{>{\columncolor{lightblue}}c}{0.2133}   & \multicolumn{1}{>{\columncolor{lightblue}}c}{0.1443}   & \multicolumn{1}{>{\columncolor{lightblue}}c}{0.0572}   & \multicolumn{1}{>{\columncolor{lightblue}}c}{\textbf{0.0051}} \\
    
    \bottomrule
  \end{tabular}
\let\everydisplay\oldeverydisplay
\let\everymath\oldeverymath
\end{table*}

\begin{table*}[t] 
\let\oldeverydisplay\everydisplay
\let\oldeverymath\everymath
\everydisplay{}
\everymath{}
\definecolor{lightblue}{RGB}{255, 255, 255}  
\definecolor{lightred}{RGB}{255, 255, 255}   
\setlength{\tabcolsep}{10pt} 
\renewcommand{\arraystretch}{0.65} 
\caption{Performance comparison of \textbf{TelePiT} against operational forecasting systems on RMSE ($\downarrow$), ACC ($\uparrow$), SepcDiv ($\downarrow$) metrics.}
\vspace{-0.3cm}
\label{appendix_tab:main_2}
  \begin{tabular}{cc|cccc|cccc}
    \toprule
    \multirow{2}{*}{\textbf{Metric}} & \multirow{2}{*}{\textbf{Model}} & \multicolumn{4}{c|}{\textbf{Weeks 3-4}} & \multicolumn{4}{c}{\textbf{Weeks 5-6}}\\ 
    & & z500 & z850 & t500 & t850 & z500 & z850 & t500 & t850 \\
    
    \midrule
    \multirow{5}{*}{\rotatebox{0}{\textbf{RMSE ($\downarrow$)}}}
    &\multicolumn{1}{>{\columncolor{lightblue}}c|}{\text{ UKMO }}  
    &\multicolumn{1}{>{\columncolor{lightblue}}c}{64.2965} & \multicolumn{1}{>{\columncolor{lightblue}}c}{43.1836} & \multicolumn{1}{>{\columncolor{lightblue}}c}{{2.1440}}& \multicolumn{1}{>{\columncolor{lightblue}}c|}{{2.5763}}      
    &\multicolumn{1}{>{\columncolor{lightblue}}c}{63.0383} & \multicolumn{1}{>{\columncolor{lightblue}}c}{45.6263} & \multicolumn{1}{>{\columncolor{lightblue}}c}{{2.2743}} & \multicolumn{1}{>{\columncolor{lightblue}}c}{{2.3924}} \\
    
    &\multicolumn{1}{>{\columncolor{lightred}}c|}{\text{ NCEP }}  
    &\multicolumn{1}{>{\columncolor{lightred}}c}{68.4408} & \multicolumn{1}{>{\columncolor{lightred}}c}{44.7869} & \multicolumn{1}{>{\columncolor{lightred}}c}{2.2250} & \multicolumn{1}{>{\columncolor{lightred}}c|}{2.7422}   
    &\multicolumn{1}{>{\columncolor{lightred}}c}{71.3320} & \multicolumn{1}{>{\columncolor{lightred}}c}{46.8340} & \multicolumn{1}{>{\columncolor{lightred}}c}{2.3433} & \multicolumn{1}{>{\columncolor{lightred}}c}{2.8331} \\

    &\multicolumn{1}{>{\columncolor{lightblue}}c|}{\text{ ECMWF }}  
    &\multicolumn{1}{>{\columncolor{lightblue}}c}{63.6737} & \multicolumn{1}{>{\columncolor{lightblue}}c}{42.3491} & \multicolumn{1}{>{\columncolor{lightblue}}c}{{2.0667}}& \multicolumn{1}{>{\columncolor{lightblue}}c|}{{2.3910}}      
    &\multicolumn{1}{>{\columncolor{lightblue}}c}{68.1156} & \multicolumn{1}{>{\columncolor{lightblue}}c}{45.3303} & \multicolumn{1}{>{\columncolor{lightblue}}c}{{2.1772}} & \multicolumn{1}{>{\columncolor{lightblue}}c}{{2.5110}} \\
    
    &\multicolumn{1}{>{\columncolor{lightred}}c|}{\text{ CMA }}  
    &\multicolumn{1}{>{\columncolor{lightred}}c}{69.8326} & \multicolumn{1}{>{\columncolor{lightred}}c}{47.0562} & \multicolumn{1}{>{\columncolor{lightred}}c}{2.5975} & \multicolumn{1}{>{\columncolor{lightred}}c|}{2.9181}   
    &\multicolumn{1}{>{\columncolor{lightred}}c}{71.3996} & \multicolumn{1}{>{\columncolor{lightred}}c}{48.5830} & \multicolumn{1}{>{\columncolor{lightred}}c}{2.6388} & \multicolumn{1}{>{\columncolor{lightred}}c}{2.9334} \\

    &\multicolumn{1}{>{\columncolor{lightblue}}c|}{\textbf{ TelePiT }}  
    &\multicolumn{1}{>{\columncolor{lightblue}}c}{\textbf{48.6714}} & \multicolumn{1}{>{\columncolor{lightblue}}c}{\textbf{31.0823}} & \multicolumn{1}{>{\columncolor{lightblue}}c}{\textbf{1.6849}}& \multicolumn{1}{>{\columncolor{lightblue}}c|}{\textbf{1.8739}}      
    &\multicolumn{1}{>{\columncolor{lightblue}}c}{\textbf{47.7618}} & \multicolumn{1}{>{\columncolor{lightblue}}c}{\textbf{30.4056}} & \multicolumn{1}{>{\columncolor{lightblue}}c}{\textbf{1.6985}} & \multicolumn{1}{>{\columncolor{lightblue}}c}{\textbf{1.8949}} \\

    \midrule
    \multirow{5}{*}{\rotatebox{0}{\textbf{ACC ($\uparrow$)}}}
    &\multicolumn{1}{>{\columncolor{lightblue}}c|}{\text{ UKMO }}  
    &\multicolumn{1}{>{\columncolor{lightblue}}c}{0.9726} & \multicolumn{1}{>{\columncolor{lightblue}}c}{0.9294} & \multicolumn{1}{>{\columncolor{lightblue}}c}{{0.9804}}& \multicolumn{1}{>{\columncolor{lightblue}}c|}{{0.9768}}      
    &\multicolumn{1}{>{\columncolor{lightblue}}c}{0.9695} & \multicolumn{1}{>{\columncolor{lightblue}}c}{0.9225} & \multicolumn{1}{>{\columncolor{lightblue}}c}{{0.9781}} & \multicolumn{1}{>{\columncolor{lightblue}}c}{{0.9767}} \\
    
    &\multicolumn{1}{>{\columncolor{lightred}}c|}{\text{ NCEP }}  
    &\multicolumn{1}{>{\columncolor{lightred}}c}{0.9694} & \multicolumn{1}{>{\columncolor{lightred}}c}{0.9206} & \multicolumn{1}{>{\columncolor{lightred}}c}{0.9790} & \multicolumn{1}{>{\columncolor{lightred}}c|}{0.9774}   
    &\multicolumn{1}{>{\columncolor{lightred}}c}{0.9669} & \multicolumn{1}{>{\columncolor{lightred}}c}{0.9133} & \multicolumn{1}{>{\columncolor{lightred}}c}{0.9766} & \multicolumn{1}{>{\columncolor{lightred}}c}{0.9760} \\

    &\multicolumn{1}{>{\columncolor{lightblue}}c|}{\text{ ECMWF }}  
    &\multicolumn{1}{>{\columncolor{lightblue}}c}{0.9724} & \multicolumn{1}{>{\columncolor{lightblue}}c}{0.9289} & \multicolumn{1}{>{\columncolor{lightblue}}c}{{0.9815}}& \multicolumn{1}{>{\columncolor{lightblue}}c|}{{0.9828}}      
    &\multicolumn{1}{>{\columncolor{lightblue}}c}{0.9677} & \multicolumn{1}{>{\columncolor{lightblue}}c}{0.9157} & \multicolumn{1}{>{\columncolor{lightblue}}c}{{0.9794}} & \multicolumn{1}{>{\columncolor{lightblue}}c}{{0.9807}} \\
    
    &\multicolumn{1}{>{\columncolor{lightred}}c|}{\text{ CMA }}  
    &\multicolumn{1}{>{\columncolor{lightred}}c}{0.9679} & \multicolumn{1}{>{\columncolor{lightred}}c}{0.9188} & \multicolumn{1}{>{\columncolor{lightred}}c}{0.9746} & \multicolumn{1}{>{\columncolor{lightred}}c|}{0.9744}   
    &\multicolumn{1}{>{\columncolor{lightred}}c}{0.9660} & \multicolumn{1}{>{\columncolor{lightred}}c}{0.9121} & \multicolumn{1}{>{\columncolor{lightred}}c}{0.9739} & \multicolumn{1}{>{\columncolor{lightred}}c}{0.9738} \\

    &\multicolumn{1}{>{\columncolor{lightblue}}c|}{\textbf{ TelePiT }}  
    &\multicolumn{1}{>{\columncolor{lightblue}}c}{\textbf{0.9848}} & \multicolumn{1}{>{\columncolor{lightblue}}c}{\textbf{0.9626}} & \multicolumn{1}{>{\columncolor{lightblue}}c}{\textbf{0.9882}}& \multicolumn{1}{>{\columncolor{lightblue}}c|}{\textbf{0.9900}}      
    &\multicolumn{1}{>{\columncolor{lightblue}}c}{\textbf{0.9852}} & \multicolumn{1}{>{\columncolor{lightblue}}c}{\textbf{0.9643}} & \multicolumn{1}{>{\columncolor{lightblue}}c}{\textbf{0.9881}} & \multicolumn{1}{>{\columncolor{lightblue}}c}{\textbf{0.9897}} \\

    \midrule
    \multirow{5}{*}{\rotatebox{0}{\textbf{SpecDiv ($\downarrow$)}}}
    &\multicolumn{1}{>{\columncolor{lightblue}}c|}{\text{ UKMO }}  
    &\multicolumn{1}{>{\columncolor{lightblue}}c}{0.1173} & \multicolumn{1}{>{\columncolor{lightblue}}c}{0.1142} & \multicolumn{1}{>{\columncolor{lightblue}}c}{{0.2408}}& \multicolumn{1}{>{\columncolor{lightblue}}c|}{{0.0989}}      
    &\multicolumn{1}{>{\columncolor{lightblue}}c}{0.1330} & \multicolumn{1}{>{\columncolor{lightblue}}c}{0.1267} & \multicolumn{1}{>{\columncolor{lightblue}}c}{{0.2905}} & \multicolumn{1}{>{\columncolor{lightblue}}c}{{0.1182}} \\
    
    &\multicolumn{1}{>{\columncolor{lightred}}c|}{\text{ NCEP }}  
    &\multicolumn{1}{>{\columncolor{lightred}}c}{0.5567} & \multicolumn{1}{>{\columncolor{lightred}}c}{0.4875} & \multicolumn{1}{>{\columncolor{lightred}}c}{4.3344} & \multicolumn{1}{>{\columncolor{lightred}}c|}{0.6459}   
    &\multicolumn{1}{>{\columncolor{lightred}}c}{0.5598} & \multicolumn{1}{>{\columncolor{lightred}}c}{0.4826} & \multicolumn{1}{>{\columncolor{lightred}}c}{4.2610} & \multicolumn{1}{>{\columncolor{lightred}}c}{0.6135} \\

    &\multicolumn{1}{>{\columncolor{lightblue}}c|}{\text{ ECMWF }}  
    &\multicolumn{1}{>{\columncolor{lightblue}}c}{0.3976} & \multicolumn{1}{>{\columncolor{lightblue}}c}{0.0599} & \multicolumn{1}{>{\columncolor{lightblue}}c}{{0.1229}}& \multicolumn{1}{>{\columncolor{lightblue}}c|}{{0.2375}}      
    &\multicolumn{1}{>{\columncolor{lightblue}}c}{0.3093} & \multicolumn{1}{>{\columncolor{lightblue}}c}{0.0687} & \multicolumn{1}{>{\columncolor{lightblue}}c}{{0.0575}} & \multicolumn{1}{>{\columncolor{lightblue}}c}{{0.1187}} \\
    
    &\multicolumn{1}{>{\columncolor{lightred}}c|}{\text{ CMA }}  
    &\multicolumn{1}{>{\columncolor{lightred}}c}{0.1612} & \multicolumn{1}{>{\columncolor{lightred}}c}{0.1671} & \multicolumn{1}{>{\columncolor{lightred}}c}{0.2725} & \multicolumn{1}{>{\columncolor{lightred}}c|}{0.0318}   
    &\multicolumn{1}{>{\columncolor{lightred}}c}{0.1589} & \multicolumn{1}{>{\columncolor{lightred}}c}{0.1544} & \multicolumn{1}{>{\columncolor{lightred}}c}{0.1277} & \multicolumn{1}{>{\columncolor{lightred}}c}{0.0324} \\

    &\multicolumn{1}{>{\columncolor{lightblue}}c|}{\textbf{ TelePiT }}  
    &\multicolumn{1}{>{\columncolor{lightblue}}c}{\textbf{0.0180}} & \multicolumn{1}{>{\columncolor{lightblue}}c}{\textbf{0.0175}} & \multicolumn{1}{>{\columncolor{lightblue}}c}{\textbf{0.0150}}& \multicolumn{1}{>{\columncolor{lightblue}}c|}{\textbf{0.0045}}      
    &\multicolumn{1}{>{\columncolor{lightblue}}c}{\textbf{0.0161}} & \multicolumn{1}{>{\columncolor{lightblue}}c}{\textbf{0.0168}} & \multicolumn{1}{>{\columncolor{lightblue}}c}{\textbf{0.0184}} & \multicolumn{1}{>{\columncolor{lightblue}}c}{\textbf{0.0042}} \\

    \bottomrule
  \end{tabular}
\let\everydisplay\oldeverydisplay
\let\everymath\oldeverymath
\vspace{-0.4cm}
\end{table*}

\section{Experiments}
\label{sec:Experiments}

\textbf{Datasets.}
We evaluate \textbf{TelePiT} using the ERA5 reanalysis dataset \citep{hersbach2020era5}, as processed by the ChaosBench benchmark \citep{nathaniel2024chaosbench}. This dataset provides global atmospheric variables at multiple pressure levels and at the surface, with a uniform spatial resolution of $1.5^\circ$, corresponding to a grid of $121 \times 240$.
Inspired by the previous study \citep{liucirt,nguyen2024scaling,bi2023accurate,kurth2023fourcastnet,lam2023learning}, we include six pressure-level fields: geopotential ($z$), specific humidity ($q$), temperature ($t$), zonal wind ($u$), meridional wind ($v$), and vertical velocity ($w$), each provided at ten discrete pressure levels ($10$, $50$, $100$, $200$, $300$, $500$, $700$, $850$, $925$, and $1000 hPa$). Additionally, three single-level surface variables are included: 2-meter temperature ($t2m$), 10-meter u-wind ($10u$) and 10-meter v-wind ($10v$), yielding a total of 63 variables. We use data from 1979–2016 for training, 2017 for validation, and 2018 for testing.

\textbf{Evaluation Metric.}
Following existing works \cite{liucirt,nathaniel2024chaosbench,rasp2024weatherbench}, we adopt three metrics: 
(1) Root Mean Squared Error (RMSE) with a latitude-based weighting scheme that accounts for Earth's spherical geometry; 
(2) Anomaly Correlation Coefficient (ACC), which measures the pattern similarity between predicted and observed anomalies relative to climatology; 
and (3) Spectral Divergence (SpecDiv), a physics-based metric that quantifies the deviation between the power spectra of prediction and target in the spectral domain using Kullback-Leibler divergence principles \cite{kullback1951information}. 
For RMSE and SpecDiv, lower values indicate better performance, while higher ACC values signify better prediction skill. The detailed mathematical formulations of these metrics are provided in Appendix \ref{appendix_sec:Evaluation Metric}.

\textbf{Baselines.}
To rigorously evaluate \textbf{TelePiT}, we benchmark against state-of-the-art forecasting models in two categories: (1) data-driven baselines, including FourCastNetV2 \cite{kurth2023fourcastnet}, Pangu-Weather \cite{bi2023accurate}, GraphCast \cite{lam2023learning}, ClimaX \cite{nguyen2023climax}, and CirT \cite{liucirt}, following the standard S2S evaluation protocol of \cite{nathaniel2024chaosbench}; and (2) physics-based baselines of leading operational S2S systems: ECMWF \cite{molteni1996ecmwf}, UKMO \cite{williams2015met}, NCEP \cite{saha2014ncep}, CMA \cite{wu2019beijing}. We provide more details of baselines in Appendix \ref{appendix_sec:Baselines}.

\begin{figure*}[t]
    \centering
    \includegraphics[width=1\linewidth]{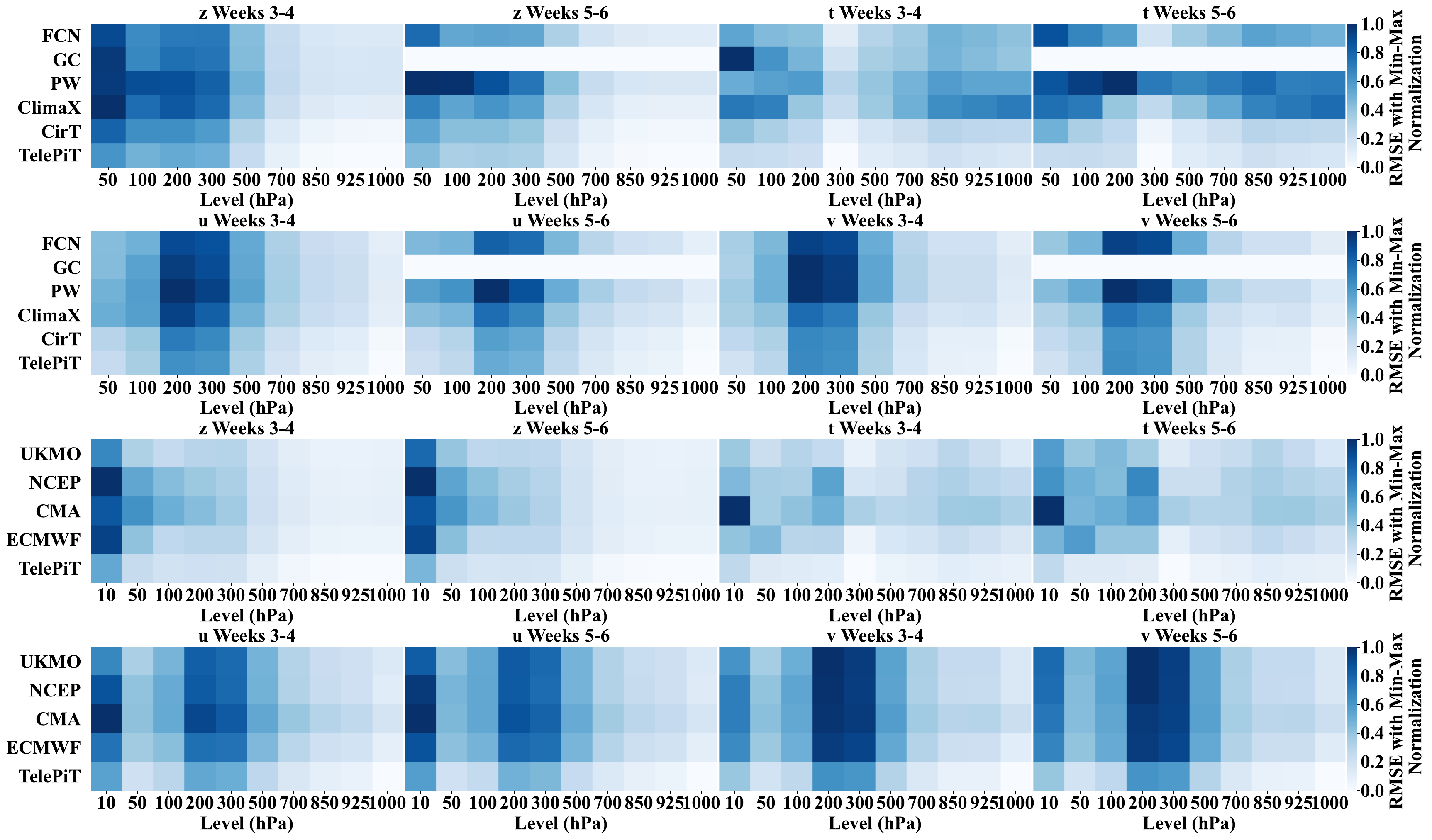}
    \caption{RMSE comparison on variable $z$, $t$, $u$, and $v$ of different pressure levels.}
    \label{fig:heatmap}
\end{figure*}

\textbf{Experimental setup.} 
We implement all models using PyTorch Lightning and train our direct training baselines with consistent hyperparameters: batch size 16, hidden dimension 256, learning rate 0.01, $\gamma = 0.1$, and $\lambda = 0.2$. For simplicity, we initialize $\mathbf{P}_j$ as a trainable vector that enables to learn the teleconnection patterns from data directly.
All training was conducted on 4 GeForce RTX A40 GPUs. 
ClimaX and CirT were trained using same configuration as \textbf{TelePiT}, while pre-trained versions of FourCastNetV2, Pangu-Weather, and GraphCast were accessed through ECMWF's API\footnote{\url{https://github.com/ecmwf-lab/ai-models}}. Inference for these pre-trained models was performed on NVIDIA A800 80G GPUs. It should be noted that although FourCastNetV2 and Pangu-Weather reportedly support $t2m$ predictions, this capability was not available in the ECMWF-provided implementations. Additionally, GraphCast encountered out-of-memory when attempting inference for Weeks 5-6 forecast period.

Due to space constraints, we present results for key variables and primary metrics in the main text. More additional experiments and model complexity analysis can be found in the Appendix \ref{appendix_sec:Additional Experimental Details}.

\begin{table*}[t]
\let\oldeverydisplay\everydisplay
\let\oldeverymath\everymath
\everydisplay{}
\everymath{}

\definecolor{lightblue}{RGB}{255, 255, 255}  
\definecolor{lightred}{RGB}{255,255,255}   

\small
\setlength{\tabcolsep}{6pt} 
\renewcommand{\arraystretch}{0.65} 

\caption{Ablation study of different \textbf{TelePiT} model variants.}
\label{tab:ablation}
  \begin{tabular}{cc|ccccccc|ccccccc}
    \toprule
    \multirow{2}{*}{\textbf{}} &\multirow{2}{*}{\textbf{Model}} & \multicolumn{7}{c|}{\textbf{RMSE ($\downarrow$)}} & \multicolumn{7}{c}{\textbf{ACC ($\uparrow$)}}\\ 
    &  &  z500 &  z850 & t500 & t850 & t2m & u10 &  v10 &  z500 & z850 & t500 & t850 & t2m & u10 & v10\\
    
    \midrule
    \multirow{5}{*}{\rotatebox{90}{\textbf{Weeks 3-4}}}
    & \multicolumn{1}{>{\columncolor{lightblue}}c|}{\text{ w/o SHE }}    & \multicolumn{1}{>{\columncolor{lightblue}}c}{51.217}   & \multicolumn{1}{>{\columncolor{lightblue}}c}{33.153}   & \multicolumn{1}{>{\columncolor{lightblue}}c}{1.712}   & \multicolumn{1}{>{\columncolor{lightblue}}c}{1.983}   & \multicolumn{1}{>{\columncolor{lightblue}}c}{27.206}   & \multicolumn{1}{>{\columncolor{lightblue}}c}{0.562}   & \multicolumn{1}{>{\columncolor{lightblue}}c|}{0.526}   & \multicolumn{1}{>{\columncolor{lightblue}}c}{0.983}   & \multicolumn{1}{>{\columncolor{lightblue}}c}{0.957}   & \multicolumn{1}{>{\columncolor{lightblue}}c}{0.987}   & \multicolumn{1}{>{\columncolor{lightblue}}c}{0.988}   & \multicolumn{1}{>{\columncolor{lightblue}}c}{0.980}   & \multicolumn{1}{>{\columncolor{lightblue}}c}{0.908}   & \multicolumn{1}{>{\columncolor{lightblue}}c}{0.914}   \\
    \multicolumn{1}{c}{}& \multicolumn{1}{>{\columncolor{lightred}}c|}{\text{ w/o WD }}     & \multicolumn{1}{>{\columncolor{lightred}}c}{51.949}   & \multicolumn{1}{>{\columncolor{lightred}}c}{33.635}   & \multicolumn{1}{>{\columncolor{lightred}}c}{1.724}   & \multicolumn{1}{>{\columncolor{lightred}}c}{2.018}   & \multicolumn{1}{>{\columncolor{lightred}}c}{22.901}   & \multicolumn{1}{>{\columncolor{lightred}}c}{0.542}   & \multicolumn{1}{>{\columncolor{lightred}}c|}{0.510}   & \multicolumn{1}{>{\columncolor{lightred}}c}{0.982}   & \multicolumn{1}{>{\columncolor{lightred}}c}{0.956}   & \multicolumn{1}{>{\columncolor{lightred}}c}{0.986}   & \multicolumn{1}{>{\columncolor{lightred}}c}{0.987}   & \multicolumn{1}{>{\columncolor{lightred}}c}{0.986}   & \multicolumn{1}{>{\columncolor{lightred}}c}{0.915}   & \multicolumn{1}{>{\columncolor{lightred}}c}{0.921}   \\
    & \multicolumn{1}{>{\columncolor{lightblue}}c|}{\text{ w/o ODE }}    & \multicolumn{1}{>{\columncolor{lightblue}}c}{50.582}   & \multicolumn{1}{>{\columncolor{lightblue}}c}{32.818}   & \multicolumn{1}{>{\columncolor{lightblue}}c}{1.697}   & \multicolumn{1}{>{\columncolor{lightblue}}c}{2.002}   & \multicolumn{1}{>{\columncolor{lightblue}}c}{22.780}   & \multicolumn{1}{>{\columncolor{lightblue}}c}{0.542}   & \multicolumn{1}{>{\columncolor{lightblue}}c|}{0.510}   & \multicolumn{1}{>{\columncolor{lightblue}}c}{0.983}   & \multicolumn{1}{>{\columncolor{lightblue}}c}{0.959}   & \multicolumn{1}{>{\columncolor{lightblue}}c}{0.987}   & \multicolumn{1}{>{\columncolor{lightblue}}c}{0.987}   & \multicolumn{1}{>{\columncolor{lightblue}}c}{0.986}   & \multicolumn{1}{>{\columncolor{lightblue}}c}{0.916}   & \multicolumn{1}{>{\columncolor{lightblue}}c}{0.922}   \\
    \multicolumn{1}{c}{}& \multicolumn{1}{>{\columncolor{lightred}}c|}{\text{ w/o TA }}     & \multicolumn{1}{>{\columncolor{lightred}}c}{51.333}   & \multicolumn{1}{>{\columncolor{lightred}}c}{33.193}   & \multicolumn{1}{>{\columncolor{lightred}}c}{1.727}   & \multicolumn{1}{>{\columncolor{lightred}}c}{2.023}   & \multicolumn{1}{>{\columncolor{lightred}}c}{23.923}   & \multicolumn{1}{>{\columncolor{lightred}}c}{0.548}   & \multicolumn{1}{>{\columncolor{lightred}}c|}{0.509}   & \multicolumn{1}{>{\columncolor{lightred}}c}{0.982}   & \multicolumn{1}{>{\columncolor{lightred}}c}{0.957}   & \multicolumn{1}{>{\columncolor{lightred}}c}{0.987}   & \multicolumn{1}{>{\columncolor{lightred}}c}{0.987}   & \multicolumn{1}{>{\columncolor{lightred}}c}{0.985}   & \multicolumn{1}{>{\columncolor{lightred}}c}{0.916}   & \multicolumn{1}{>{\columncolor{lightred}}c}{0.921}   \\
    & \multicolumn{1}{>{\columncolor{lightblue}}c|}{\textbf{ \textbf{TelePiT} }} & \multicolumn{1}{>{\columncolor{lightblue}}c}{\textbf{48.671}}   & \multicolumn{1}{>{\columncolor{lightblue}}c}{\textbf{31.082}}   & \multicolumn{1}{>{\columncolor{lightblue}}c}{\textbf{1.684}}   & \multicolumn{1}{>{\columncolor{lightblue}}c}{\textbf{1.873}}   & \multicolumn{1}{>{\columncolor{lightblue}}c}{\textbf{12.057}}   & \multicolumn{1}{>{\columncolor{lightblue}}c}{\textbf{0.491}}   & \multicolumn{1}{>{\columncolor{lightblue}}c|}{\textbf{0.464}}   & \multicolumn{1}{>{\columncolor{lightblue}}c}{\textbf{0.985}}   & \multicolumn{1}{>{\columncolor{lightblue}}c}{\textbf{0.963}}   & \multicolumn{1}{>{\columncolor{lightblue}}c}{\textbf{0.988}}   & \multicolumn{1}{>{\columncolor{lightblue}}c}{\textbf{0.990}}   & \multicolumn{1}{>{\columncolor{lightblue}}c}{\textbf{0.996}}   & \multicolumn{1}{>{\columncolor{lightblue}}c}{\textbf{0.928}}   & \multicolumn{1}{>{\columncolor{lightblue}}c}{\textbf{0.934}}   \\
    \midrule
    \multirow{5}{*}{\rotatebox{90}{\textbf{Weeks 5-6}}}
    & \multicolumn{1}{>{\columncolor{lightblue}}c|}{\text{ w/o SHE }}    & \multicolumn{1}{>{\columncolor{lightblue}}c}{51.926}   & \multicolumn{1}{>{\columncolor{lightblue}}c}{33.510}   & \multicolumn{1}{>{\columncolor{lightblue}}c}{1.747}   & \multicolumn{1}{>{\columncolor{lightblue}}c}{2.024}   & \multicolumn{1}{>{\columncolor{lightblue}}c}{27.217}   & \multicolumn{1}{>{\columncolor{lightblue}}c}{0.567}   & \multicolumn{1}{>{\columncolor{lightblue}}c|}{0.529}   & \multicolumn{1}{>{\columncolor{lightblue}}c}{0.982}   & \multicolumn{1}{>{\columncolor{lightblue}}c}{0.956}   & \multicolumn{1}{>{\columncolor{lightblue}}c}{0.986}   & \multicolumn{1}{>{\columncolor{lightblue}}c}{0.987}   & \multicolumn{1}{>{\columncolor{lightblue}}c}{0.980}   & \multicolumn{1}{>{\columncolor{lightblue}}c}{0.904}   & \multicolumn{1}{>{\columncolor{lightblue}}c}{0.911}   \\
    \multicolumn{1}{c}{}& \multicolumn{1}{>{\columncolor{lightred}}c|}{\text{ w/o WD }}     & \multicolumn{1}{>{\columncolor{lightred}}c}{52.083}   & \multicolumn{1}{>{\columncolor{lightred}}c}{33.810}   & \multicolumn{1}{>{\columncolor{lightred}}c}{1.730}   & \multicolumn{1}{>{\columncolor{lightred}}c}{2.001}   & \multicolumn{1}{>{\columncolor{lightred}}c}{22.912}   & \multicolumn{1}{>{\columncolor{lightred}}c}{0.544}   & \multicolumn{1}{>{\columncolor{lightred}}c|}{0.513}   & \multicolumn{1}{>{\columncolor{lightred}}c}{0.981}   & \multicolumn{1}{>{\columncolor{lightred}}c}{0.954}   & \multicolumn{1}{>{\columncolor{lightred}}c}{0.986}   & \multicolumn{1}{>{\columncolor{lightred}}c}{0.987}   & \multicolumn{1}{>{\columncolor{lightred}}c}{0.986}   & \multicolumn{1}{>{\columncolor{lightred}}c}{0.910}   & \multicolumn{1}{>{\columncolor{lightred}}c}{0.917}   \\
    & \multicolumn{1}{>{\columncolor{lightblue}}c|}{\text{ w/o ODE }}    & \multicolumn{1}{>{\columncolor{lightblue}}c}{50.499}   & \multicolumn{1}{>{\columncolor{lightblue}}c}{32.793}   & \multicolumn{1}{>{\columncolor{lightblue}}c}{1.689}   & \multicolumn{1}{>{\columncolor{lightblue}}c}{1.971}   & \multicolumn{1}{>{\columncolor{lightblue}}c}{22.780}   & \multicolumn{1}{>{\columncolor{lightblue}}c}{0.545}   & \multicolumn{1}{>{\columncolor{lightblue}}c|}{0.509}   & \multicolumn{1}{>{\columncolor{lightblue}}c}{0.983}   & \multicolumn{1}{>{\columncolor{lightblue}}c}{0.958}   & \multicolumn{1}{>{\columncolor{lightblue}}c}{0.987}   & \multicolumn{1}{>{\columncolor{lightblue}}c}{0.988}   & \multicolumn{1}{>{\columncolor{lightblue}}c}{0.986}   & \multicolumn{1}{>{\columncolor{lightblue}}c}{0.913}   & \multicolumn{1}{>{\columncolor{lightblue}}c}{0.921}   \\
    \multicolumn{1}{c}{}& \multicolumn{1}{>{\columncolor{lightred}}c|}{\text{ w/o TA }}     & \multicolumn{1}{>{\columncolor{lightred}}c}{50.391}   & \multicolumn{1}{>{\columncolor{lightred}}c}{32.697}   & \multicolumn{1}{>{\columncolor{lightred}}c}{1.709}   & \multicolumn{1}{>{\columncolor{lightred}}c}{1.991}   & \multicolumn{1}{>{\columncolor{lightred}}c}{23.928}   & \multicolumn{1}{>{\columncolor{lightred}}c}{0.549}   & \multicolumn{1}{>{\columncolor{lightred}}c|}{0.508}   & \multicolumn{1}{>{\columncolor{lightred}}c}{0.983}   & \multicolumn{1}{>{\columncolor{lightred}}c}{0.959}   & \multicolumn{1}{>{\columncolor{lightred}}c}{0.987}   & \multicolumn{1}{>{\columncolor{lightred}}c}{0.987}   & \multicolumn{1}{>{\columncolor{lightred}}c}{0.986}   & \multicolumn{1}{>{\columncolor{lightred}}c}{0.914}   & \multicolumn{1}{>{\columncolor{lightred}}c}{0.921}   \\
    & \multicolumn{1}{>{\columncolor{lightblue}}c|}{\textbf{ \textbf{TelePiT} }} & \multicolumn{1}{>{\columncolor{lightblue}}c}{\textbf{47.761}}   & \multicolumn{1}{>{\columncolor{lightblue}}c}{\textbf{30.405}}   & \multicolumn{1}{>{\columncolor{lightblue}}c}{\textbf{1.698}}   & \multicolumn{1}{>{\columncolor{lightblue}}c}{\textbf{1.894}}   & \multicolumn{1}{>{\columncolor{lightblue}}c}{\textbf{12.063}}   & \multicolumn{1}{>{\columncolor{lightblue}}c}{\textbf{0.490}}   & \multicolumn{1}{>{\columncolor{lightblue}}c|}{\textbf{0.466}}  & \multicolumn{1}{>{\columncolor{lightblue}}c}{\textbf{0.985}}   & \multicolumn{1}{>{\columncolor{lightblue}}c}{\textbf{0.964}}   & \multicolumn{1}{>{\columncolor{lightblue}}c}{\textbf{0.988}}   & \multicolumn{1}{>{\columncolor{lightblue}}c}{\textbf{0.990}}   & \multicolumn{1}{>{\columncolor{lightblue}}c}{\textbf{0.996}}   & \multicolumn{1}{>{\columncolor{lightblue}}c}{\textbf{0.928}}   & \multicolumn{1}{>{\columncolor{lightblue}}c}{\textbf{0.934}}  \\
    
    \bottomrule
  \end{tabular}
\let\everydisplay\oldeverydisplay
\let\everymath\oldeverymath
\end{table*}

\subsection{Overall Performance}

\subsubsection{Compared with data-driven models.}
Table \ref{tab:main} and Figure \ref{fig:heatmap} present a comprehensive comparison of \textbf{TelePiT} against state-of-the-art data-driven S2S forecasting models. Our model demonstrates consistent and substantial improvements across all meteorological variables and forecast horizons.
For geopotential heights ($z500$, $z850$), \textbf{TelePiT} reduces RMSE by 9.5\% and 8.6\% respectively compared to the CirT at Weeks 3-4, while simultaneously achieving higher ACC values. This improvement is particularly noteworthy given that these upper-atmospheric variables govern large-scale circulation patterns crucial for extended-range forecasting.


Critically, \textbf{TelePiT} retains a performance advantage even at extended lead times (weeks 5–6), exhibiting consistent skill margins across all variables. This sustained skill indicates that the model successfully captures both local physical processes and global circulation patterns across multiple temporal and spatial scales, and it underscores \textbf{TelePiT}’s robust representation of slowly evolving atmospheric phenomena that underpin sub-seasonal predictability. 

Table \ref{tab:main_SpecDiv} evaluates model performance using the SpecDiv metric, which quantifies preservation of physical energy distribution across spatial frequencies. \textbf{TelePiT} significantly outperforms all baselines, with improvements ranging from 37.1\% to 98.2\% compared to CirT. 
These results demonstrate that \textbf{TelePiT}'s multi-scale physics-informed architecture maintains exceptionally realistic physical characteristics throughout the forecast period, particularly for variables with complex boundary-layer interactions where conventional models typically struggle with spectral fidelity.

\begin{figure*}[t]
    \centering
    \includegraphics[width=1\linewidth]{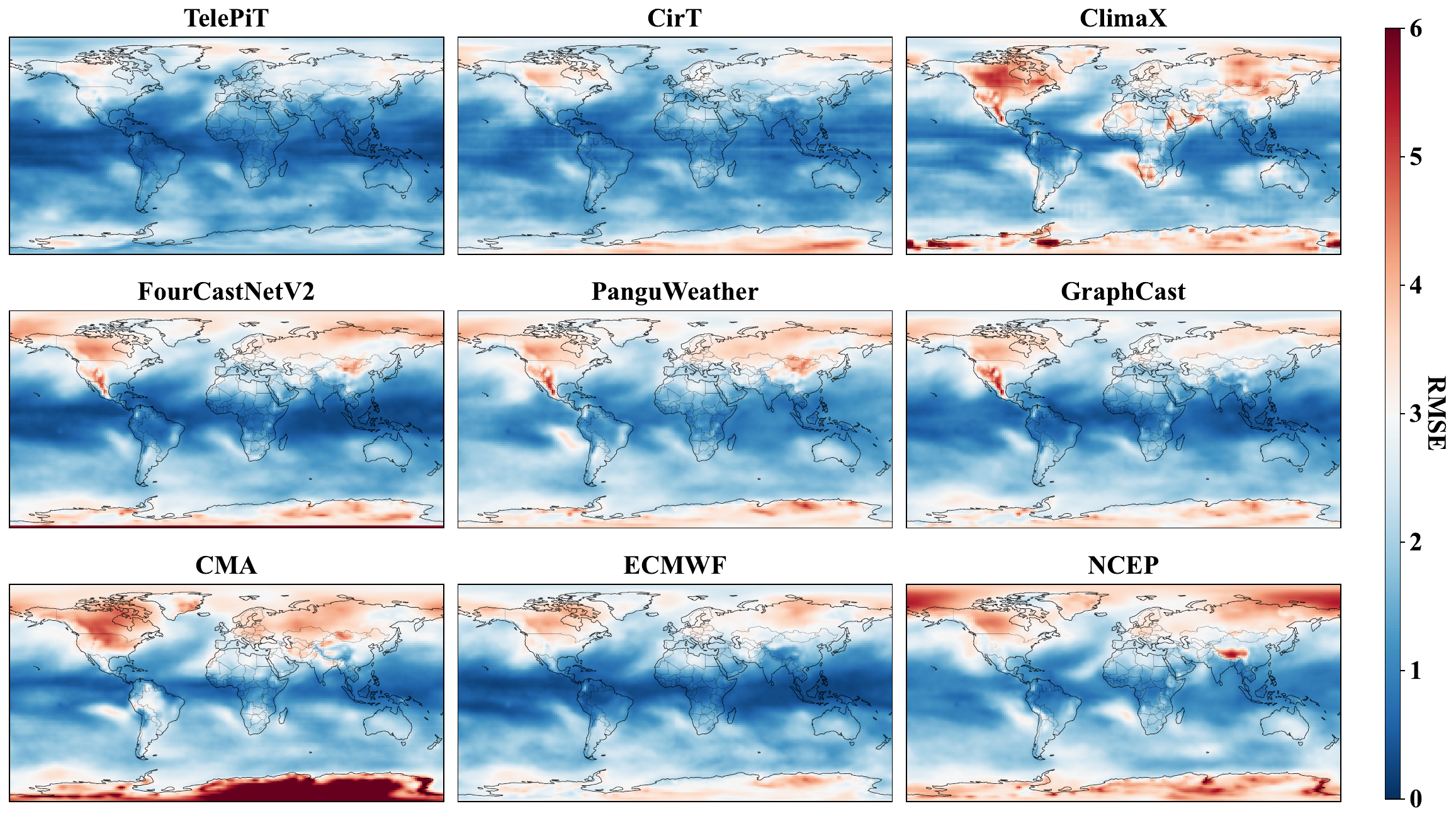}
    \caption{The global RMSE distribution of $t850$ with lead times weeks 3-4 in the testing set.}
    \label{fig:map}
\end{figure*}

\subsubsection{Compared with numerical models.}
Beyond outperforming other data-driven approaches, our evaluation against operational numerical models reveals \textbf{TelePiT}'s remarkable capabilities across the entire atmospheric column. As shown in Figure \ref{fig:heatmap}, \textbf{TelePiT} consistently achieves lower normalized RMSE (lighter colors) across almost all pressure levels (10-1000 hPa) for all variables.

Most notably, \textbf{TelePiT} demonstrates exceptional performance for temperature ($t$) forecasts throughout the troposphere and lower stratosphere, with substantial improvements over physics-based models that have benefited from decades of scientific development and operational refinement. This advantage is particularly pronounced at upper levels (10-200 hPa), where traditional models struggle to maintain forecast skill beyond week 2. 
Critically, while physics-based models show substantial skill degradation between Weeks 3-4 and Weeks 5-6 (evidenced by darker colors in the right panels), \textbf{TelePiT} maintains more consistent performance across both forecast horizons.

\vspace{-0.2cm}
\subsection{Ablation Study}

The ablation study presented in Table \ref{tab:ablation} systematically evaluates the contribution of each component in our proposed \textbf{TelePiT} architecture.
Spherical Harmonic Embedding (SHE) provides the most critical contribution, particularly for $t2m$ prediction where its removal increases RMSE from 12.057 to 27.206 at Weeks 3-4. This substantial degradation demonstrates that effectively encoding Earth's spherical geometry is fundamental for capturing planetary-scale circulation patterns. 
Wavelet Decomposition (WD) plays an essential role in separating signals across temporal frequencies. Its removal impacts geopotential height predictions ($z500$, $z850$), with RMSE increasing by 7-8\%. This confirms our hypothesis that atmospheric dynamics operate simultaneously across multiple time scales, from fast-moving synoptic systems to slow-evolving planetary waves, requiring explicit multi-scale processing for accurate prediction.

Physics-Informed ODE shows more subtle but consistent contributions. Unlike purely data-driven approaches, the integration of simplified atmospheric dynamics through neural ODEs provides physical consistency that enhances model generalization, especially for variables with complex boundary-layer interactions. 
Teleconnection Attention (TA) becomes increasingly valuable at longer lead times, with its removal causing notable degradation at Weeks 5-6. This aligns with meteorological understanding that as initial condition influence diminishes, large-scale teleconnections increasingly govern predictability. 

\begin{figure*}
    \centering
    \includegraphics[width=1\linewidth]{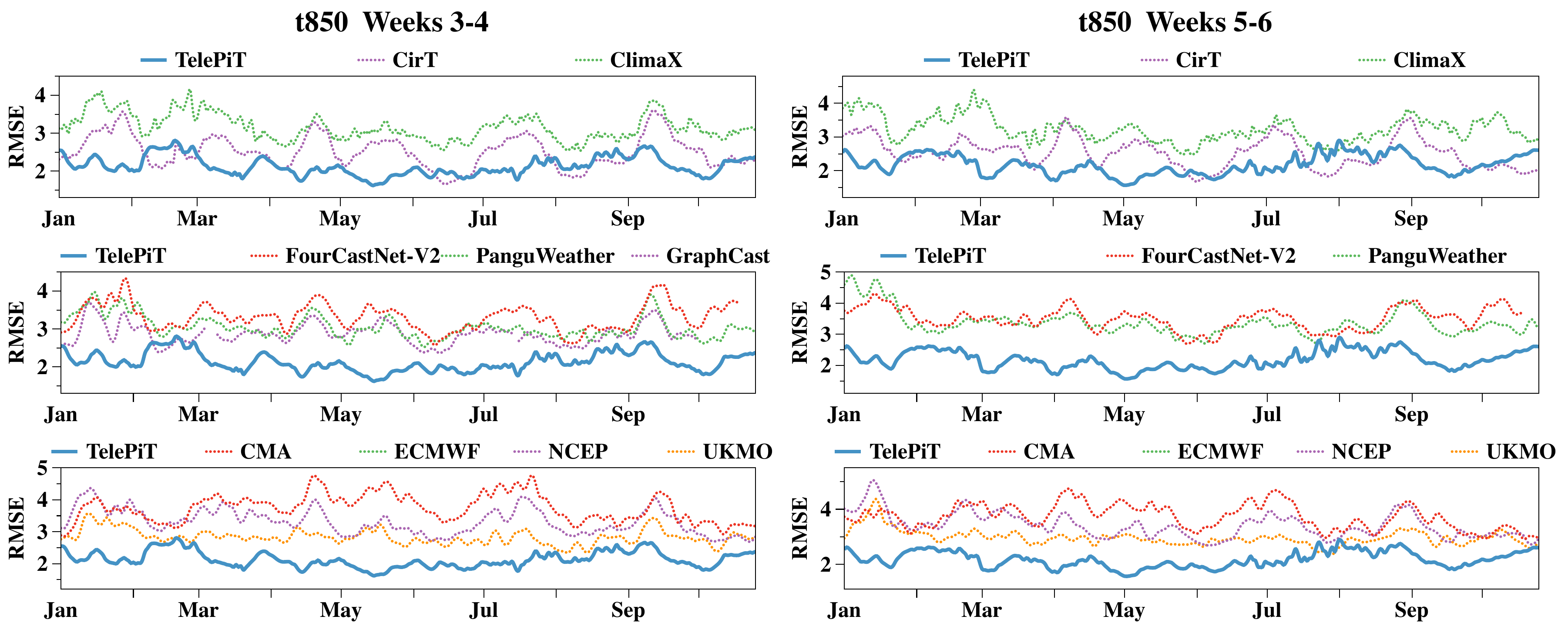}
    \vspace{-0.6cm}
    \caption{Forecasting Performance of $t850$ on each day in the testing set.}
    \label{fig:line}
\end{figure*}

\vspace{-0.2cm}
\subsection{Empirical Analysis}

\textbf{Global visualization.}
Figure \ref{fig:map} visualizes the global distribution of RMSE for $t850$ at Weeks 3-4. TelePiT demonstrates remarkably consistent low error (deep blue) across nearly all geographical regions, including traditionally challenging areas like continental interiors and the polar regions. 
This spatial analysis confirms that \textbf{TelePiT}'s architectural innovations, especially its spherical harmonic embedding and teleconnection attention mechanisms, enable it to effectively capture both tropical atmospheric dynamics and mid-to-high latitude variability, areas where existing models often struggle due to fundamentally different governing physical processes. 

\textbf{Forecasting visualization.}
Figure \ref{fig:line} presents the $t850$ forecast performance on each day. \textbf{TelePiT} demonstrates consistently lower RMSE values throughout the annual cycle for both Weeks 3-4 and Weeks 5-6 forecasts, maintaining this advantage across all seasonal transitions. 
The performance gap widens even further in the more challenging Weeks 5-6 timeframe, where \textbf{TelePiT} maintains relatively stable error metrics while baselines show increased variability and deterioration. 
This consistent temporal superiority complements the spatial advantages described in the global visualization analysis, suggesting that \textbf{TelePiT}'s architectural design effectively captures both the spatial and temporal aspects of atmospheric dynamics critical for extended-range forecasting.

\section{Related Works}

S2S forecasting has traditionally relied on sophisticated numerical weather prediction (NWP) systems \cite{vitart2017subseasonal,white2017potential}, which integrate atmospheric and oceanic physics through coupled dynamical models. These systems, exemplified by ECMWF \cite{molteni1996ecmwf}, UKMO \cite{williams2015met}, NCEP \cite{saha2014ncep}, and CMA \cite{wu2019beijing}, have gradually extended forecast horizons through enhanced initialization techniques and ensemble methods \cite{vitart2018sub}. However, these methods remain computationally intensive and often struggle to fully capture complex teleconnection patterns that significantly influence predictability beyond the 2-week deterministic limit \cite{bauer2015quiet,robertson2020subseasonal}. Recent advancements utilizing empirical methods to extract large-scale modes of variability \cite{cohen2019s2s,merryfield2020current} have shown promise but lack physical dynamics across scales.

Deep learning has revolutionized climate prediction. Foundational work by Rasp et al. \cite{rasp2020weatherbench} established benchmarks for data-driven atmospheric forecasting, while subsequent architectures including FourCastNetV2 \cite{kurth2023fourcastnet}, Pangu-Weather \cite{bi2023accurate}, GraphCast \cite{lam2023learning}, and ClimaX \cite{nguyen2023climax} demonstrated competitive performance against operational NWP systems for medium-range forecasting. Regional S2S forecasting has seen notable advances through specialized models like SubseasonalClimateUSA \cite{mouatadid2023subseasonalclimateusa}, which focuses on North American climate predictions, and approaches by He et al. \cite{he2022learning} and Hwang et al. \cite{hwang2019improving} that target localized weather patterns. However, these regional models often overlook crucial global teleconnections. Recent global S2S efforts include FuXi-S2S \cite{chen2024machine}, which employs an encoder-decoder framework for global precipitation forecasting, and CirT \cite{liucirt}, which incorporates spherical geometry. Nevertheless, these approaches lack explicit incorporation of physical constraints and multi-scale atmospheric processes, limiting their ability to maintain skill at extended lead times \cite{nathaniel2024chaosbench}.

Physics-informed neural networks (PINNs) have shown increasing promise for integrating physical knowledge into data-driven models \cite{raissi2019physics,karniadakis2021physics}. In climate science, Reichstein et al. \cite{reichstein2019deep} and Willard et al. \cite{willard2022integrating} established frameworks for integrating scientific knowledge with machine learning for Earth system modeling. Recent work on climate modeling with neural advection-diffusion equations \cite{choi2023climate} has demonstrated how neural networks can parameterize complex climate dynamics while preserving physical principles. Complementary approaches in machine learning-accelerated computational fluid dynamics \cite{kochkov2021machine} have shown that neural networks can significantly enhance traditional solvers while maintaining physical consistency. Unlike these works, which typically focus on individual atmospheric processes or fixed spatial scales, \textbf{TelePiT} uniquely combines learnable multi-scale decomposition with physics-informed ODEs that adaptively capture scale-dependent dynamics. For teleconnection modeling, previous efforts \cite{cohen2021linking} have explored statistical approaches, while Ham et al. \cite{ham2019deep} pioneered deep learning for El Niño prediction. \textbf{TelePiT} advances beyond these approaches by seamlessly integrating scale-aware physics with teleconnection-aware attention mechanisms, explicitly capturing the cross-scale interactions and global patterns that are crucial for S2S predictability while preserving Earth's spherical geometry in a unified end-to-end architecture.

\section{Conclusion}

We presented \textbf{TelePiT}, a novel deep learning architecture for global S2S forecasting that integrates physical principles with data-driven learning. By combining Spherical Harmonic Embedding, Multi-Scale Physics-Informed ODE, and Teleconnection-Aware Transformer, we explicitly model the fundamental processes governing S2S predictability.
Our evaluation shows that \textbf{TelePiT} consistently outperforms both state-of-the-art data-driven models and operational numerical systems across all meteorological variables and forecast horizons. The most significant improvements occur in near-surface variables, particularly 2-meter temperature, with over 57\% RMSE reduction compared to previous models. Ablation studies confirm each component's meaningful contribution, with their combined effect becoming more pronounced at extended forecasting ranges.
This advances climate resilience applications in agriculture, energy, and disaster preparedness. Future research directions include incorporating additional Earth system components and specialized datasets such as the Arctic System Reanalysis (ASR) to enhance regional performance, ultimately bridging the gap between weather and climate prediction scales.


\bibliographystyle{ACM-Reference-Format}
\balance
\bibliography{sample-base}

\appendix
\onecolumn  

\begin{center}
    \LARGE\textbf{Physics-Informed Teleconnection-Aware Transformer for \\ Global Subseasonal-to-Seasonal Forecasting}
\end{center}

\vspace{0.5cm}

\begin{center}
\large\textbf{Appendix Contents}
\end{center}

\vspace{0.5cm}

\begin{list}{}{\setlength{\leftmargin}{2em}\setlength{\itemindent}{-1em}}
\item \textbf{A. Theoretical Analysis} \dotfill \pageref{appendix_sec:Theoretical Analysis}
  \begin{list}{}{\setlength{\leftmargin}{1em}}
  \item A.1. Spectral Representation on the Sphere \dotfill \pageref{appendix_sec:Spectral Representation on the Sphere}
  \item A.2. Multi-Scale Analysis and Function Approximation \dotfill \pageref{appendix_sec:Multi-Scale Analysis and Function Approximation}
  \item A.3. Physics-Informed Neural ODEs \dotfill \pageref{appendix_sec:Physics-Informed Neural ODEs}
  \item A.4. Teleconnection Representation Theory \dotfill \pageref{appendix_sec:Teleconnection Representation Theory}
  \item A.5. Complexity and Generalization Analysis \dotfill \pageref{appendix_sec:Complexity and Generalization Analysis}
  \end{list}
\item \textbf{B. Additional Experimental Details} \dotfill \pageref{appendix_sec:Additional Experimental Details}
  \begin{list}{}{\setlength{\leftmargin}{1em}}
  \item B.1. Evaluation Metric \dotfill \pageref{appendix_sec:Evaluation Metric}
  \item B.2. Baselines \dotfill \pageref{appendix_sec:Baselines}
  \item B.3. Model Complexity \dotfill \pageref{appendix_sec:Model Complexity}
  \item B.4. Additional Overall Performance \dotfill \pageref{appendix_sec:Additional Overall Performance}
  \item B.5. Additional Ablation Study \dotfill \pageref{appendix_sec:Additional Ablation Study}
  \item B.6. Additional Empirical Analysis \dotfill \pageref{appendix_sec:Additional Empirical Analysis}
  \item B.7. Robustness Study \dotfill \pageref{appendix_sec:Robustness Study}
  \item B.8. Parameter Sensitivity Analysis \dotfill \pageref{appendix_sec:Parameter Sensitivity Analysis}
  \end{list}
\end{list}

\vspace{-10cm}

\begin{figure}[!b]  

    \centering
    \begin{subfigure}[t]{0.32\textwidth}
        \centering
        \includegraphics[width=\textwidth]{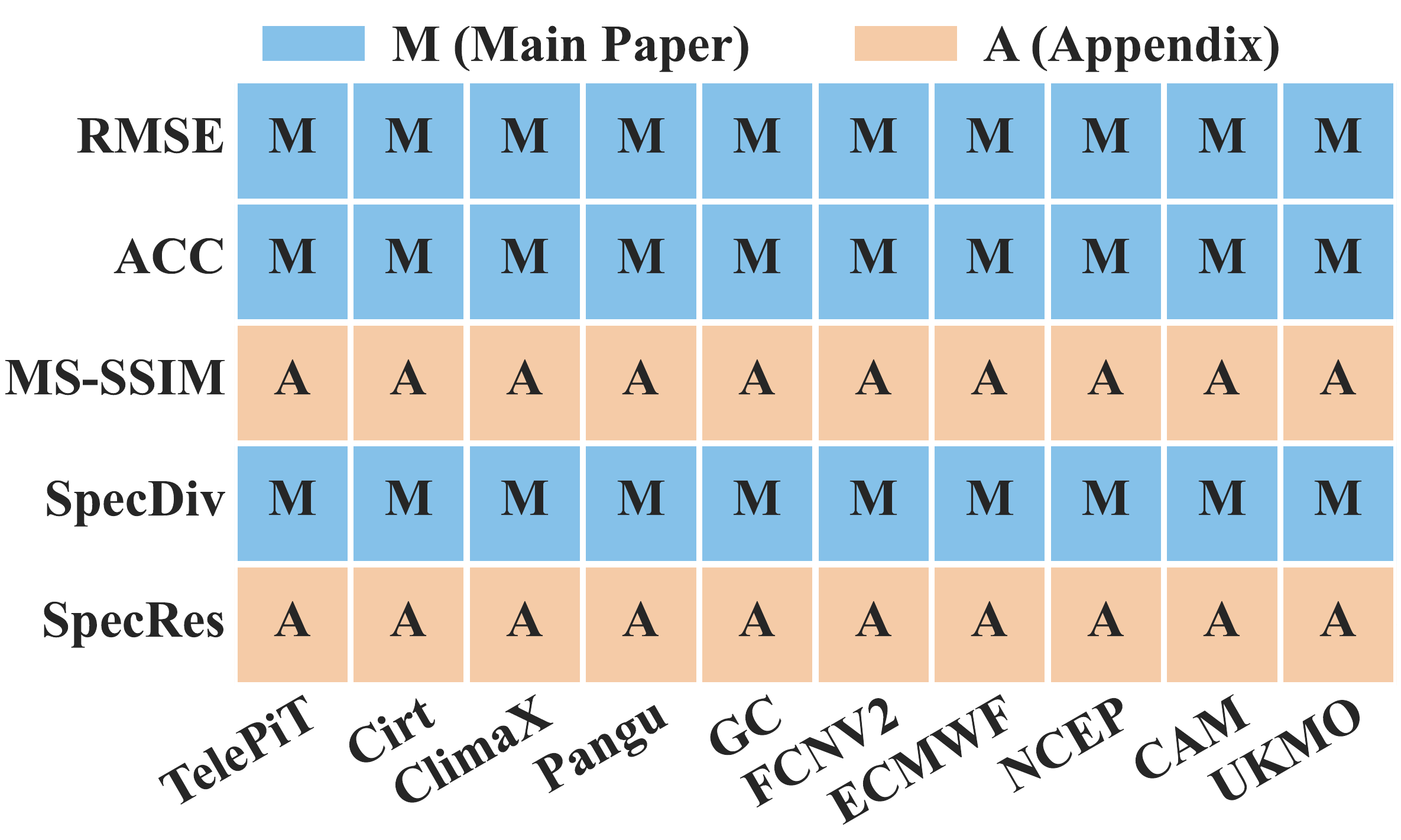}
        \caption{Distribution of overall performance for key variables $z500$, $z850$, $t500$, $t850$, $t2m$, $u10$, $v10$.}
        \label{appendix_fig:all_exp_overall_performance}
    \end{subfigure}
    \hfill
    \begin{subfigure}[t]{0.32\textwidth}
        \centering
        \includegraphics[width=\textwidth]{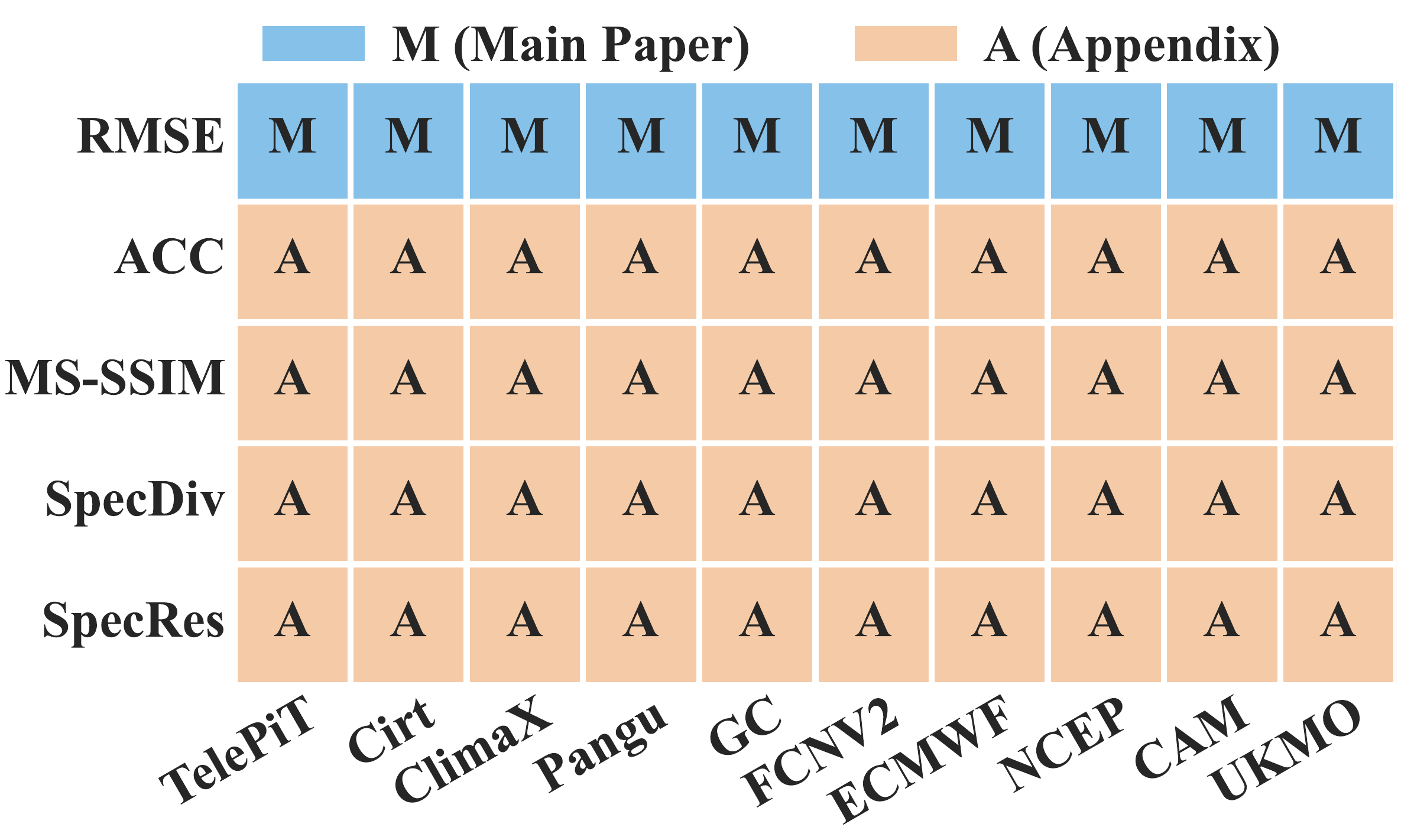}
        \caption{Distribution of heatmap visualizations for variables ($z$, $t$, $u$, $v$) across all pressure levels.}
        \label{appendix_fig:all_exp_heatmap}
    \end{subfigure}
    \hfill
    \begin{subfigure}[t]{0.32\textwidth}
        \centering
        \raisebox{12pt}{\includegraphics[width=\textwidth]{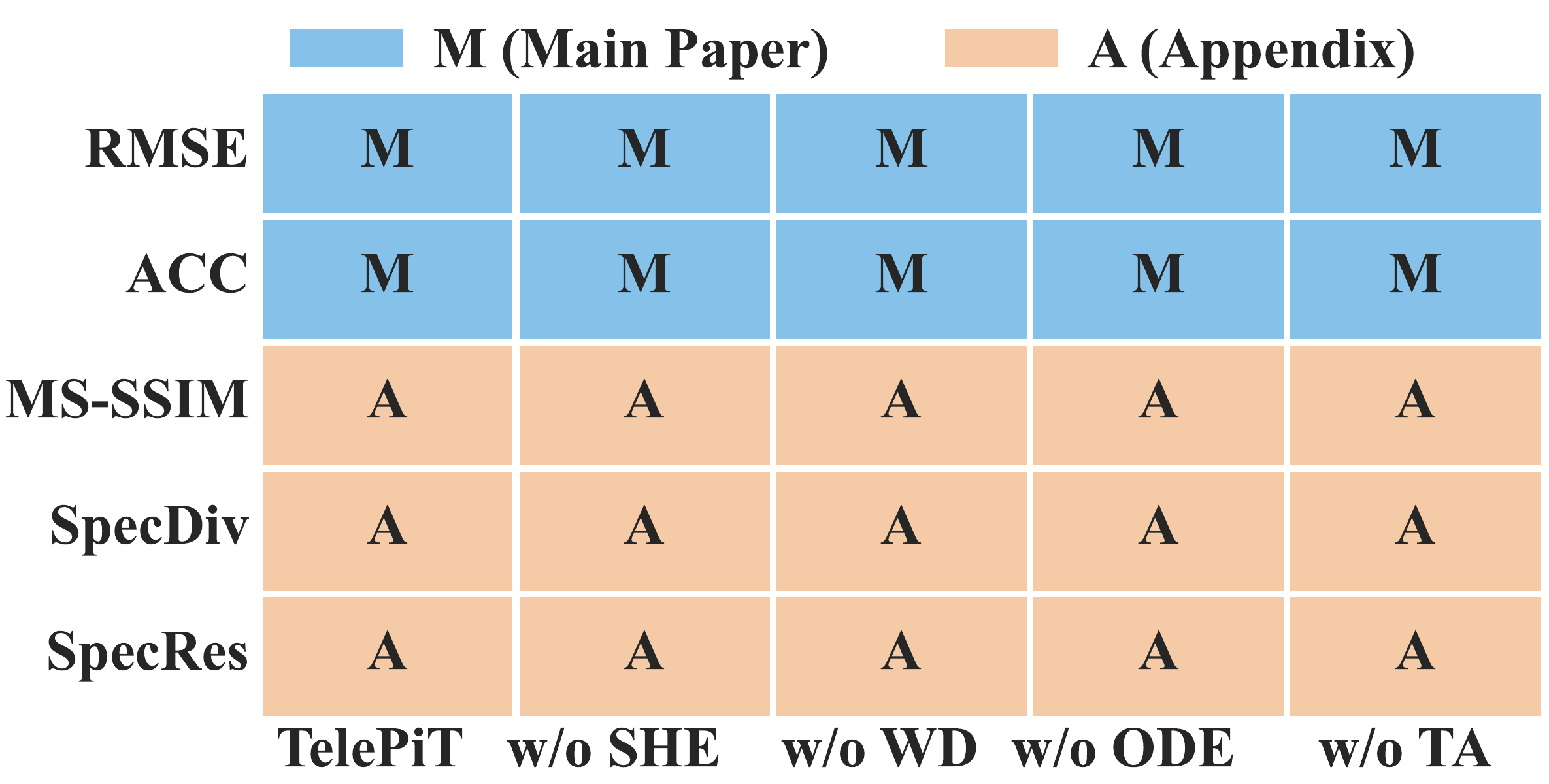}}
        \caption{Distribution of ablation studies comparing model variants across all metrics.}
        \label{appendix_fig:all_exp_ablation_study}
    \end{subfigure}

    \vspace{1cm}
    
    \begin{subfigure}[t]{0.43\textwidth}
        \centering
        \includegraphics[width=\textwidth]{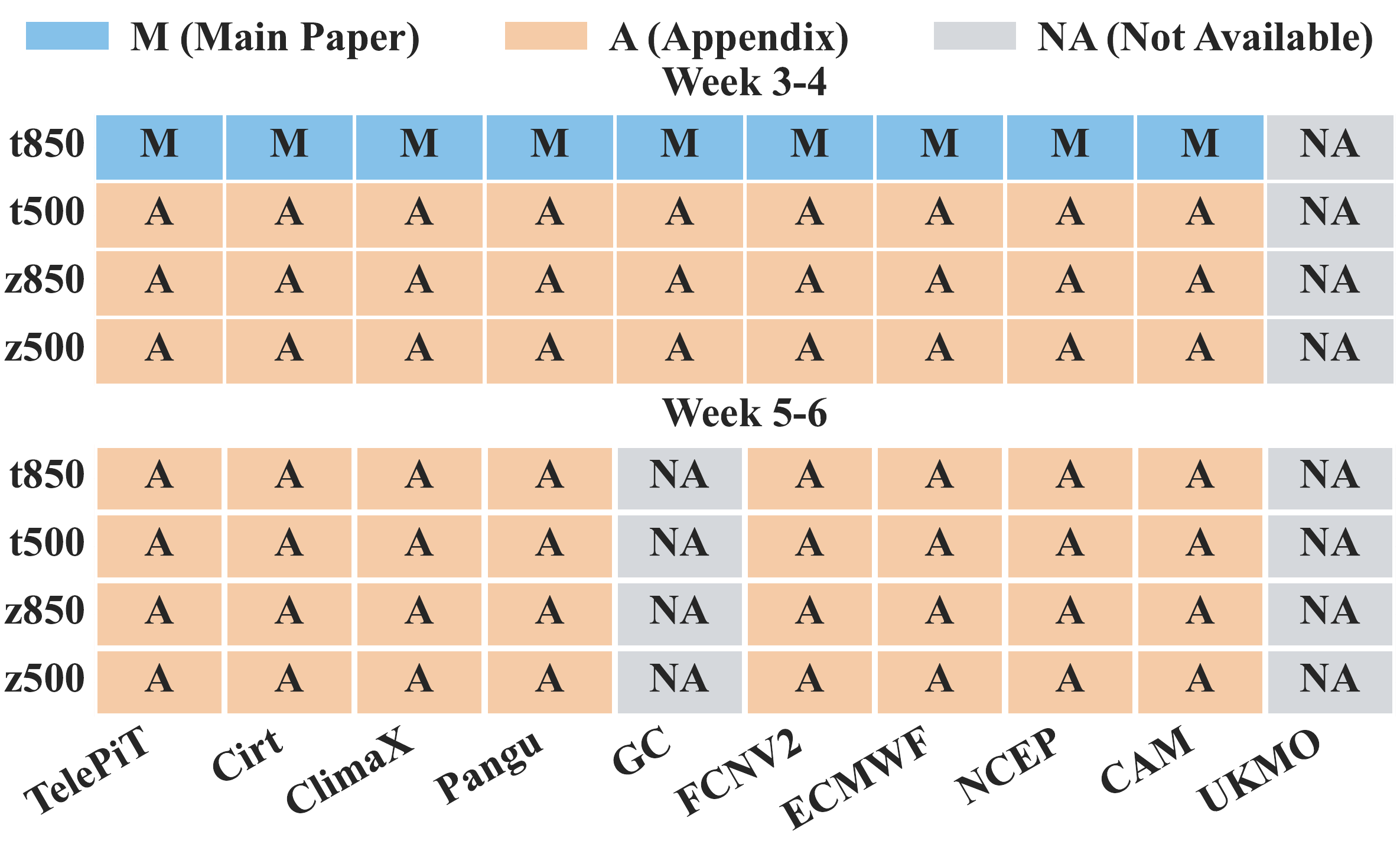}
        \caption{Distribution of global visualizations of RMSE across variables and lead times.}
        \label{appendix_fig:all_exp_map}
    \end{subfigure}
    \hfill
    \begin{subfigure}[t]{0.43\textwidth}
        \centering
        \includegraphics[width=\textwidth]{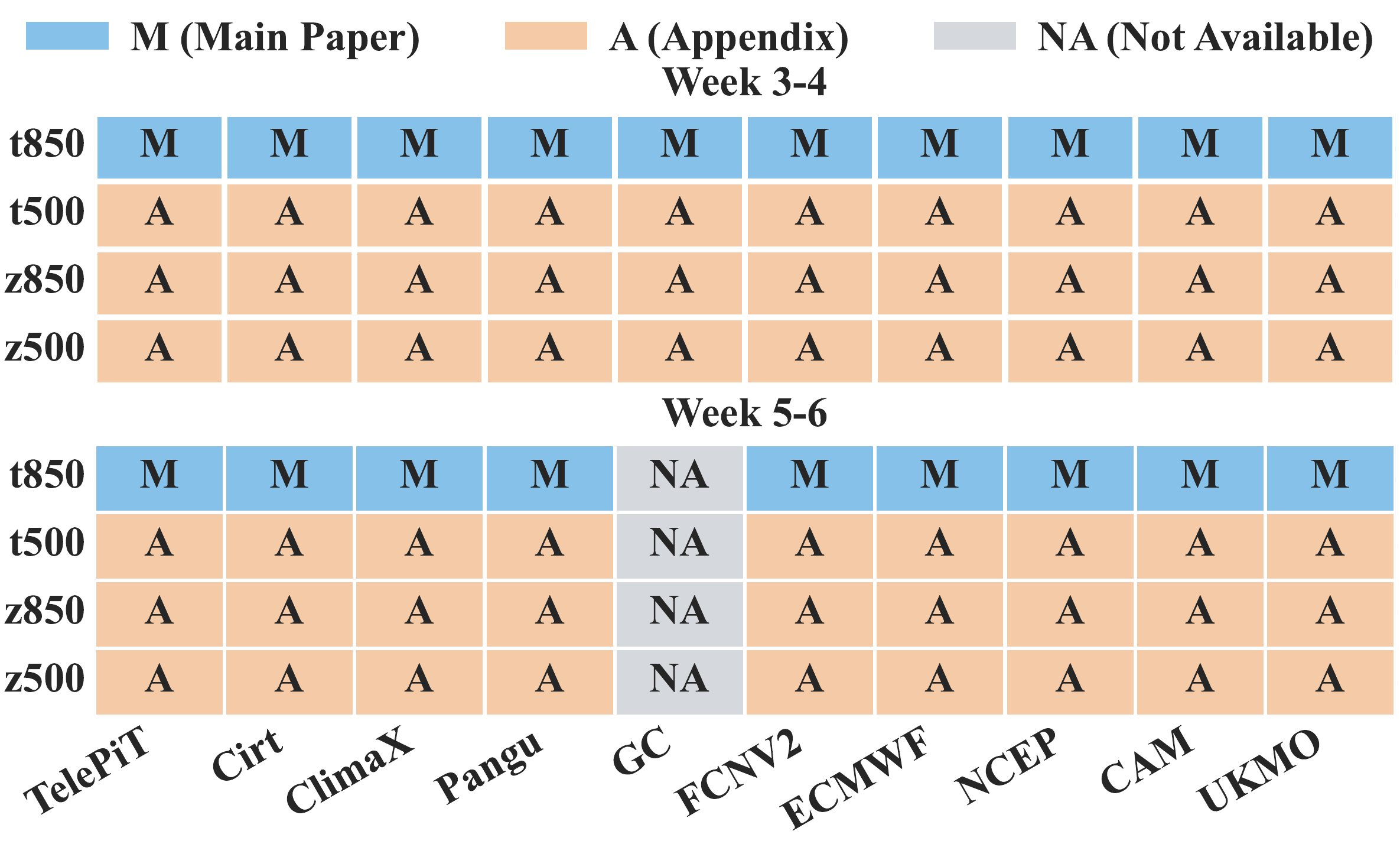}
        \caption{Distribution of daily forecasting performance (RMSE) in the test 2018.}
        \label{appendix_fig:all_exp_line}
    \end{subfigure}
    \vspace{-0.2cm}
    \caption{Experiments distribution between main paper and appendix. Note that GraphCast (GC) results are unavailable at Weeks 5-6 due to out-of-memory issues during inference. UKMO's global RMSE visualizations are not shown due to significant missing values in their inference results.}
    \label{appendix_fig:dif}
\end{figure}

\twocolumn

\textbf{Overview.}
The main paper presents core results including {\color[RGB]{169, 50, 38}\textbf{overall performance}} of RMSE and ACC metrics for key variables~($z500$, $z850, t500, t850, t2m, u10, v10$) {\color[RGB]{169, 50, 38}\textbf{[Table \ref{tab:main}]}}, SpecDiv for physics-based evaluation {\color[RGB]{169, 50, 38} \textbf{[Table \ref{tab:main_SpecDiv}]}}, {\color[RGB]{36, 113, 163} \textbf{heatmap comparisons}} across pressure levels {\color[RGB]{36, 113, 163} \textbf{[Figure \ref{fig:heatmap}]}}, {\color[RGB]{23, 165, 137} \textbf{ablation study}} on RMSE {\color[RGB]{23, 165, 137} \textbf{[Table \ref{tab:ablation}]}}, {\color[RGB]{214, 137, 16} \textbf{global visualization}} of RMSE distribution for $t850$ {\color[RGB]{214, 137, 16} \textbf{[Figure \ref{fig:map}]}}, and {\color[RGB]{125, 60, 152}\textbf{daily forecasting}} performance visualization for $t850$ {\color[RGB]{125, 60, 152}\textbf{[Figure \ref{fig:line}]}}. Figure \ref{appendix_fig:dif} offers an overview of how experiments are distributed between the main paper and appendix.

This appendix significantly extends these analyses with:
\begin{enumerate}
    \item \textbf{Additional Theoretical Analysis}: We provide in-depth mathematical foundations for each component of \textbf{TelePiT}, including spectral representations on the sphere, multi-scale analysis, physics-informed neural ODEs, and teleconnection representation theorems.
    
    \item \textbf{Model Complexity Analysis}: We provide a comprehensive comparison of computational requirements across several key baselines {\color[RGB]{169, 50, 38}\textbf{[Table \ref{tab:model_complexity}]}}.
    
    \item \textbf{Expanded Metrics}: Beyond RMSE, ACC, and SpedDiv, we evaluate using additional MS-SSIM for structural similarity and SpecRes for spectral accuracy {\color[RGB]{169, 50, 38}\textbf{[Table \ref{appendix_tab:main}]}}.

    \item \textbf{Operational System Comparisons}: We present detailed comparisons against ECMWF, UKMO, NCEP, CMA across all metrics  {\color[RGB]{169, 50, 38}\textbf{[Table \ref{appendix_tab:main_2}]}}.
    
    \item \textbf{Comprehensive Variable Coverage}: We present detailed results for all pressure levels and variables, including comprehensive heatmaps for ACC {\color[RGB]{36, 113, 163} \textbf{[Figure \ref{appendix_fig:heatmap_acc}]}}, MS-SSIM {\color[RGB]{36, 113, 163} \textbf{[Figure \ref{appendix_fig:heatmap_MS_SSIM}]}}, SpecDiv {\color[RGB]{36, 113, 163} \textbf{[Figure \ref{appendix_fig:heatmap_specdiv}]}}, SpecRes {\color[RGB]{36, 113, 163}\textbf{[Figure \ref{appendix_fig:heatmap_specres}]}} metrics.
    
    \item \textbf{Extended Ablation Studies}: We analyze component contributions using additional multiple metrics (MS-SSIM, SpecDiv, SpecRes) in {\color[RGB]{23, 165, 137} \textbf{[Table \ref{appendix_tab:ablation}]}}, providing deeper insights into how each architectural element affects different aspects of forecast quality.
    
    \item \textbf{Additional Global Visualizations}: We expand global visualizations of RMSE distribution to cover key variables ($t850$, $t500$, $z850$, $z500$) for both forecast horizons (Weeks 3-4 and Weeks 5-6) {\color[RGB]{214, 137, 16} \textbf{[Figure \ref{appendix_fig:map_t850_1},\ref{appendix_fig:map_t500_0},\ref{appendix_fig:map_t500_1},\ref{appendix_fig:map_z850_0},\ref{appendix_fig:map_z850_1},\ref{appendix_fig:map_z500_0},\ref{appendix_fig:map_z500_1}]}}.
    
    \item \textbf{Daily Forecasting Performance}: We include additional daily forecasting performance (RMSE) for $t500$ {\color[RGB]{125, 60, 152}\textbf{[Figure \ref{appendix_fig:t500_rmse_2018}]}}, $z850$ {\color[RGB]{125, 60, 152}\textbf{[Figure \ref{appendix_fig:z850_rmse_2018}]}} and $z500$ {\color[RGB]{125, 60, 152}\textbf{[Figure \ref{appendix_fig:z500_rmse_2018}]}} in 2018, demonstrating \textbf{TelePiT}'s consistent advantages across seasonal transitions.

    \item \textbf{Robustness Evaluation}: We provide out-of-sample validation using 2019 data across all metrics, confirming \textbf{TelePiT}'s strong generalization capabilities and consistent performance advantages {\color[RGB]{125, 102, 8}\textbf{[Table \ref{appendix_tab:robustness}]}}.
    
    \item \textbf{Parameter Sensitivity Analysis}: We analyze the impact of the teleconnection coefficient $\lambda$ on model performance, revealing optimal parameter settings and validating architectural design choices {\color[RGB]{125, 102, 8}\textbf{[Figure \ref{appendix_fig:lamda}]}}.
\end{enumerate}

\section{Theoretical Analysis}
\label{appendix_sec:Theoretical Analysis}

In this section, we provide a theoretical analysis of \textbf{TelePiT} and establish mathematical connections to fundamental principles in both machine learning and atmospheric science. 

\subsection{Spectral Representation on the Sphere}
\label{appendix_sec:Spectral Representation on the Sphere}

The Spherical Harmonic Embedding (SHE) in our model builds upon the well-established spectral theory for functions on the sphere. Any square-integrable function $f(\theta, \phi)$ on the sphere can be represented as an infinite series of spherical harmonics:
\begin{equation}
f(\theta, \phi) = \sum_{\ell=0}^{\infty} \sum_{m=-\ell}^{\ell} a_{\ell m} Y_{\ell}^m(\theta, \phi)
\end{equation}
where $Y_{\ell}^m(\theta, \phi)$ are the spherical harmonics of degree $\ell$ and order $m$, and $a_{\ell m}$ are the spectral coefficients. 
Our learnable zonal averaging approach in SHE provides a computationally efficient approximation to the important subset of zonal harmonics ($m = 0$). This connection is justified by the observations in atmospheric science that many large-scale circulation patterns exhibit strong zonal symmetry \cite{baldwin2003stratospheric, kidston2015stratospheric}.

\begin{proposition}[Zonal Projection Property]
The zonal averaging operation implements an orthogonal projection onto the subspace of zonally symmetric functions. For a function $f(\theta,\phi)$, the zonal average $\bar{f}(\theta) = \frac{1}{2\pi}\int_{-\pi}^{\pi} f(\theta,\phi) d\phi$ projects onto the $m=0$ spherical harmonic modes:
\begin{equation}
\bar{f}(\theta) = \sum_{\ell=0}^{\infty} a_{\ell,0} Y_{\ell}^0(\theta)
\end{equation}
where $Y_{\ell}^0(\theta) = \sqrt{\frac{2\ell+1}{4\pi}} P_{\ell}(\cos\theta)$ are the zonal harmonics and $P_{\ell}$ are the Legendre polynomials.
\end{proposition}

\begin{proposition}[Approximation Quality]
The discrete zonal averaging in SHE approximates the continuous projection with error bounded by:
\begin{equation}
\left\|\mathbf{u}_i - \sum_{\ell=0}^{L} a_{\ell,0} P_{\ell}(\cos\theta_i)\right\| \leq \mathcal{O}(D^{-1/2} + L^{-1})
\end{equation}
where $\mathbf{u}_i = \frac{1}{W}\sum_{j=1}^W \mathbf{X}_{i,j}$ is the empirical zonal average and the approximation improves with embedding dimension $D$ and truncation level $L$.
\end{proposition}

This selective spectral focus allows \textbf{TelePiT} to efficiently represent large-scale atmospheric patterns while maintaining computational tractability, with theoretical guarantees on the approximation quality.

\subsection{Multi-Scale Analysis and Function Approximation}
\label{appendix_sec:Multi-Scale Analysis and Function Approximation}

The Wavelet Decomposition (WD) component implements a learnable multi-scale analysis that generalizes traditional wavelet transforms. Mathematically, our approach creates a data-adaptive basis that separates signals across different temporal and spatial frequencies.

\textbf{Assumption 1} (Signal Regularity). We assume that atmospheric signals possess sufficient regularity, specifically Hölder continuity with exponent $\alpha > 0$, enabling the stated approximation rates. This assumption is well-justified for large-scale atmospheric variables due to the underlying physics.

\begin{proposition}[Task-Adaptive Decomposition]
Under Assumption 1, the learnable wavelet decomposition converges to a multi-scale representation that minimizes the S2S forecasting objective, achieving decomposition error bounded by:
\begin{equation}
\mathbb{E}[\|\mathbf{x} - \sum_{\ell=0}^L \mathbf{x}^{(\ell)}\|^2] \leq C \cdot L^{-\alpha}
\end{equation}
where $C$ is a constant depending on the signal characteristics and $\alpha > 0$ reflects the smoothness of atmospheric signals. $L$ is the number of frequency bands.
\end{proposition}

\begin{proof}[Proof sketch]
The MLP transforms at each level provide universal function approximation capabilities with sufficient width. By minimizing the end-to-end prediction loss, these MLPs learn to extract frequency components that maximize predictive skill rather than reconstruction fidelity. The recursive binary splitting operation ensures that the $L+1$ frequency bands form a partition of the signal's spectral content. Under the regularity assumption, this adaptive partitioning achieves the stated convergence rate through standard approximation theory arguments.
\end{proof}

\begin{proposition}[Information Preservation]
Each frequency band $\ell$ preserves mutual information with the forecast target, satisfying:
\begin{equation}
\sum_{\ell=0}^L I(\mathbf{x}^{(\ell)}; \mathbf{y}) \geq (1-\delta) I(\mathbf{x}; \mathbf{y})
\end{equation}
for arbitrarily small $\delta > 0$ given sufficient model capacity, where $I(\cdot;\cdot)$ denotes mutual information.
\end{proposition}

This adaptability is critical for S2S forecasting, as traditional fixed wavelet bases may not optimally separate the complex, non-stationary atmospheric signals that span multiple scales.

\subsection{Physics-Informed Neural ODEs}
\label{appendix_sec:Physics-Informed Neural ODEs}

The Physics-Informed Neural ODE component provides a learnable framework that incorporates fundamental atmospheric transport processes as inductive biases for latent space dynamics. While we do not claim to implement specific atmospheric equations, our formulation draws inspiration from key physical principles governing atmospheric evolution.

\textbf{Latent Space Transport Dynamics:} We model the evolution of latent representations $\mathbf{x}_i \in \mathbb{R}^{D}$ along the latitude dimension using a transport-inspired ODE:
\begin{equation}
\frac{d \mathbf{x}_i}{dt} = \gamma \cdot \tanh(\mathbf{R}_i)
\end{equation}
\begin{equation}
{\small
\mathbf{R}_i = {\color[RGB]{2, 119, 189}\underbrace{\boldsymbol{\nu}\odot(\mathbf{x}_{i+1} - 2\mathbf{x}_i + \mathbf{x}_{i-1})}_{\text{diffusion}}} + 
{\color[RGB]{216, 67, 21}\underbrace{\boldsymbol{\mu}\odot\frac{\mathbf{x}_{i+1} - \mathbf{x}_{i-1}}{2}}_{\text{advection}}} + 
{\color[RGB]{48, 63, 159}\underbrace{\boldsymbol{f}}_{\text{forcing}}} + 
{\color[RGB]{0, 137, 123}\underbrace{\alpha \cdot \text{MLP}(\mathbf{x}_i)}_{\text{neural correction}}}
}
\end{equation}

\begin{proposition}[Latent Space Transport Interpretation]
The general form of our transport-inspired neural ODE decomposes as:
\begin{equation}
\frac{d\mathbf{x}}{dt} = \mathcal{T}_{transport}(\mathbf{x};\boldsymbol{\theta}_{phys}) + \mathcal{T}_{neural}(\mathbf{x};\boldsymbol{\theta}_{nn})
\end{equation}
where $\mathcal{T}_{transport}$ captures simplified transport processes with learnable physics parameters $\boldsymbol{\theta}_{phys} = \{\boldsymbol{\nu}, \boldsymbol{\mu}, \boldsymbol{f}\} \in \mathbb{R}^{3D}$, and $\mathcal{T}_{neural}$ provides additional modeling capacity through neural network parameters $\boldsymbol{\theta}_{nn}$.
\end{proposition}

\textbf{Design Rationale}: Our approach establishes correspondence with atmospheric transport processes:
\begin{equation}
\begin{aligned}
{\color[RGB]{2, 119, 189}\boldsymbol{\nu}\odot(\mathbf{x}_{i+1} - 2\mathbf{x}_i + \mathbf{x}_{i-1})} &\longleftrightarrow \text{Meridional mixing and smoothing} \\
{\color[RGB]{216, 67, 21}\boldsymbol{\mu}\odot\frac{\mathbf{x}_{i+1} - \mathbf{x}_{i-1}}{2}} &\longleftrightarrow \text{Meridional transport processes} \\
{\color[RGB]{48, 63, 159}\boldsymbol{f}} &\longleftrightarrow \text{External climate drivers} \\
{\color[RGB]{0, 137, 123}\alpha \cdot \text{MLP}(\mathbf{x}_i)} &\longleftrightarrow \text{Unresolved nonlinear dynamics}
\end{aligned}
\end{equation}

\begin{proposition}[Numerical Stability]
The transport-inspired ODE with hyperbolic tangent activation is globally Lipschitz continuous with constant $L = \gamma$, ensuring controlled error propagation:
\begin{equation}
\left\|\frac{d\mathbf{x}}{dt} - \frac{d\mathbf{x}'}{dt}\right\| \leq \gamma \|\mathbf{x} - \mathbf{x}'\|
\end{equation}
for any latent states $\mathbf{x}, \mathbf{x}' \in \mathbb{R}^{HD}$. This global Lipschitz property guarantees existence and uniqueness of solutions and supports numerical stability during integration.
\end{proposition}

\begin{proposition}[Parameter Efficiency]
The physics-inspired structure reduces the effective parameter space by constraining the dynamics to transport-consistent manifolds. Specifically, the constraint reduces the degrees of freedom from $\mathcal{O}(HD^2)$ (unconstrained) to $\mathcal{O}(HD + D^2)$ (transport-constrained), potentially improving sample efficiency.
\end{proposition}

\textbf{Implementation Considerations}: We apply periodic boundary conditions consistent with our latitude-based processing and employ adaptive step-size integration methods for numerical accuracy. The scaling factor $\gamma$ and hyperbolic tangent activation prevent gradient explosion while preserving the fundamental transport structure.

We acknowledge that our latent space dynamics are inspired by, rather than derived from, specific atmospheric equations. The correspondence between latent dimensions and physical quantities emerges through the training process rather than being explicitly enforced.

\subsection{Teleconnection Representation Theory}
\label{appendix_sec:Teleconnection Representation Theory}

The Teleconnection-Aware Transformer implements a mechanism to learn and leverage global patterns in atmospheric data. We establish its theoretical connection to atmospheric teleconnections and provide convergence guarantees.

\begin{proposition}[Nonlinear Teleconnection Learning]
The learned teleconnection patterns $\{\mathbf{P}_j\}_{j=1}^{n_p} \subset \mathbb{R}^{HD}$ converge to a basis that captures the dominant modes of atmospheric covariance, encompassing both linear relationships and nonlinear teleconnection patterns. Formally:
\begin{equation}
\lim_{T \to \infty} \text{span}(\{\mathbf{P}_j\}) \supseteq \text{span}(\{\text{EOF}_j^{(linear)}\})
\end{equation}
where the inclusion becomes equality when the attention mechanism reduces to linear operations.
\end{proposition}

\begin{proof}[Proof sketch]
The teleconnection attention computes pattern weights as $\boldsymbol{\omega} = \text{softmax}(\mathbf{\bar{x}}\mathbf{W}^p)$ where $\mathbf{\bar{x}} \in \mathbb{R}^D$ is the global mean state. During training, the patterns $\mathbf{P}_j$ and projection matrix $\mathbf{W}^p \in \mathbb{R}^{D \times n_p}$ are jointly optimized to minimize prediction error.

This optimization is equivalent to finding patterns that maximize the predictive information content, which necessarily includes the variance-maximizing directions captured by linear EOFs. The nonlinear attention mechanism extends this to capture higher-order statistical relationships, ensuring the learned subspace contains the linear EOF subspace.
\end{proof}

\begin{proposition}[Sample Complexity for Teleconnection Learning]
To learn $k$ dominant teleconnection patterns with approximation error $\varepsilon$ in the $L_2$ norm, the teleconnection module requires $\mathcal{O}(k \log(D)/\varepsilon^2)$ training samples, where the logarithmic dependence on dimension reflects the low-intrinsic-dimensional nature of atmospheric teleconnections.
\end{proposition}

\begin{proposition}[Convergence Rate]
Under standard regularity conditions on the atmospheric data distribution, the teleconnection patterns converge to their population optima at rate $\mathcal{O}(T^{-1/2})$ where $T$ is the number of training iterations, matching the convergence rate of stochastic gradient methods for smooth objectives.
\end{proposition}

This theoretical framework connects our learning-based approach with traditional techniques while extending to capture nonlinear teleconnection relationships that linear methods cannot detect.

\subsection{Complexity and Generalization Analysis}
\label{appendix_sec:Complexity and Generalization Analysis}

We provide a analysis of \textbf{TelePiT}'s computational complexity, generalization capabilities, and error propagation properties.

\subsubsection{Computational Complexity}
\begin{proposition}[Time Complexity]
The forward pass of \textbf{TelePiT} has time complexity:
\begin{equation}
\mathcal{O}(HLD^2 + L^2 D \log L + n_p HD + HD^2)
\end{equation}
where the terms correspond to spherical embedding, wavelet decomposition, teleconnection attention, and neural ODE integration, respectively.
\end{proposition}

\subsubsection{Error Propagation Analysis}
\begin{proposition}[Multi-Scale Error Bounds]
For a signal decomposed into $L+1$ frequency bands, the prediction error is bounded by:
\begin{equation}
\mathbb{E}[\|y - \hat{y}\|^2] \leq \sum_{\ell=0}^{L} \lambda_\ell \mathbb{E}[\|y_\ell - \hat{y}_\ell\|^2] + \mathcal{O}(L^{-\alpha})
\end{equation}
where $y_\ell, \hat{y}_\ell$ represent ground truth and prediction for frequency band $\ell$, $\lambda_\ell \geq 0$ are band-specific weights with $\sum_\ell \lambda_\ell = 1$, and the additional term bounds decomposition error.
\end{proposition}

\begin{proof}[Proof sketch]
The multi-scale architecture processes frequency bands independently before fusion, creating a natural error decomposition. The hierarchical fusion mechanism that privileges low-frequency information creates weights $\lambda_\ell$ that typically decrease with frequency, providing robustness against high-frequency noise. The decomposition error term reflects the finite precision of the learned wavelet basis.
\end{proof}

\subsubsection{Generalization Bounds}
\begin{proposition}[Generalization Error]
For a training set of size $n$ drawn from the atmospheric state distribution, with probability at least $1-\delta$, the generalization error satisfies:
\begin{equation}
\mathcal{R}(h) - \hat{\mathcal{R}}_n(h) \leq \mathcal{O}\left(\sqrt{\frac{\log(L+1) + \sum_{\ell=0}^L \mathcal{C}(h_\ell) + \log(1/\delta)}{n}}\right)
\end{equation}
where $\mathcal{R}(h)$ is the true risk, $\hat{\mathcal{R}}_n(h)$ is the empirical risk, and $\mathcal{C}(h_\ell)$ measures the complexity of the model component for frequency band $\ell$.
\end{proposition}

\begin{proof}[Proof sketch]
The physics-inspired constraints effectively reduce the hypothesis space complexity for each frequency band, lowering $\mathcal{C}(h_\ell)$ compared to unconstrained neural networks. The multi-scale decomposition enables specialized processing that further reduces per-band complexity. These reductions, combined with standard uniform convergence arguments, yield the stated bound.
\end{proof}

This analysis demonstrates that \textbf{TelePiT}'s architecture provides both computational efficiency and strong generalization guarantees, explaining its robust performance across diverse atmospheric conditions and forecast horizons.

\section{Additional Experimental Details}
\label{appendix_sec:Additional Experimental Details}

\subsection{Evaluation Metric}
\label{appendix_sec:Evaluation Metric}

\paragraph{RMSE and ACC}
Following existing works \cite{liucirt,nathaniel2024chaosbench,rasp2024weatherbench}, we adopt three metrics to evaluate the model performance: Root Mean Squared Error (RMSE) with a weighting scheme at each latitude $\theta_i$, Anomaly Correlation Coefficient (ACC), and Spectral Divergence (SpecDiv).

The RMSE and ACC are defined as follows:
\begin{equation}
\text{RMSE} = \sqrt{\frac{1}{|\theta||\phi|} \sum_{i=1}^{|\theta|} \sum_{j=1}^{|\phi|} w(\theta_i)(\hat{\mathbf{Y}}_{i,j} - \mathbf{Y}_{i,j})^2}, 
\end{equation}
\begin{equation}
\text{ACC} = \frac{\sum(\theta_i)[A_{\hat{Y}}\cdot A_{Y}]}{\sqrt{\sum w(\theta_i) A_{\hat{Y}}^{2} \sum w(\theta_i)A_{Y}^2}},
\end{equation}
where $w(\theta_i)=\frac{cos(\theta_i)}{\frac{1}{|\theta|}\sum_{a=1}^{|\theta|}cos(\theta_a)}$, $\theta_i \in \theta=[-\frac{\pi}{2}, \frac{\pi}{2}]$ and $\phi_i \in \phi= [-\pi, \pi]$ denotes the set of all latitudes and longitude, respectively. 
The predicted and observed anomalies at each grid are denoted by
$A_{\hat{\mathbf{Y}}_{i,j}}=\hat{\mathbf{Y}}_{i,j}-C$ and $A_{\mathbf{Y}_{i,j}}=\mathbf{Y}_{i,j}-C$
where $C$ is observational climatology.

\paragraph{SpecDiv}
For the physics-based metric, we adopt the SpecDiv as proposed in \cite{nathaniel2024chaosbench}. This metric measures the deviation between the power spectra of prediction and target in the spectral domain. It is particularly useful for evaluating the preservation of high-frequency components in forecasts, which tend to become blurry due to power divergence in the spectral domain. The SpecDiv is defined as:
\begin{equation}
\begin{aligned}
\text{SpecDiv} = \sum_k S^{\prime}(k)\cdot \log(\frac{S^{\prime}(k)}{\hat{S}^{\prime}(k)})
\end{aligned}
\end{equation}
where $S^{\prime}(k)$ and $\hat{S}^{\prime}(k)$ represent the normalized power spectra of target and prediction respectively. This metric follows the principles of Kullback-Leibler divergence to quantify the spectral differences between predictions and ground truth.

In addition to the metrics described in the main paper, we also utilize Multi-Scale Structural Similarity (MS-SSIM) and Spectral Residual (SpecRes) for comprehensive evaluation.

\paragraph{MS-SSIM}
This metric evaluates structural similarity between forecast and ground-truth across multiple scales, making it particularly suitable for weather forecasting where phenomena range from large-scale systems like cyclones to small localized features such as thunderstorms. 

Let $\mathbf{Y}$ and $\hat{\mathbf{Y}}$ represent the ground-truth and predicted images respectively. MS-SSIM incorporates three comparison measures defined as:
\begin{equation}
l(\mathbf{Y}, \hat{\mathbf{Y}}) = \frac{2\mu_\mathbf{Y}\mu_{\hat{\mathbf{Y}}} + C_1}{\mu_\mathbf{Y}^2 + \mu_{\hat{\mathbf{Y}}}^2 + C_1},
c(\mathbf{Y}, \hat{\mathbf{Y}}) = \frac{2\sigma_\mathbf{Y}\sigma_{\hat{\mathbf{Y}}} + C_2}{\sigma_\mathbf{Y}^2 + \sigma_{\hat{\mathbf{Y}}}^2 + C_2},
s(\mathbf{Y}, \hat{\mathbf{Y}}) = \frac{\sigma_{\mathbf{Y}\hat{\mathbf{Y}}} + C_3}{\sigma_\mathbf{Y}\sigma_{\hat{\mathbf{Y}}} + C_3}
\end{equation}
where $\mu_\mathbf{Y}$, $\sigma_\mathbf{Y}^2$ and $\sigma_{\mathbf{Y}\hat{\mathbf{Y}}}$ are the mean of $\mathbf{Y}$, the variance of $\mathbf{Y}$, and the covariance between $\mathbf{Y}$ and $\hat{\mathbf{Y}}$, respectively. The constants $C_1 = (K_1 L)^2$, $C_2 = (K_2 L)^2$, and $C_3 = C_2/2$ are included to avoid instability when the denominators are close to zero, where $L = 255$ is the dynamic range of the grayscale images, and $K_1 \ll 1$ and $K_2 \ll 1$ are two small constants.
The MS-SSIM metric is computed by analyzing these components across multiple scales through successive low-pass filtering and downsampling. At each scale $j$, the contrast comparison $c_j(\mathbf{Y}, \hat{\mathbf{Y}})$ and structure comparison $s_j(\mathbf{Y}, \hat{\mathbf{Y}})$ are calculated, while luminance comparison $l_M(\mathbf{Y}, \hat{\mathbf{Y}})$ is only considered at the coarsest scale $M$. The final MS-SSIM metric is defined as:
\begin{equation}
\text{MS-SSIM} = [l_M(\mathbf{Y}, \hat{\mathbf{Y}})]^{\alpha_M} \cdot \prod_{j=1}^{M} [c_j(\mathbf{Y}, \hat{\mathbf{Y}})]^{\beta_j} [s_j(\mathbf{Y}, \hat{\mathbf{Y}})]^{\gamma_j}
\end{equation}
where $\alpha_M$, $\beta_j$, and $\gamma_j$ are weighting parameters. We follow standard practice by setting $M=5$ and using the parameter values established in previous literature\cite{nathaniel2024chaosbench}.

\paragraph{SpecRes}
While SpecDiv measures the divergence between power spectra, SpecRes quantifies the root of the expected squared residual in the spectral domain:
\begin{equation}
\text{SpecRes} = \sqrt{\mathbb{E}_k[(\hat{S}'(k) - S'(k))^2]}
\end{equation}
where $\hat{S}'(k)$ and $S'(k)$ are the normalized power spectra of prediction and target respectively, and $\mathbb{E}_k$ denotes expectation over the high-frequency wavenumbers $\mathbf{K}_q$. Similar to SpecDiv, this metric will be zero if the power spectra are identical, but it emphasizes the magnitude of spectral differences rather than their distributional divergence.

\subsection{Baselines}
\label{appendix_sec:Baselines}

\paragraph{Operational Forecasting Systems} These established numerical weather prediction systems represent the current state-of-practice in operational meteorological forecasting, employing sophisticated physics-based models with data assimilation techniques.
\begin{itemize}
    \item \textbf{ECMWF:} The European Centre for Medium-Range Weather Forecasts employs the Integrated Forecasting System (IFS) CY48R1 version, generating predictions twice weekly (Mondays and Thursdays) with a 46-day forecast horizon.
    
    \item \textbf{UKMO:} The UK Met Office utilizes the GloSea6 (Global Seasonal forecast system version 6) for daily control forecasts extending to 60 days.
    
    \item \textbf{NCEP:} The National Centers for Environmental Prediction deploys the Climate Forecast System version 2 (CFSv2) for daily control forecasts with a 45-day prediction window.
    
    \item \textbf{CMA:} The China Meteorological Administration implements the Beijing Climate Center's fully-coupled BCC-CSM2-HR model, providing twice-weekly forecasts (Mondays and Thursdays) with a 60-day lead time.
\end{itemize}

\paragraph{Data-Driven Approaches} These data-driven models leverage recent advances in machine learning to capture atmospheric dynamics without explicit physical equations, offering complementary capabilities to traditional forecasting methods.
\begin{itemize}
    \item \textbf{GraphCast:} A graph neural network architecture that employs multi-mesh methodology to capture complex atmospheric dynamics. Implementation accessed via the ECMWF AI model repository. Due to computational constraints, evaluation is limited to weeks 3-4.
    
    \item \textbf{FourCastNetV2:} A Vision Transformer-based iterative model incorporating Adaptive Fourier Neural Operators for spatial patch mixing. Analysis covers 11 months of data (excluding October 2018 due to data availability issues).
    
    \item \textbf{PanguWeather:} Employs separate tokenization for pressure-level and single-level data within a transformer architecture, trained using hierarchical temporal aggregation.
    
    \item \textbf{ClimaX:} Utilizes a direct training approach with Vision Transformers, featuring independent tokenization per variable followed by variable aggregation.
    
    \item \textbf{CirT:} Implements geometry-aware transformers specifically designed to accommodate Earth's spherical structure for improved subseasonal-to-seasonal forecasting performance.
\end{itemize}

For FourCastNetV2, GraphCast, and PanguWeather, the inference was performed using the ECMWF AI model repository API (\url{https://github.com/ecmwf-lab/ai-models}). For ClimaX and CirT, these models were retrained following the same configurations as \textbf{TelePiT} to ensure fair comparison.

\subsection{Model Complexity}
\label{appendix_sec:Model Complexity}

We analyze the computational complexity of TelePiT compared to state-of-the-art baselines in Table \ref{tab:model_complexity}. While TelePiT requires more resources than the lightweight CirT (14.5G vs 2.2G FLOPs, 37M vs 16M parameters), this modest increase delivers substantial accuracy improvements, reducing forecast errors by up to 20\% across critical variables. More importantly, TelePiT remains dramatically more efficient than current leading models, requiring 7,500× fewer FLOPs than GraphCast (110T) and 11,500× fewer than PanguWeather (168T). With a practical model size of 141.64 MB and real-time inference capability (32.25 ms, 31 samples/second), TelePiT can be deployed on standard GPUs or even high-end edge devices—a critical advantage for operational weather services worldwide. This optimal balance between accuracy and efficiency stems from our physics-informed architecture that replaces brute-force spatial computation with intelligent design: spherical harmonic embedding reduces complexity from $O(H\times W)$ to $O(H)$, while teleconnection-aware attention leverages sparse global patterns instead of dense pixel-wise interactions. For S2S forecasting where both accuracy and computational feasibility determine real-world impact, TelePiT demonstrates that incorporating domain knowledge can achieve state-of-the-art performance at a fraction of the computational cost, making advanced climate prediction accessible even in resource-constrained settings.

\begin{table}[t] 
  \renewcommand{\arraystretch}{0.8} 
  \caption{Computation complexity and model size comparison.}
  \label{tab:model_complexity}
  \begin{tabular}{ccccc}
    \toprule
    \textbf { Model } & \text{GraphCast} & \text{PanguWeather} & \text{CirT} & \text{TelePiT}\\
    \midrule
    \textbf{FLOPs}  & 110T & 168T & 2.2G & 14.5G \\
    \textbf{Params} & 37M  & 256M & 16M  & 37M \\
    \bottomrule
  \end{tabular}
\end{table}

\subsection{Additional Overall Performance}
\label{appendix_sec:Additional Overall Performance}
Table \ref{appendix_tab:main} provides a comprehensive evaluation using two additional metrics: MS-SSIM and SpecRes. Higher MS-SSIM values indicate better structural similarity between predictions and ground truth, while lower SpecRes values indicate better spectral accuracy. TelePiT consistently outperforms all baseline models across these metrics, with particularly strong improvements in near-surface variables.

Table \ref{appendix_tab:main_2} extends our comparison against operational forecasting systems across all five metrics (MS-SSIM and SpecRes). The results demonstrate TelePiT's significant performance advantages over current operational systems across all metrics and forecast horizons, with improvements ranging from 20-40\% depending on the variable and metric.

Figures \ref{appendix_fig:heatmap_acc}, \ref{appendix_fig:heatmap_MS_SSIM}, \ref{appendix_fig:heatmap_specdiv}, and \ref{appendix_fig:heatmap_specres} visualize model performance across all pressure levels using ACC, MS-SSIM, SpecDiv, and SpecRes metrics respectively. These heatmaps highlight TelePiT's consistent outperformance throughout the atmospheric column, with particularly strong advantages in the mid-troposphere and for temperature variables.

\subsection{Additional Ablation Study}
\label{appendix_sec:Additional Ablation Study}
Table \ref{appendix_tab:ablation} extends our ablation analysis using three complementary metrics: MS-SSIM (structural similarity) and spectral metrics (SpecDiv and SpecRec). These results provide deeper insights into how each component contributes to the model's predictive capabilities across different variables and spatial scales.
For MS-SSIM (higher is better), TelePiT consistently outperforms all ablated variants across all variables and time horizons. Removing the Spherical Harmonic Embedding (SHE) causes the most significant performance drops (particularly for t2m where the score decreases from 0.8568 to 0.8399), indicating its importance for capturing Earth's spherical geometry and spatial relationships.
The spectral metrics (SpecDiv and SpecRec, lower is better) reveal interesting component-specific contributions. TelePiT achieves the best performance for most variables, especially for near-surface variables (t2m, u10, v10) where the scores are substantially better. 
Overall, these extended results confirm that each component makes meaningful contributions to TelePiT's superior performance, with the relative importance of each component varying by variable and evaluation metric. 

\subsection{Additional Empirical Analysis}
\label{appendix_sec:Additional Empirical Analysis}
Figures \ref{appendix_fig:map_t850_1}, \ref{appendix_fig:map_t500_0}, \ref{appendix_fig:map_t500_1}, \ref{appendix_fig:map_z850_0}, \ref{appendix_fig:map_z850_1}, \ref{appendix_fig:map_z500_0}, \ref{appendix_fig:map_z500_1} provide detailed geographical visualizations of model performance for various atmospheric variables at Weeks 3-4 and Weeks 5-6 forecast horizons. These visualizations reveal TelePiT's consistent performance advantages across different geographical regions, including traditionally challenging areas like continental interiors and polar regions.

The time series analyses in Figures \ref{appendix_fig:t500_rmse_2018}, \ref{appendix_fig:z850_rmse_2018}, and \ref{appendix_fig:z500_rmse_2018} track daily forecast performance throughout the 2018 test period. TelePiT maintains lower error rates with remarkably less variability compared to both data-driven and physics-based baseline models. This stability across different seasons and atmospheric conditions demonstrates the model's robust generalization capabilities.

The consistent error reduction across all variables, pressure levels, geographical regions, and throughout the annual cycle confirms that TelePiT's architectural innovations effectively capture both the spatial and temporal aspects of atmospheric dynamics critical for extended-range forecasting.

\subsection{Robustness Study}
\label{appendix_sec:Robustness Study}
Table \ref{appendix_tab:robustness} presents a comprehensive robustness evaluation using 2019 data as an out-of-sample test set. \textbf{TelePiT} demonstrates exceptional generalization capabilities, maintaining substantial performance advantages across all metrics. For surface variables, the model achieves particularly impressive results: t2m RMSE of 12.81K represents a 55.0\% improvement over CirT (28.42K), while ACC reaches 0.9954. The spectral metrics reveal SpecDiv values 2-3 orders of magnitude lower than competing approaches, indicating superior physical consistency. These consistent performance advantages across different years provide strong evidence that \textbf{TelePiT} captures fundamental atmospheric processes rather than dataset-specific patterns.

\subsection{Parameter Sensitivity Analysis}
\label{appendix_sec:Parameter Sensitivity Analysis}
Figure \ref{appendix_fig:lamda} shows the sensitivity analysis of teleconnection coefficient $\lambda$, which controls the balance between local dynamics and global teleconnection patterns in the attention mechanism. The results reveal a consistent optimal value of $\lambda = 0.2$ across most variables and metrics, where RMSE reaches minimum and ACC achieves maximum values. Temperature variables show sharp optima, while wind components exhibit the most pronounced sensitivity, reflecting the complex interplay between boundary layer processes and large-scale pressure gradients. The consistency of $\lambda = 0.2$ across different variables validates our architectural design and simplifies model deployment.

\newpage

\begin{table*}[t] 
\let\oldeverydisplay\everydisplay
\let\oldeverymath\everymath
\everydisplay{}
\everymath{}
\definecolor{lightblue}{RGB}{240, 249, 255}  
\definecolor{lightred}{RGB}{255, 255, 255}   

\caption{Performance comparison of \textbf{TelePiT} against data-driven models: MS-SSIM (higher values indicate better structural similarity) and SpecRes (lower values indicate better spectral accuracy). Results are shown for both Weeks 3-4 and Weeks 5-6 forecasts, demonstrating TelePiT's consistent performance advantages across all key variables, metrics, and forecast lead times.}

\label{appendix_tab:main}
  \begin{tabular}{cc|cccccc|cccccc}
    \toprule
    \multirow{2}{*}{\textbf{}} &\multirow{2}{*}{\textbf{Variable}} & \multicolumn{6}{c|}{\textbf{MS-SSIM ($\uparrow$)}} & \multicolumn{6}{c}{\textbf{SpecRes ($\downarrow$)}}\\ 
    &  &  FCNV2 &  GC & Pangu & ClimaX & CirT & \textbf{TelePiT} &  FCN-V2 &  GC & Pangu & ClimaX & CirT & \textbf{TelePiT}\\
    
    \midrule
    \multirow{7}{*}{\rotatebox{90}{\textbf{Weeks 3-4}}}
    
    & \multicolumn{1}{>{\columncolor{lightblue}}c|}{\text{ z500 ($gpm$) }}  
    & \multicolumn{1}{>{\columncolor{lightblue}}c}{0.8629}   & \multicolumn{1}{>{\columncolor{lightblue}}c}{0.8686}   & \multicolumn{1}{>{\columncolor{lightblue}}c}{0.8638}   
    & \multicolumn{1}{>{\columncolor{lightblue}}c}{0.8588}   & \multicolumn{1}{>{\columncolor{lightblue}}c}{0.9061}   & \multicolumn{1}{>{\columncolor{lightblue}}c|}{\textbf{0.9146}}      
    & \multicolumn{1}{>{\columncolor{lightblue}}c}{0.0422}    & \multicolumn{1}{>{\columncolor{lightblue}}c}{0.0283}    & \multicolumn{1}{>{\columncolor{lightblue}}c}{0.1063}   
    & \multicolumn{1}{>{\columncolor{lightblue}}c}{0.0786}    & \multicolumn{1}{>{\columncolor{lightblue}}c}{0.0184}    & \multicolumn{1}{>{\columncolor{lightblue}}c}{\textbf{0.0127}}  \\
    
    \multicolumn{1}{c}{}& \multicolumn{1}{>{\columncolor{lightred}}c|}{\text{ z850 ($gpm$) }}  
    & \multicolumn{1}{>{\columncolor{lightred}}c}{0.8130}    & \multicolumn{1}{>{\columncolor{lightred}}c}{0.7985}    & \multicolumn{1}{>{\columncolor{lightred}}c}{0.7900}   
    & \multicolumn{1}{>{\columncolor{lightred}}c}{0.7887}    & \multicolumn{1}{>{\columncolor{lightred}}c}{0.8655}    & \multicolumn{1}{>{\columncolor{lightred}}c|}{{\textbf{0.8786}}}    
    & \multicolumn{1}{>{\columncolor{lightred}}c}{0.0554}     & \multicolumn{1}{>{\columncolor{lightred}}c}{0.0574}     & \multicolumn{1}{>{\columncolor{lightred}}c}{0.0390}   
    & \multicolumn{1}{>{\columncolor{lightred}}c}{0.0148}     & \multicolumn{1}{>{\columncolor{lightred}}c}{0.0311}     & \multicolumn{1}{>{\columncolor{lightred}}c}{\textbf{0.0134}}  \\
    
    & \multicolumn{1}{>{\columncolor{lightblue}}c|}{\text{ t500 ($K$) }}          
    & \multicolumn{1}{>{\columncolor{lightblue}}c}{0.8923}    & \multicolumn{1}{>{\columncolor{lightblue}}c}{0.8857}    & \multicolumn{1}{>{\columncolor{lightblue}}c}{0.8793}    
    & \multicolumn{1}{>{\columncolor{lightblue}}c}{0.8695}    & \multicolumn{1}{>{\columncolor{lightblue}}c}{0.9201}    & \multicolumn{1}{>{\columncolor{lightblue}}c|}{\textbf{0.9263}}       
    & \multicolumn{1}{>{\columncolor{lightblue}}c}{0.0700}    & \multicolumn{1}{>{\columncolor{lightblue}}c}{0.0248}    & \multicolumn{1}{>{\columncolor{lightblue}}c}{0.0512}   
    & \multicolumn{1}{>{\columncolor{lightblue}}c}{0.0954}    & \multicolumn{1}{>{\columncolor{lightblue}}c}{0.0297}    & \multicolumn{1}{>{\columncolor{lightblue}}c}{\textbf{0.0139}}  \\
    
    \multicolumn{1}{c}{}& \multicolumn{1}{>{\columncolor{lightred}}c|}{\text{ t850 ($K$) }}          
    & \multicolumn{1}{>{\columncolor{lightred}}c}{0.9148}    & \multicolumn{1}{>{\columncolor{lightred}}c}{0.9165}    & \multicolumn{1}{>{\columncolor{lightred}}c}{0.9088}    
    & \multicolumn{1}{>{\columncolor{lightred}}c}{0.8898}    & \multicolumn{1}{>{\columncolor{lightred}}c}{0.9428}    & \multicolumn{1}{>{\columncolor{lightred}}c|}{\textbf{0.9473}}       
    & \multicolumn{1}{>{\columncolor{lightred}}c}{0.0633}    & \multicolumn{1}{>{\columncolor{lightred}}c}{0.0279}    & \multicolumn{1}{>{\columncolor{lightred}}c}{0.0519}   
    & \multicolumn{1}{>{\columncolor{lightred}}c}{0.0546}    & \multicolumn{1}{>{\columncolor{lightred}}c}{0.0162}    & \multicolumn{1}{>{\columncolor{lightred}}c}{\textbf{0.0067}}  \\
    
    & \multicolumn{1}{>{\columncolor{lightblue}}c|}{\text{ t2m ($K$) }}           
    & \multicolumn{1}{>{\columncolor{lightblue}}c}{--}       & \multicolumn{1}{>{\columncolor{lightblue}}c}{--}       & \multicolumn{1}{>{\columncolor{lightblue}}c}{--}       
    & \multicolumn{1}{>{\columncolor{lightblue}}c}{0.7017}   & \multicolumn{1}{>{\columncolor{lightblue}}c}{0.8216}   & \multicolumn{1}{>{\columncolor{lightblue}}c|}{\textbf{0.8568}}      
    & \multicolumn{1}{>{\columncolor{lightblue}}c}{--}       & \multicolumn{1}{>{\columncolor{lightblue}}c}{--}       & \multicolumn{1}{>{\columncolor{lightblue}}c}{--}      
    & \multicolumn{1}{>{\columncolor{lightblue}}c}{0.0797}    & \multicolumn{1}{>{\columncolor{lightblue}}c}{0.0344}    & \multicolumn{1}{>{\columncolor{lightblue}}c}{\textbf{0.0042}}  \\
    
    \multicolumn{1}{c}{}& \multicolumn{1}{>{\columncolor{lightred}}c|}{\text{ u10 (m/s) }}           
    & \multicolumn{1}{>{\columncolor{lightred}}c}{0.2265}    & \multicolumn{1}{>{\columncolor{lightred}}c}{--}       & \multicolumn{1}{>{\columncolor{lightred}}c}{0.2307}    
    & \multicolumn{1}{>{\columncolor{lightred}}c}{0.8659}    & \multicolumn{1}{>{\columncolor{lightred}}c}{0.9348}    & \multicolumn{1}{>{\columncolor{lightred}}c|}{{\textbf{0.9431}}}       
    & \multicolumn{1}{>{\columncolor{lightred}}c}{0.0488}    & \multicolumn{1}{>{\columncolor{lightred}}c}{--}       & \multicolumn{1}{>{\columncolor{lightred}}c}{0.0265}   
    & \multicolumn{1}{>{\columncolor{lightred}}c}{0.0586}    & \multicolumn{1}{>{\columncolor{lightred}}c}{0.0221}    & \multicolumn{1}{>{\columncolor{lightred}}c}{\textbf{0.0048}}  \\
    
    & \multicolumn{1}{>{\columncolor{lightblue}}c|}{\text{ v10 (m/s) }}           
    & \multicolumn{1}{>{\columncolor{lightblue}}c}{0.3289}    & \multicolumn{1}{>{\columncolor{lightblue}}c}{--}       & \multicolumn{1}{>{\columncolor{lightblue}}c}{0.3180}    
    & \multicolumn{1}{>{\columncolor{lightblue}}c}{0.8667}    & \multicolumn{1}{>{\columncolor{lightblue}}c}{0.9307}    & \multicolumn{1}{>{\columncolor{lightblue}}c|}{\textbf{0.9408}}       
    & \multicolumn{1}{>{\columncolor{lightblue}}c}{0.0705}    & \multicolumn{1}{>{\columncolor{lightblue}}c}{--}       & \multicolumn{1}{>{\columncolor{lightblue}}c}{0.0556}   
    & \multicolumn{1}{>{\columncolor{lightblue}}c}{0.0409}    & \multicolumn{1}{>{\columncolor{lightblue}}c}{0.0228}    & \multicolumn{1}{>{\columncolor{lightblue}}c}{\textbf{0.0063}}  \\
    
    \midrule
    \multirow{7}{*}{\rotatebox{90}{\textbf{Weeks 5-6}}}
    
    & \multicolumn{1}{>{\columncolor{lightblue}}c|}{\text{ z500 ($gpm$) }}  
    & \multicolumn{1}{>{\columncolor{lightblue}}c}{0.8559}   & \multicolumn{1}{>{\columncolor{lightblue}}c}{--}       & \multicolumn{1}{>{\columncolor{lightblue}}c}{0.8450}   
    & \multicolumn{1}{>{\columncolor{lightblue}}c}{0.8543}   & \multicolumn{1}{>{\columncolor{lightblue}}c}{0.9083}   & \multicolumn{1}{>{\columncolor{lightblue}}c|}{\textbf{0.9159}}      
    & \multicolumn{1}{>{\columncolor{lightblue}}c}{0.0422}    & \multicolumn{1}{>{\columncolor{lightblue}}c}{--}       & \multicolumn{1}{>{\columncolor{lightblue}}c}{0.1152}   
    & \multicolumn{1}{>{\columncolor{lightblue}}c}{0.0821}    & \multicolumn{1}{>{\columncolor{lightblue}}c}{0.0263}    & \multicolumn{1}{>{\columncolor{lightblue}}c}{\textbf{0.0121}}  \\
    
    \multicolumn{1}{c}{}& \multicolumn{1}{>{\columncolor{lightred}}c|}{\text{ z850 ($gpm$) }}  
    & \multicolumn{1}{>{\columncolor{lightred}}c}{0.8025}    & \multicolumn{1}{>{\columncolor{lightred}}c}{--}       & \multicolumn{1}{>{\columncolor{lightred}}c}{0.7608}   
    & \multicolumn{1}{>{\columncolor{lightred}}c}{0.7882}    & \multicolumn{1}{>{\columncolor{lightred}}c}{0.8699}   & \multicolumn{1}{>{\columncolor{lightred}}c|}{{\textbf{0.8800}}}      
    & \multicolumn{1}{>{\columncolor{lightred}}c}{0.0575}     & \multicolumn{1}{>{\columncolor{lightred}}c}{--}       & \multicolumn{1}{>{\columncolor{lightred}}c}{0.0400}   
    & \multicolumn{1}{>{\columncolor{lightred}}c}{0.0316}     & \multicolumn{1}{>{\columncolor{lightred}}c}{0.0368}    & \multicolumn{1}{>{\columncolor{lightred}}c}{\textbf{0.0132}}  \\
    
    & \multicolumn{1}{>{\columncolor{lightblue}}c|}{\text{ t500 ($K$) }}          
    & \multicolumn{1}{>{\columncolor{lightblue}}c}{0.8833}    & \multicolumn{1}{>{\columncolor{lightblue}}c}{--}       & \multicolumn{1}{>{\columncolor{lightblue}}c}{0.8571}    
    & \multicolumn{1}{>{\columncolor{lightblue}}c}{0.8636}    & \multicolumn{1}{>{\columncolor{lightblue}}c}{0.9208}    & \multicolumn{1}{>{\columncolor{lightblue}}c|}{\textbf{0.9269}}       
    & \multicolumn{1}{>{\columncolor{lightblue}}c}{0.0689}    & \multicolumn{1}{>{\columncolor{lightblue}}c}{--}       & \multicolumn{1}{>{\columncolor{lightblue}}c}{0.0533}   
    & \multicolumn{1}{>{\columncolor{lightblue}}c}{0.0836}    & \multicolumn{1}{>{\columncolor{lightblue}}c}{0.0424}    & \multicolumn{1}{>{\columncolor{lightblue}}c}{\textbf{0.0161}}  \\

    \multicolumn{1}{c}{}& \multicolumn{1}{>{\columncolor{lightred}}c|}{\text{ t850 ($K$) }}          
    & \multicolumn{1}{>{\columncolor{lightred}}c}{0.9088}    & \multicolumn{1}{>{\columncolor{lightred}}c}{--}       & \multicolumn{1}{>{\columncolor{lightred}}c}{0.8964}    
    & \multicolumn{1}{>{\columncolor{lightred}}c}{0.8840}    & \multicolumn{1}{>{\columncolor{lightred}}c}{0.9435}    & \multicolumn{1}{>{\columncolor{lightred}}c|}{{\textbf{0.9473}}}       
    & \multicolumn{1}{>{\columncolor{lightred}}c}{0.0622}    & \multicolumn{1}{>{\columncolor{lightred}}c}{--}       & \multicolumn{1}{>{\columncolor{lightred}}c}{0.0539}   
    & \multicolumn{1}{>{\columncolor{lightred}}c}{0.0547}    & \multicolumn{1}{>{\columncolor{lightred}}c}{0.0132}    & \multicolumn{1}{>{\columncolor{lightred}}c}{\textbf{0.0064}}  \\
    
    & \multicolumn{1}{>{\columncolor{lightblue}}c|}{\text{ t2m ($K$) }}           
    & \multicolumn{1}{>{\columncolor{lightblue}}c}{--}       & \multicolumn{1}{>{\columncolor{lightblue}}c}{--}       & \multicolumn{1}{>{\columncolor{lightblue}}c}{--}       
    & \multicolumn{1}{>{\columncolor{lightblue}}c}{0.7013}   & \multicolumn{1}{>{\columncolor{lightblue}}c}{0.8209}   & \multicolumn{1}{>{\columncolor{lightblue}}c|}{\textbf{0.8568}}      
    & \multicolumn{1}{>{\columncolor{lightblue}}c}{--}       & \multicolumn{1}{>{\columncolor{lightblue}}c}{--}       & \multicolumn{1}{>{\columncolor{lightblue}}c}{--}      
    & \multicolumn{1}{>{\columncolor{lightblue}}c}{0.0800}    & \multicolumn{1}{>{\columncolor{lightblue}}c}{0.0333}    & \multicolumn{1}{>{\columncolor{lightblue}}c}{\textbf{0.0042}}  \\
    
    \multicolumn{1}{c}{}& \multicolumn{1}{>{\columncolor{lightred}}c|}{\text{ u10 (m/s) }}           
    & \multicolumn{1}{>{\columncolor{lightred}}c}{0.2227}    & \multicolumn{1}{>{\columncolor{lightred}}c}{--}       & \multicolumn{1}{>{\columncolor{lightred}}c}{0.2238}    
    & \multicolumn{1}{>{\columncolor{lightred}}c}{0.8655}    & \multicolumn{1}{>{\columncolor{lightred}}c}{0.9344}    & \multicolumn{1}{>{\columncolor{lightred}}c|}{{\textbf{0.9437}}}       
    & \multicolumn{1}{>{\columncolor{lightred}}c}{0.0478}    & \multicolumn{1}{>{\columncolor{lightred}}c}{--}       & \multicolumn{1}{>{\columncolor{lightred}}c}{0.0263}   
    & \multicolumn{1}{>{\columncolor{lightred}}c}{0.0616}    & \multicolumn{1}{>{\columncolor{lightred}}c}{0.0215}    & \multicolumn{1}{>{\columncolor{lightred}}c}{\textbf{0.0052}}  \\
    
    & \multicolumn{1}{>{\columncolor{lightblue}}c|}{\text{ v10 (m/s) }}           
    & \multicolumn{1}{>{\columncolor{lightblue}}c}{0.3213}    & \multicolumn{1}{>{\columncolor{lightblue}}c}{--}       & \multicolumn{1}{>{\columncolor{lightblue}}c}{0.3187}    
    & \multicolumn{1}{>{\columncolor{lightblue}}c}{0.8637}    & \multicolumn{1}{>{\columncolor{lightblue}}c}{0.9317}    & \multicolumn{1}{>{\columncolor{lightblue}}c|}{\textbf{9404}}       
    & \multicolumn{1}{>{\columncolor{lightblue}}c}{0.0691}    & \multicolumn{1}{>{\columncolor{lightblue}}c}{--}       & \multicolumn{1}{>{\columncolor{lightblue}}c}{0556}   
    & \multicolumn{1}{>{\columncolor{lightblue}}c}{0.0462}    & \multicolumn{1}{>{\columncolor{lightblue}}c}{0.0227}    & \multicolumn{1}{>{\columncolor{lightblue}}c}{\textbf{0.0059}}  \\
    
    \bottomrule
  \end{tabular}
\let\everydisplay\oldeverydisplay
\let\everymath\oldeverymath
\vspace{2cm}
\end{table*}


\begin{table*}[t] 
\let\oldeverydisplay\everydisplay
\let\oldeverymath\everymath
\everydisplay{}
\everymath{}
\definecolor{lightblue}{RGB}{240, 249, 255}  
\definecolor{lightred}{RGB}{255, 255, 255}   
\setlength{\tabcolsep}{10pt} 
\caption{Performance comparison of \textbf{TelePiT} against operational forecasting systems. Results are shown for both Weeks 3-4 and 5-6 forecasts, demonstrating TelePiT's significant performance advantages over current operational systems across all metrics.}
\label{appendix_tab:main_2}
  \begin{tabular}{cc|cccc|cccc}
    \toprule
    \multirow{2}{*}{\textbf{Metric}} & \multirow{2}{*}{\textbf{Model}} & \multicolumn{4}{c|}{\textbf{Weeks 3-4}} & \multicolumn{4}{c}{\textbf{Weeks 5-6}}\\ 
    & & z500 & z850 & t500 & t850 & z500 & z850 & t500 & t850 \\

    \midrule
    \multirow{5}{*}{\rotatebox{0}{\textbf{MS-SSIM ($\uparrow$)}}}
    &\multicolumn{1}{>{\columncolor{lightblue}}c|}{\text{ UKMO }}  
    &\multicolumn{1}{>{\columncolor{lightblue}}c}{0.8682} & \multicolumn{1}{>{\columncolor{lightblue}}c}{0.7959} & \multicolumn{1}{>{\columncolor{lightblue}}c}{{0.8848}}& \multicolumn{1}{>{\columncolor{lightblue}}c|}{{0.8004}}      
    &\multicolumn{1}{>{\columncolor{lightblue}}c}{0.8572} & \multicolumn{1}{>{\columncolor{lightblue}}c}{0.7767} & \multicolumn{1}{>{\columncolor{lightblue}}c}{{0.8747}} & \multicolumn{1}{>{\columncolor{lightblue}}c}{{0.8570}} \\
    
    &\multicolumn{1}{>{\columncolor{lightred}}c|}{\text{ NCEP }}  
    &\multicolumn{1}{>{\columncolor{lightred}}c}{0.8614} & \multicolumn{1}{>{\columncolor{lightred}}c}{0.7928} & \multicolumn{1}{>{\columncolor{lightred}}c}{0.8786} & \multicolumn{1}{>{\columncolor{lightred}}c|}{0.9091}   
    &\multicolumn{1}{>{\columncolor{lightred}}c}{0.8535} & \multicolumn{1}{>{\columncolor{lightred}}c}{0.7781} & \multicolumn{1}{>{\columncolor{lightred}}c}{0.8699} & \multicolumn{1}{>{\columncolor{lightred}}c}{0.9037} \\

    &\multicolumn{1}{>{\columncolor{lightblue}}c|}{\text{ ECMWF }}  
    &\multicolumn{1}{>{\columncolor{lightblue}}c}{0.8715} & \multicolumn{1}{>{\columncolor{lightblue}}c}{0.8153} & \multicolumn{1}{>{\columncolor{lightblue}}c}{{0.8904}}& \multicolumn{1}{>{\columncolor{lightblue}}c|}{{0.9175}}      
    &\multicolumn{1}{>{\columncolor{lightblue}}c}{0.8575} & \multicolumn{1}{>{\columncolor{lightblue}}c}{0.7974} & \multicolumn{1}{>{\columncolor{lightblue}}c}{{0.8804}} & \multicolumn{1}{>{\columncolor{lightblue}}c}{{0.9103}} \\
    
    &\multicolumn{1}{>{\columncolor{lightred}}c|}{\text{ CMA }}  
    &\multicolumn{1}{>{\columncolor{lightred}}c}{0.8521} & \multicolumn{1}{>{\columncolor{lightred}}c}{0.7823} & \multicolumn{1}{>{\columncolor{lightred}}c}{0.8710} & \multicolumn{1}{>{\columncolor{lightred}}c|}{0.8928}   
    &\multicolumn{1}{>{\columncolor{lightred}}c}{0.8478} & \multicolumn{1}{>{\columncolor{lightred}}c}{0.7758} & \multicolumn{1}{>{\columncolor{lightred}}c}{0.8651} & \multicolumn{1}{>{\columncolor{lightred}}c}{0.8906} \\

    &\multicolumn{1}{>{\columncolor{lightblue}}c|}{\textbf{ TelePiT }}  
    &\multicolumn{1}{>{\columncolor{lightblue}}c}{\textbf{0.9146}} & \multicolumn{1}{>{\columncolor{lightblue}}c}{\textbf{0.8786}} & \multicolumn{1}{>{\columncolor{lightblue}}c}{\textbf{0.9263}}& \multicolumn{1}{>{\columncolor{lightblue}}c|}{\textbf{0.9473}}      
    &\multicolumn{1}{>{\columncolor{lightblue}}c}{\textbf{0.9159}} & \multicolumn{1}{>{\columncolor{lightblue}}c}{\textbf{0.8800}} & \multicolumn{1}{>{\columncolor{lightblue}}c}{\textbf{0.9269}} & \multicolumn{1}{>{\columncolor{lightblue}}c}{\textbf{0.9473}} \\

    \midrule
    \multirow{5}{*}{\rotatebox{0}{\textbf{SpecRec ($\downarrow$)}}}
    &\multicolumn{1}{>{\columncolor{lightblue}}c|}{\text{ UKMO }}  
    &\multicolumn{1}{>{\columncolor{lightblue}}c}{0.0386} & \multicolumn{1}{>{\columncolor{lightblue}}c}{0.0377} & \multicolumn{1}{>{\columncolor{lightblue}}c}{{0.0470}}& \multicolumn{1}{>{\columncolor{lightblue}}c|}{{0.0310}}      
    &\multicolumn{1}{>{\columncolor{lightblue}}c}{0.0681} & \multicolumn{1}{>{\columncolor{lightblue}}c}{0.0403} & \multicolumn{1}{>{\columncolor{lightblue}}c}{{0.0523}} & \multicolumn{1}{>{\columncolor{lightblue}}c}{{0.0302}} \\
    
    &\multicolumn{1}{>{\columncolor{lightred}}c|}{\text{ NCEP }}  
    &\multicolumn{1}{>{\columncolor{lightred}}c}{0.0678} & \multicolumn{1}{>{\columncolor{lightred}}c}{0.0613} & \multicolumn{1}{>{\columncolor{lightred}}c}{0.2449} & \multicolumn{1}{>{\columncolor{lightred}}c|}{0.1432}   
    &\multicolumn{1}{>{\columncolor{lightred}}c}{0.042} & \multicolumn{1}{>{\columncolor{lightred}}c}{0.0613} & \multicolumn{1}{>{\columncolor{lightred}}c}{0.2449} & \multicolumn{1}{>{\columncolor{lightred}}c}{0.1398} \\

    &\multicolumn{1}{>{\columncolor{lightblue}}c|}{\text{ ECMWF }}  
    &\multicolumn{1}{>{\columncolor{lightblue}}c}{0.0727} & \multicolumn{1}{>{\columncolor{lightblue}}c}{0.0266} & \multicolumn{1}{>{\columncolor{lightblue}}c}{{0.0427}}& \multicolumn{1}{>{\columncolor{lightblue}}c|}{{0.0612}}      
    &\multicolumn{1}{>{\columncolor{lightblue}}c}{0.0634} & \multicolumn{1}{>{\columncolor{lightblue}}c}{0.0282} & \multicolumn{1}{>{\columncolor{lightblue}}c}{{0.0262}} & \multicolumn{1}{>{\columncolor{lightblue}}c}{{0.0388}} \\
    
    &\multicolumn{1}{>{\columncolor{lightred}}c|}{\text{ CMA }}  
    &\multicolumn{1}{>{\columncolor{lightred}}c}{0.0400} & \multicolumn{1}{>{\columncolor{lightred}}c}{0.0496} & \multicolumn{1}{>{\columncolor{lightred}}c}{0.0603} & \multicolumn{1}{>{\columncolor{lightred}}c|}{0.0203}   
    &\multicolumn{1}{>{\columncolor{lightred}}c}{0.0395} & \multicolumn{1}{>{\columncolor{lightred}}c}{0.0481} & \multicolumn{1}{>{\columncolor{lightred}}c}{0.0384} & \multicolumn{1}{>{\columncolor{lightred}}c}{0.0205} \\

    &\multicolumn{1}{>{\columncolor{lightblue}}c|}{\textbf{ TelePiT }}  
    &\multicolumn{1}{>{\columncolor{lightblue}}c}{\textbf{0.0127}} & \multicolumn{1}{>{\columncolor{lightblue}}c}{\textbf{0.0134}} & \multicolumn{1}{>{\columncolor{lightblue}}c}{\textbf{0.0139}}& \multicolumn{1}{>{\columncolor{lightblue}}c|}{\textbf{0.0067}}      
    &\multicolumn{1}{>{\columncolor{lightblue}}c}{\textbf{0.0121}} & \multicolumn{1}{>{\columncolor{lightblue}}c}{\textbf{0.0132}} & \multicolumn{1}{>{\columncolor{lightblue}}c}{\textbf{0.0161}} & \multicolumn{1}{>{\columncolor{lightblue}}c}{\textbf{0.0064}} \\
    
    \bottomrule
  \end{tabular}
\let\everydisplay\oldeverydisplay
\let\everymath\oldeverymath
\end{table*}

\begin{figure*}[t]
    \centering
    \includegraphics[width=0.88\linewidth]{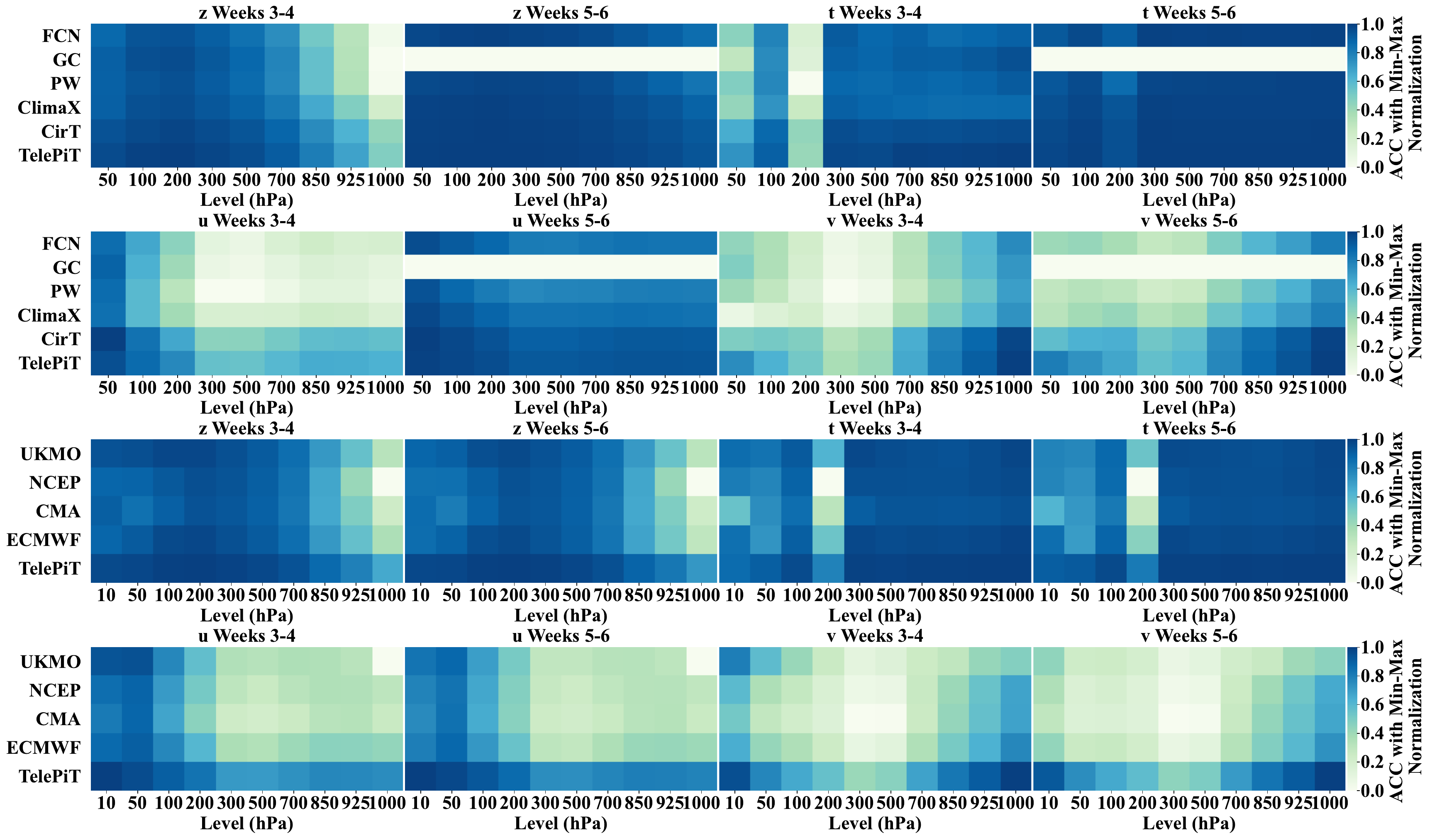}
    \vspace{-0.3cm}
    \caption{ACC comparison of variable $z$, $t$, $u$, and $v$ on all different pressure levels. The color scale indicates normalized ACC values (higher is better), with darker blue representing superior performance. TelePiT consistently shows higher ACC values across all variables and pressure levels compared to both data-driven models and operational forecasting systems. GraphCast (GC) is unavailable at Weeks 5-6 due to the out of memory.}
    \label{appendix_fig:heatmap_acc}
\end{figure*}

\begin{figure*}[t]
    \centering
    \includegraphics[width=0.88\linewidth]{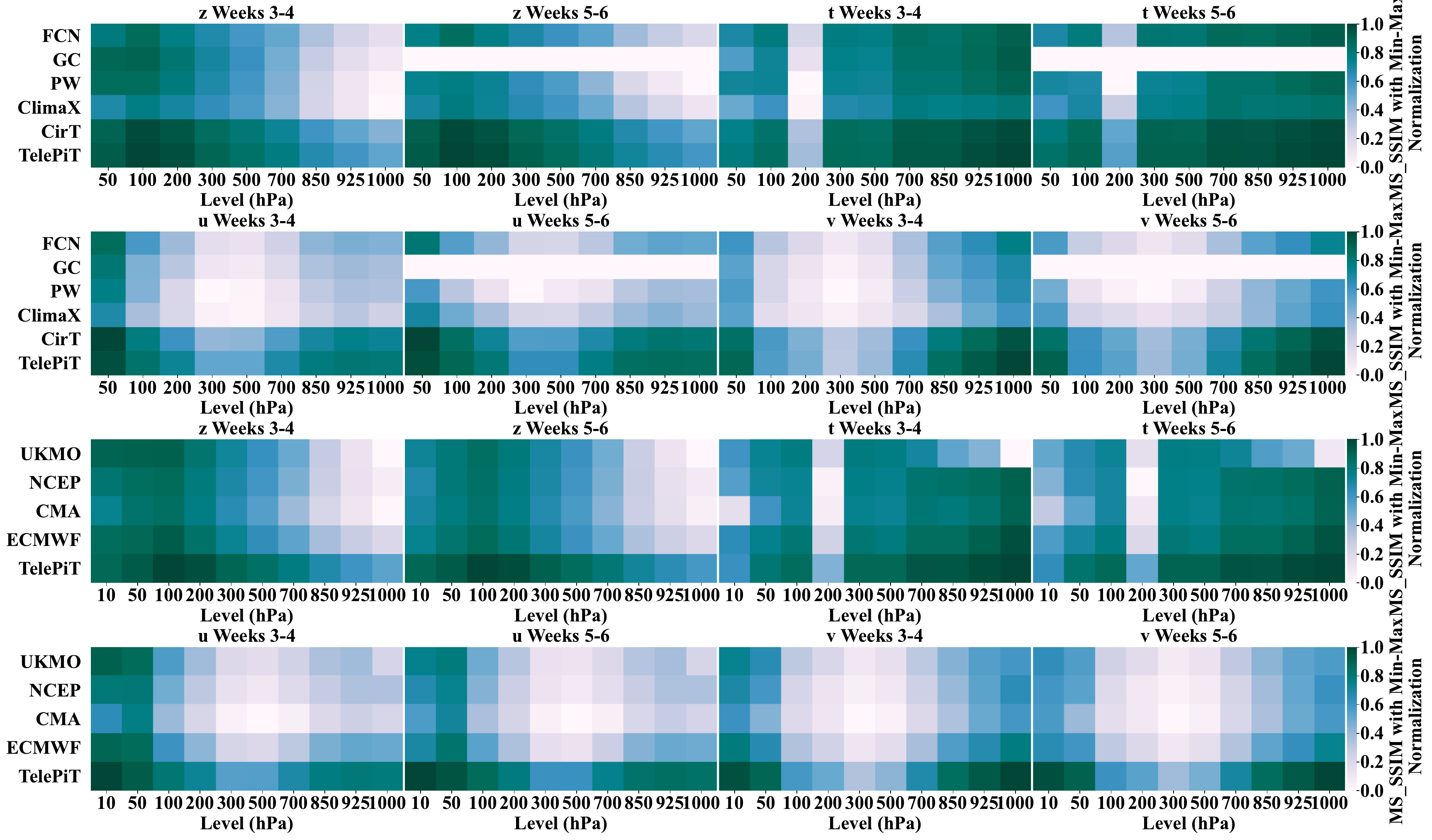}
    \vspace{-0.3cm}
    \caption{MS-SSIM comparison of variable $z$, $t$, $u$, and $v$ on all different pressure levels. The color gradient indicates normalized MS-SSIM values (higher is better), showing TelePiT's superior ability to preserve structural features in forecasts compared to all baseline models, with particularly strong performance in temperature and wind variables at lower-to-mid tropospheric levels. GraphCast (GC) is unavailable at Weeks 5-6 due to the out of memory.}
    \label{appendix_fig:heatmap_MS_SSIM}
\end{figure*}

\begin{figure*}[t]
    \centering
    \includegraphics[width=0.84\linewidth]{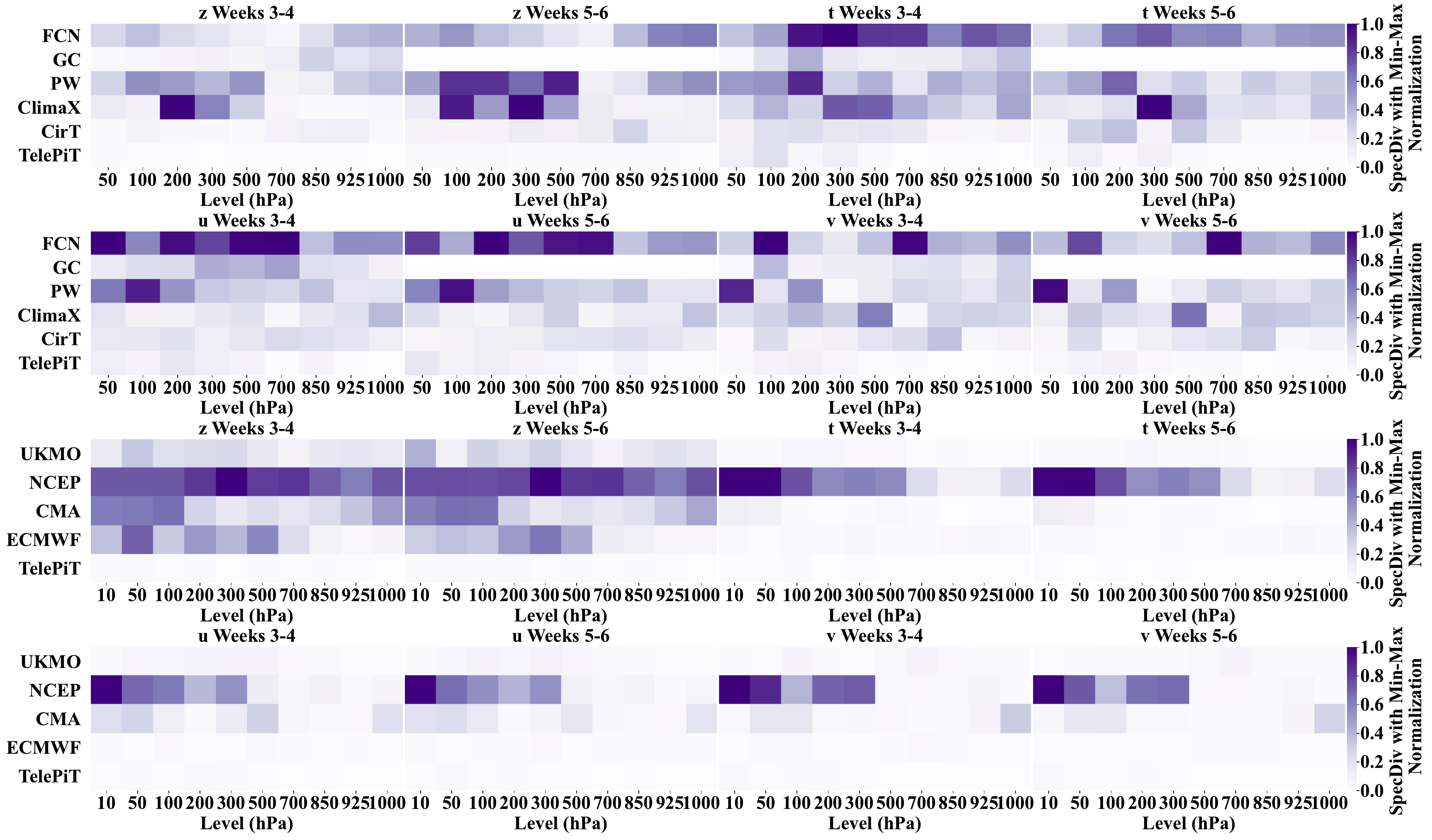}
    \vspace{-0.3cm}
    \caption{SpecDiv comparison of variable $z$, $t$, $u$, and $v$ on all different pressure levels. This metric quantifies differences in power spectrum distributions between predicted and ground truth fields. Lower values (lighter colors) indicate better spectral match. TelePiT demonstrates consistently lower SpecDiv values across variables and pressure levels, highlighting its ability to accurately capture atmospheric wave characteristics at various spatial scales. GraphCast (GC) is unavailable at Weeks 5-6 due to the out of memory.}
    \label{appendix_fig:heatmap_specdiv}
\end{figure*}

\begin{figure*}[t]
    \centering
    \includegraphics[width=0.84\linewidth]{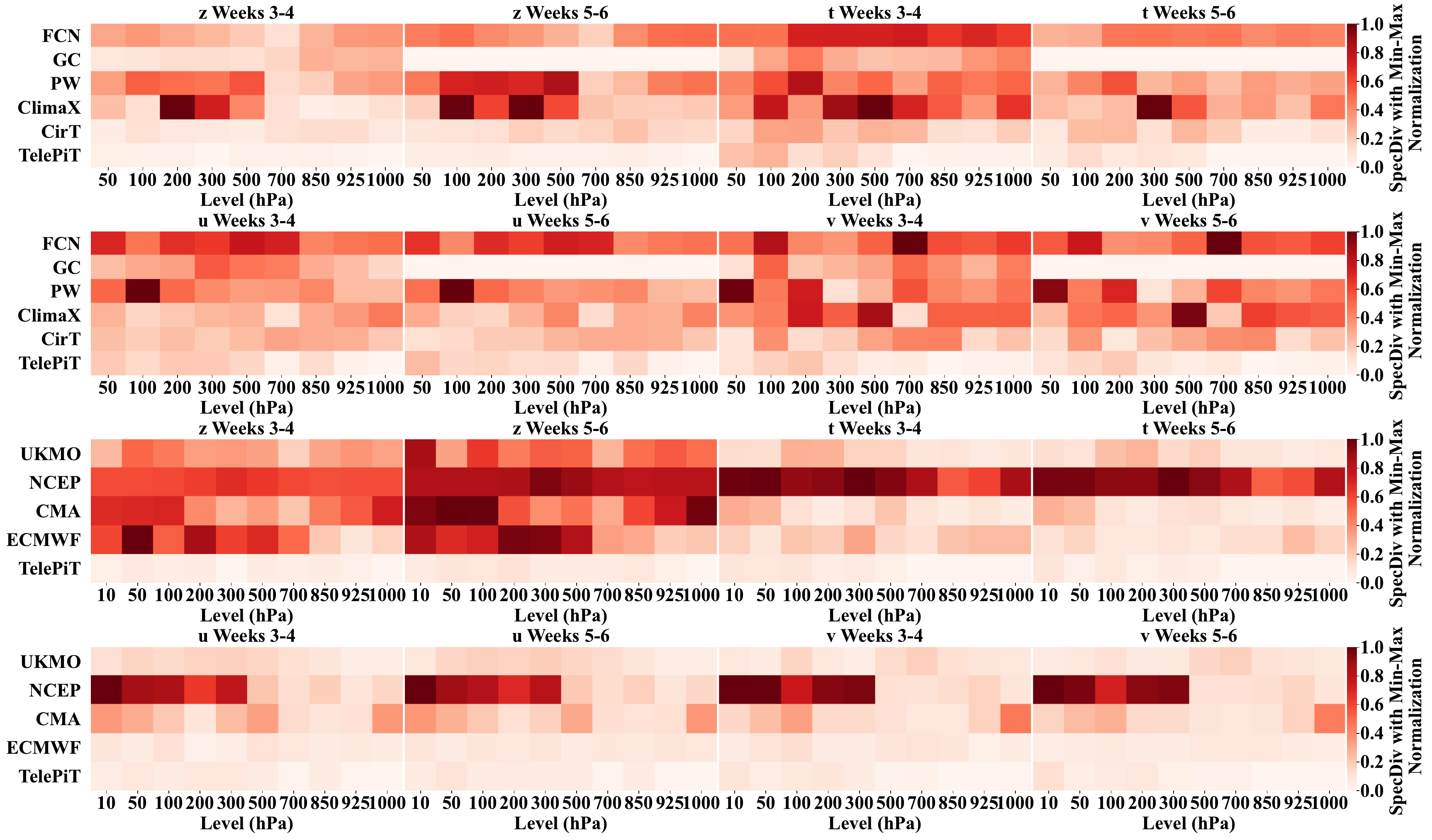}
    \vspace{-0.3cm}
    \caption{SpecRes comparison of variable $z$, $t$, $u$, and $v$ on all different pressure levels. This metric measures spectral energy distribution accuracy. Lower values (lighter colors) indicate better performance. TelePiT exhibits substantially lower SpecRes values compared to all competing models, with particular advantages in temperature and wind fields, demonstrating its enhanced ability to represent both large-scale circulation patterns and finer-scale atmospheric features. GraphCast (GC) is unavailable at Weeks 5-6 due to the out of memory.}
    \label{appendix_fig:heatmap_specres}
\end{figure*}

\clearpage

\begin{table*}[t] 
\let\oldeverydisplay\everydisplay
\let\oldeverymath\everymath
\everydisplay{}
\everymath{}
\definecolor{lightblue}{RGB}{240, 249, 255}  
\definecolor{lightred}{RGB}{255, 255, 255}   
\small
\setlength{\tabcolsep}{4pt} 
\caption{Extended ablation study comparing \textbf{TelePiT} against variants with removed components (SHE: Spherical Harmonic Embedding, WD: Wavelet Decomposition, ODE: Physics-Informed ODE, TA: Teleconnection Attention) using MS-SSIM (higher is better), SpecDiv and SpecRec (lower is better) metrics across various atmospheric variables for both Weeks 3-4 and 5-6 forecast horizons.}
\label{appendix_tab:ablation}
  \begin{tabular}{cc|ccccccc|ccccccc}
    \toprule
    \multirow{2}{*}{\textbf{Metric}} & \multirow{2}{*}{\textbf{Model}} & \multicolumn{7}{c|}{\textbf{Weeks 3-4}} & \multicolumn{7}{c}{\textbf{Weeks 5-6}}\\ 
    & & z500 & z850 & t500 & t850 & t2m & u10 & v10 & z500 & z850 & t500 & t850 & t2m & u10 & v10 \\

    \midrule
    \multirow{5}{*}{\rotatebox{90}{\textbf{MS-SSIM ($\uparrow$)}}}
    &\multicolumn{1}{>{\columncolor{lightblue}}c|}{\text{ w/o SHE }}  
    &\multicolumn{1}{>{\columncolor{lightblue}}c}{0.9066} & \multicolumn{1}{>{\columncolor{lightblue}}c}{0.8671} & \multicolumn{1}{>{\columncolor{lightblue}}c}{0.9198}& \multicolumn{1}{>{\columncolor{lightblue}}c}{0.9434}      
    &\multicolumn{1}{>{\columncolor{lightblue}}c}{0.8399} & \multicolumn{1}{>{\columncolor{lightblue}}c}{0.9398} & \multicolumn{1}{>{\columncolor{lightblue}}c|}{{0.9346}} 
    &\multicolumn{1}{>{\columncolor{lightblue}}c}{0.9046} & \multicolumn{1}{>{\columncolor{lightblue}}c}{0.8668} & \multicolumn{1}{>{\columncolor{lightblue}}c}{{0.9184}}& \multicolumn{1}{>{\columncolor{lightblue}}c}{{0.9421}}      
    &\multicolumn{1}{>{\columncolor{lightblue}}c}{0.8396} & \multicolumn{1}{>{\columncolor{lightblue}}c}{0.9372} & \multicolumn{1}{>{\columncolor{lightblue}}c}{{0.9337}} \\
    
    &\multicolumn{1}{>{\columncolor{lightred}}c|}{\text{ w/o WD }}  
    &\multicolumn{1}{>{\columncolor{lightred}}c}{0.9013} & \multicolumn{1}{>{\columncolor{lightred}}c}{0.8635} & \multicolumn{1}{>{\columncolor{lightred}}c}{{0.9176}}& \multicolumn{1}{>{\columncolor{lightred}}c}{{0.9404}}      
    &\multicolumn{1}{>{\columncolor{lightred}}c}{0.8276} & \multicolumn{1}{>{\columncolor{lightred}}c}{0.9379} & \multicolumn{1}{>{\columncolor{lightred}}c|}{{0.9342}} 
    &\multicolumn{1}{>{\columncolor{lightred}}c}{0.9037} & \multicolumn{1}{>{\columncolor{lightred}}c}{0.8647} & \multicolumn{1}{>{\columncolor{lightred}}c}{{0.9187}}& \multicolumn{1}{>{\columncolor{lightred}}c}{{0.9429}}      
    &\multicolumn{1}{>{\columncolor{lightred}}c}{0.8282} & \multicolumn{1}{>{\columncolor{lightred}}c}{0.9382} & \multicolumn{1}{>{\columncolor{lightred}}c}{{0.9343}} \\

    &\multicolumn{1}{>{\columncolor{lightblue}}c|}{\text{ w/o ODE }}  
    &\multicolumn{1}{>{\columncolor{lightblue}}c}{0.9050} & \multicolumn{1}{>{\columncolor{lightblue}}c}{0.8679} & \multicolumn{1}{>{\columncolor{lightblue}}c}{{0.9182}}& \multicolumn{1}{>{\columncolor{lightblue}}c}{{0.9415}}      
    &\multicolumn{1}{>{\columncolor{lightblue}}c}{0.8489} & \multicolumn{1}{>{\columncolor{lightblue}}c}{0.9391} & \multicolumn{1}{>{\columncolor{lightblue}}c|}{{0.9353}} 
    &\multicolumn{1}{>{\columncolor{lightblue}}c}{0.9049} & \multicolumn{1}{>{\columncolor{lightblue}}c}{0.8705} & \multicolumn{1}{>{\columncolor{lightblue}}c}{{0.9195}}& \multicolumn{1}{>{\columncolor{lightblue}}c}{{0.9419}}      
    &\multicolumn{1}{>{\columncolor{lightblue}}c}{0.8492} & \multicolumn{1}{>{\columncolor{lightblue}}c}{0.9383} & \multicolumn{1}{>{\columncolor{lightblue}}c}{{0.9352}} \\
    
    &\multicolumn{1}{>{\columncolor{lightred}}c|}{\text{ w/o TA }}  
    &\multicolumn{1}{>{\columncolor{lightred}}c}{0.9017} & \multicolumn{1}{>{\columncolor{lightred}}c}{0.8637} & \multicolumn{1}{>{\columncolor{lightred}}c}{{0.9162}}& \multicolumn{1}{>{\columncolor{lightred}}c}{{0.9400}}      
    &\multicolumn{1}{>{\columncolor{lightred}}c}{0.8292} & \multicolumn{1}{>{\columncolor{lightred}}c}{0.9363} & \multicolumn{1}{>{\columncolor{lightred}}c|}{{0.9349}} 
    &\multicolumn{1}{>{\columncolor{lightred}}c}{0.9043} & \multicolumn{1}{>{\columncolor{lightred}}c}{0.8689} & \multicolumn{1}{>{\columncolor{lightred}}c}{{0.9184}}& \multicolumn{1}{>{\columncolor{lightred}}c}{{0.9412}}      
    &\multicolumn{1}{>{\columncolor{lightred}}c}{0.8291} & \multicolumn{1}{>{\columncolor{lightred}}c}{0.9354} & \multicolumn{1}{>{\columncolor{lightred}}c}{{0.9351}} \\

    &\multicolumn{1}{>{\columncolor{lightblue}}c|}{\textbf{ TelePiT }}  
    &\multicolumn{1}{>{\columncolor{lightblue}}c}{\textbf{0.9146}} & \multicolumn{1}{>{\columncolor{lightblue}}c}{\textbf{0.8786}} & \multicolumn{1}{>{\columncolor{lightblue}}c}{\textbf{0.9263}}& \multicolumn{1}{>{\columncolor{lightblue}}c}{\textbf{0.9473}}      
    &\multicolumn{1}{>{\columncolor{lightblue}}c}{\textbf{0.8568}} & \multicolumn{1}{>{\columncolor{lightblue}}c}{\textbf{0.9431}} & \multicolumn{1}{>{\columncolor{lightblue}}c|}{\textbf{0.9408}} 
    &\multicolumn{1}{>{\columncolor{lightblue}}c}{\textbf{0.9159}} & \multicolumn{1}{>{\columncolor{lightblue}}c}{\textbf{0.8800}} & \multicolumn{1}{>{\columncolor{lightblue}}c}{\textbf{0.9269}}& \multicolumn{1}{>{\columncolor{lightblue}}c}{\textbf{0.9473}}      
    &\multicolumn{1}{>{\columncolor{lightblue}}c}{\textbf{0.8568}} & \multicolumn{1}{>{\columncolor{lightblue}}c}{\textbf{0.9437}} & \multicolumn{1}{>{\columncolor{lightblue}}c}{\textbf{0.9404}} \\

    \midrule
    \multirow{5}{*}{\rotatebox{90}{\textbf{SpecDiv ($\downarrow$)}}}
    &\multicolumn{1}{>{\columncolor{lightblue}}c|}{\text{ w/o SHE }}  
    &\multicolumn{1}{>{\columncolor{lightblue}}c}{0.0429} & \multicolumn{1}{>{\columncolor{lightblue}}c}{0.0745} & \multicolumn{1}{>{\columncolor{lightblue}}c}{0.0316}& \multicolumn{1}{>{\columncolor{lightblue}}c}{0.0184}      
    &\multicolumn{1}{>{\columncolor{lightblue}}c}{0.0275} & \multicolumn{1}{>{\columncolor{lightblue}}c}{0.0131} & \multicolumn{1}{>{\columncolor{lightblue}}c|}{{0.0206}} 
    &\multicolumn{1}{>{\columncolor{lightblue}}c}{0.0447} & \multicolumn{1}{>{\columncolor{lightblue}}c}{0.0734} & \multicolumn{1}{>{\columncolor{lightblue}}c}{{0.0300}}& \multicolumn{1}{>{\columncolor{lightblue}}c}{{0.0183}}      
    &\multicolumn{1}{>{\columncolor{lightblue}}c}{0.0275} & \multicolumn{1}{>{\columncolor{lightblue}}c}{0.0126} & \multicolumn{1}{>{\columncolor{lightblue}}c}{{0.0236}} \\
    
    &\multicolumn{1}{>{\columncolor{lightred}}c|}{\text{ w/o WD }}  
    &\multicolumn{1}{>{\columncolor{lightred}}c}{0.1075} & \multicolumn{1}{>{\columncolor{lightred}}c}{0.0365} & \multicolumn{1}{>{\columncolor{lightred}}c}{{0.0175}}& \multicolumn{1}{>{\columncolor{lightred}}c}{{0.0106}}      
    &\multicolumn{1}{>{\columncolor{lightred}}c}{0.0362} & \multicolumn{1}{>{\columncolor{lightred}}c}{0.0169} & \multicolumn{1}{>{\columncolor{lightred}}c|}{{0.0194}} 
    &\multicolumn{1}{>{\columncolor{lightred}}c}{0.1064} & \multicolumn{1}{>{\columncolor{lightred}}c}{0.0436} & \multicolumn{1}{>{\columncolor{lightred}}c}{{0.0297}}& \multicolumn{1}{>{\columncolor{lightred}}c}{{0.0118}}      
    &\multicolumn{1}{>{\columncolor{lightred}}c}{0.0367} & \multicolumn{1}{>{\columncolor{lightred}}c}{0.0163} & \multicolumn{1}{>{\columncolor{lightred}}c}{{0.0224}} \\

    &\multicolumn{1}{>{\columncolor{lightblue}}c|}{\text{ w/o ODE }}  
    &\multicolumn{1}{>{\columncolor{lightblue}}c}{0.0547} & \multicolumn{1}{>{\columncolor{lightblue}}c}{0.0339} & \multicolumn{1}{>{\columncolor{lightblue}}c}{{0.0408}}& \multicolumn{1}{>{\columncolor{lightblue}}c}{{0.0160}}      
    &\multicolumn{1}{>{\columncolor{lightblue}}c}{0.0510} & \multicolumn{1}{>{\columncolor{lightblue}}c}{0.0085} & \multicolumn{1}{>{\columncolor{lightblue}}c|}{{0.0250}} 
    &\multicolumn{1}{>{\columncolor{lightblue}}c}{0.0307} & \multicolumn{1}{>{\columncolor{lightblue}}c}{0.0326} & \multicolumn{1}{>{\columncolor{lightblue}}c}{{0.0380}}& \multicolumn{1}{>{\columncolor{lightblue}}c}{{0.0159}}      
    &\multicolumn{1}{>{\columncolor{lightblue}}c}{0.0509} & \multicolumn{1}{>{\columncolor{lightblue}}c}{0.0089} & \multicolumn{1}{>{\columncolor{lightblue}}c}{{0.0273}} \\
    
    &\multicolumn{1}{>{\columncolor{lightred}}c|}{\text{ w/o TA }}  
    &\multicolumn{1}{>{\columncolor{lightred}}c}{0.0322} & \multicolumn{1}{>{\columncolor{lightred}}c}{0.0362} & \multicolumn{1}{>{\columncolor{lightred}}c}{\textbf{0.0147}}& \multicolumn{1}{>{\columncolor{lightred}}c}{\textbf{0.0032}}      
    &\multicolumn{1}{>{\columncolor{lightred}}c}{0.0277} & \multicolumn{1}{>{\columncolor{lightred}}c}{0.0123} & \multicolumn{1}{>{\columncolor{lightred}}c|}{{0.0196}} 
    &\multicolumn{1}{>{\columncolor{lightred}}c}{0.0474} & \multicolumn{1}{>{\columncolor{lightred}}c}{0.0364} & \multicolumn{1}{>{\columncolor{lightred}}c}{\textbf{0.0146}}& \multicolumn{1}{>{\columncolor{lightred}}c}{\textbf{0.0026}}      
    &\multicolumn{1}{>{\columncolor{lightred}}c}{0.0276} & \multicolumn{1}{>{\columncolor{lightred}}c}{0.0121} & \multicolumn{1}{>{\columncolor{lightred}}c}{{0.0231}} \\

    &\multicolumn{1}{>{\columncolor{lightblue}}c|}{\textbf{ TelePiT }}  
    &\multicolumn{1}{>{\columncolor{lightblue}}c}{\textbf{0.0180}} & \multicolumn{1}{>{\columncolor{lightblue}}c}{\textbf{0.0175}} & \multicolumn{1}{>{\columncolor{lightblue}}c}{0.0150}& \multicolumn{1}{>{\columncolor{lightblue}}c}{0.0045}      
    &\multicolumn{1}{>{\columncolor{lightblue}}c}{\textbf{0.0018}} & \multicolumn{1}{>{\columncolor{lightblue}}c}{\textbf{0.0020}} & \multicolumn{1}{>{\columncolor{lightblue}}c|}{\textbf{0.0061}} 
    &\multicolumn{1}{>{\columncolor{lightblue}}c}{\textbf{0.0161}} & \multicolumn{1}{>{\columncolor{lightblue}}c}{\textbf{0.0168}} & \multicolumn{1}{>{\columncolor{lightblue}}c}{0.0184}& \multicolumn{1}{>{\columncolor{lightblue}}c}{0.0042}      
    &\multicolumn{1}{>{\columncolor{lightblue}}c}{\textbf{0.0018}} & \multicolumn{1}{>{\columncolor{lightblue}}c}{\textbf{0.0026}} & \multicolumn{1}{>{\columncolor{lightblue}}c}{\textbf{0.0051}} \\

    \midrule
    \multirow{5}{*}{\rotatebox{90}{\textbf{SpecRec ($\downarrow$)}}}
    &\multicolumn{1}{>{\columncolor{lightblue}}c|}{\text{ w/o SHE }}  
    &\multicolumn{1}{>{\columncolor{lightblue}}c}{0.0219} & \multicolumn{1}{>{\columncolor{lightblue}}c}{0.0229} & \multicolumn{1}{>{\columncolor{lightblue}}c}{0.0188}& \multicolumn{1}{>{\columncolor{lightblue}}c}{0.0138}      
    &\multicolumn{1}{>{\columncolor{lightblue}}c}{0.0158} & \multicolumn{1}{>{\columncolor{lightblue}}c}{0.0121} & \multicolumn{1}{>{\columncolor{lightblue}}c|}{{0.0130}} 
    &\multicolumn{1}{>{\columncolor{lightblue}}c}{0.0225} & \multicolumn{1}{>{\columncolor{lightblue}}c}{0.0296} & \multicolumn{1}{>{\columncolor{lightblue}}c}{{0.0183}}& \multicolumn{1}{>{\columncolor{lightblue}}c}{{0.0137}}      
    &\multicolumn{1}{>{\columncolor{lightblue}}c}{0.0158} & \multicolumn{1}{>{\columncolor{lightblue}}c}{0.0119} & \multicolumn{1}{>{\columncolor{lightblue}}c}{{0.0140}} \\
    
    &\multicolumn{1}{>{\columncolor{lightred}}c|}{\text{ w/o WD }}  
    &\multicolumn{1}{>{\columncolor{lightred}}c}{0.0330} & \multicolumn{1}{>{\columncolor{lightred}}c}{0.0185} & \multicolumn{1}{>{\columncolor{lightred}}c}{{0.0135}}& \multicolumn{1}{>{\columncolor{lightred}}c}{{0.0101}}      
    &\multicolumn{1}{>{\columncolor{lightred}}c}{0.0175} & \multicolumn{1}{>{\columncolor{lightred}}c}{0.0125} & \multicolumn{1}{>{\columncolor{lightred}}c|}{{0.0138}} 
    &\multicolumn{1}{>{\columncolor{lightred}}c}{0.0321} & \multicolumn{1}{>{\columncolor{lightred}}c}{0.0207} & \multicolumn{1}{>{\columncolor{lightred}}c}{{0.0183}}& \multicolumn{1}{>{\columncolor{lightred}}c}{{0.0107}}      
    &\multicolumn{1}{>{\columncolor{lightred}}c}{0.0177} & \multicolumn{1}{>{\columncolor{lightred}}c}{0.0124} & \multicolumn{1}{>{\columncolor{lightred}}c}{{0.0149}} \\

    &\multicolumn{1}{>{\columncolor{lightblue}}c|}{\text{ w/o ODE }}  
    &\multicolumn{1}{>{\columncolor{lightblue}}c}{0.0239} & \multicolumn{1}{>{\columncolor{lightblue}}c}{0.0177} & \multicolumn{1}{>{\columncolor{lightblue}}c}{{0.0202}}& \multicolumn{1}{>{\columncolor{lightblue}}c}{{0.0120}}      
    &\multicolumn{1}{>{\columncolor{lightblue}}c}{0.0183} & \multicolumn{1}{>{\columncolor{lightblue}}c}{0.0088} & \multicolumn{1}{>{\columncolor{lightblue}}c|}{{0.0140}} 
    &\multicolumn{1}{>{\columncolor{lightblue}}c}{0.0185} & \multicolumn{1}{>{\columncolor{lightblue}}c}{0.0174} & \multicolumn{1}{>{\columncolor{lightblue}}c}{{0.0191}}& \multicolumn{1}{>{\columncolor{lightblue}}c}{{0.0119}}      
    &\multicolumn{1}{>{\columncolor{lightblue}}c}{0.0183} & \multicolumn{1}{>{\columncolor{lightblue}}c}{0.0090} & \multicolumn{1}{>{\columncolor{lightblue}}c}{{0.0148}} \\
    
    &\multicolumn{1}{>{\columncolor{lightred}}c|}{\text{ w/o TA }}  
    &\multicolumn{1}{>{\columncolor{lightred}}c}{0.0186} & \multicolumn{1}{>{\columncolor{lightred}}c}{0.0183} & \multicolumn{1}{>{\columncolor{lightred}}c}{\textbf{0.0132}}& \multicolumn{1}{>{\columncolor{lightred}}c}{\textbf{0.0060}}      
    &\multicolumn{1}{>{\columncolor{lightred}}c}{0.0125} & \multicolumn{1}{>{\columncolor{lightred}}c}{0.0104} & \multicolumn{1}{>{\columncolor{lightred}}c|}{{0.0135}} 
    &\multicolumn{1}{>{\columncolor{lightred}}c}{0.0210} & \multicolumn{1}{>{\columncolor{lightred}}c}{0.0183} & \multicolumn{1}{>{\columncolor{lightred}}c}{\textbf{0.0127}}& \multicolumn{1}{>{\columncolor{lightred}}c}{\textbf{0.0054}}      
    &\multicolumn{1}{>{\columncolor{lightred}}c}{0.0125} & \multicolumn{1}{>{\columncolor{lightred}}c}{0.0103} & \multicolumn{1}{>{\columncolor{lightred}}c}{{0.0148}} \\

    &\multicolumn{1}{>{\columncolor{lightblue}}c|}{\textbf{ TelePiT }}  
    &\multicolumn{1}{>{\columncolor{lightblue}}c}{\textbf{0.0127}} & \multicolumn{1}{>{\columncolor{lightblue}}c}{\textbf{0.0134}} & \multicolumn{1}{>{\columncolor{lightblue}}c}{0.0139}& \multicolumn{1}{>{\columncolor{lightblue}}c}{0.0067}     
    &\multicolumn{1}{>{\columncolor{lightblue}}c}{\textbf{0.0042}} & \multicolumn{1}{>{\columncolor{lightblue}}c}{\textbf{0.0048}} & \multicolumn{1}{>{\columncolor{lightblue}}c|}{\textbf{0.0063}} 
    &\multicolumn{1}{>{\columncolor{lightblue}}c}{\textbf{0.0121}} & \multicolumn{1}{>{\columncolor{lightblue}}c}{\textbf{0.0132}} & \multicolumn{1}{>{\columncolor{lightblue}}c}{0.0161}& \multicolumn{1}{>{\columncolor{lightblue}}c}{0.0064}      
    &\multicolumn{1}{>{\columncolor{lightblue}}c}{\textbf{0.0042}} & \multicolumn{1}{>{\columncolor{lightblue}}c}{\textbf{0.0052}} & \multicolumn{1}{>{\columncolor{lightblue}}c}{\textbf{0.0059}} \\

    \bottomrule
  \end{tabular}
\let\everydisplay\oldeverydisplay
\let\everymath\oldeverymath
\vspace{0.5cm}
\end{table*}


\begin{figure*}[t]
    \centering
    \includegraphics[width=1\linewidth]{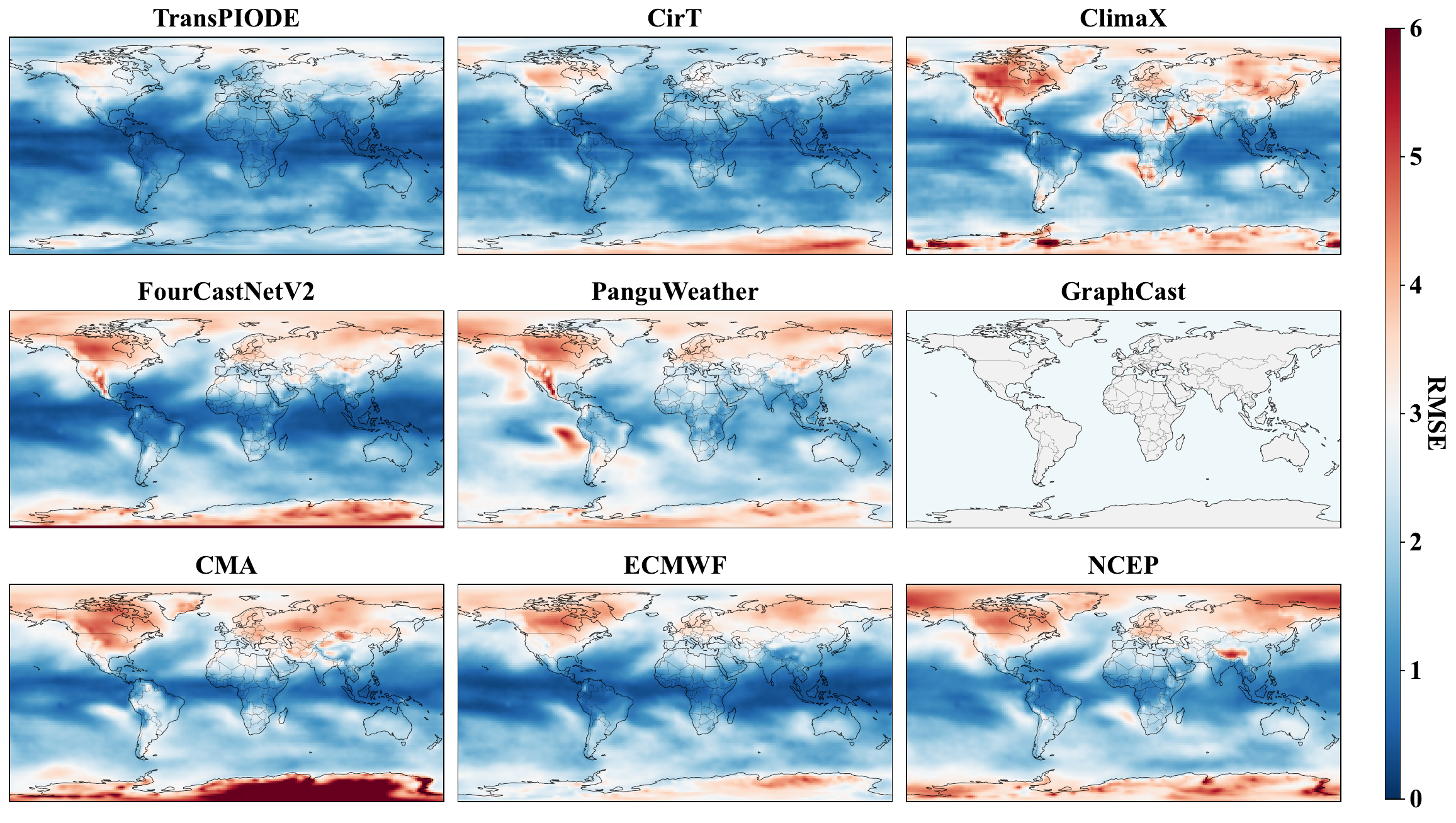}
    \caption{Global visualization of RMSE distribution of $t850$ for Weeks 5-6 forecasts across different models. The color scale represents RMSE values, with darker blue indicating better performance (lower error) and darker red indicating worse performance (higher error). Note that GraphCast results are unavailable at Weeks 5-6 due to out-of-memory issues when attempting extended-range forecasts.}
    \label{appendix_fig:map_t850_1}
\end{figure*}

\begin{figure*}[t]
    \centering
    \includegraphics[width=0.93\linewidth]{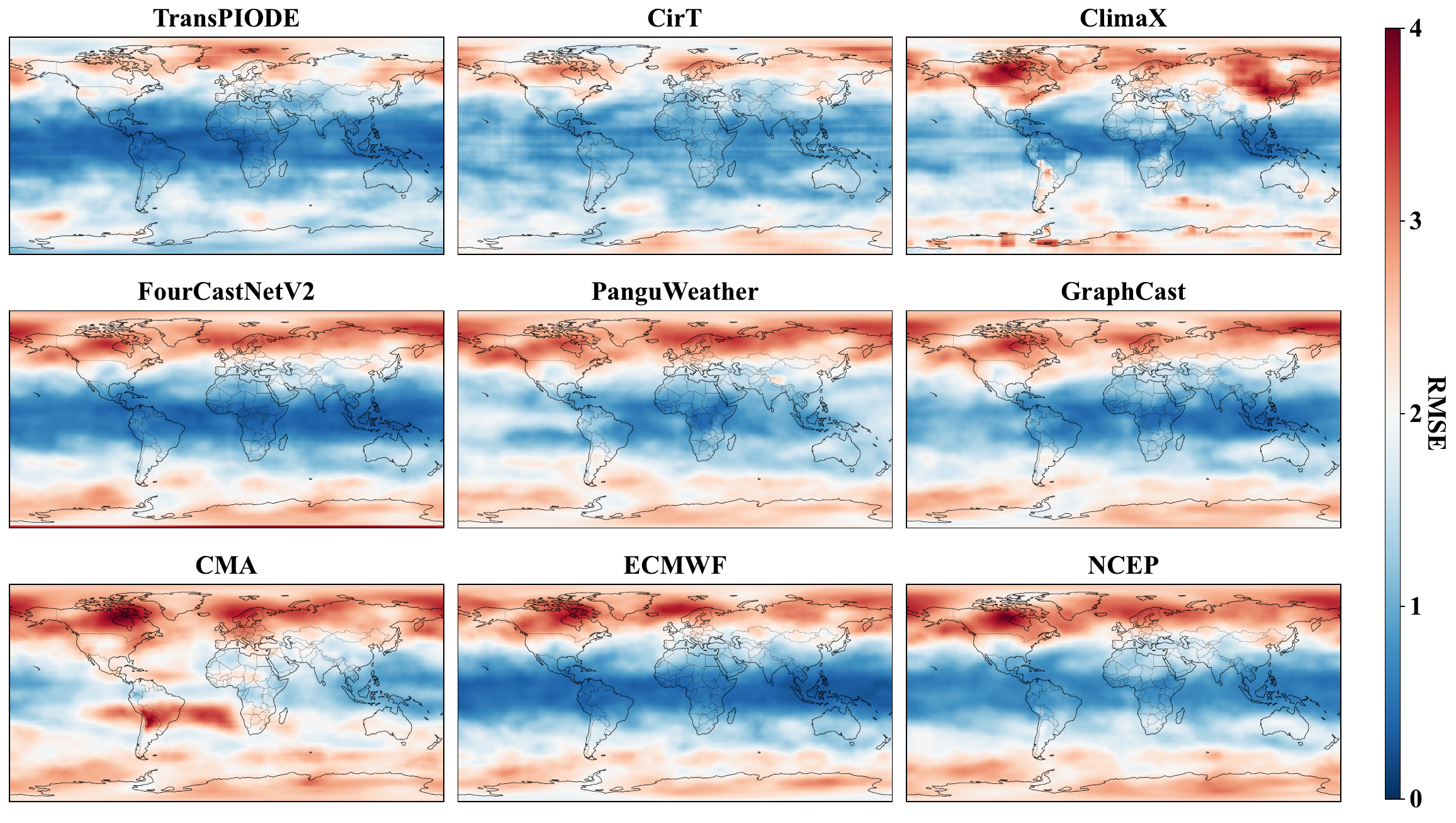}
    \vspace{-0.3cm}
    \caption{Global visualization of RMSE distribution of $t500$ for Weeks 3-4 forecasts across different models. The color scale represents RMSE values, with darker blue indicating better performance (lower error) and darker red indicating worse performance (higher error).}
    \label{appendix_fig:map_t500_0}
\end{figure*}

\begin{figure*}[t]
    \centering
    \includegraphics[width=0.93\linewidth]{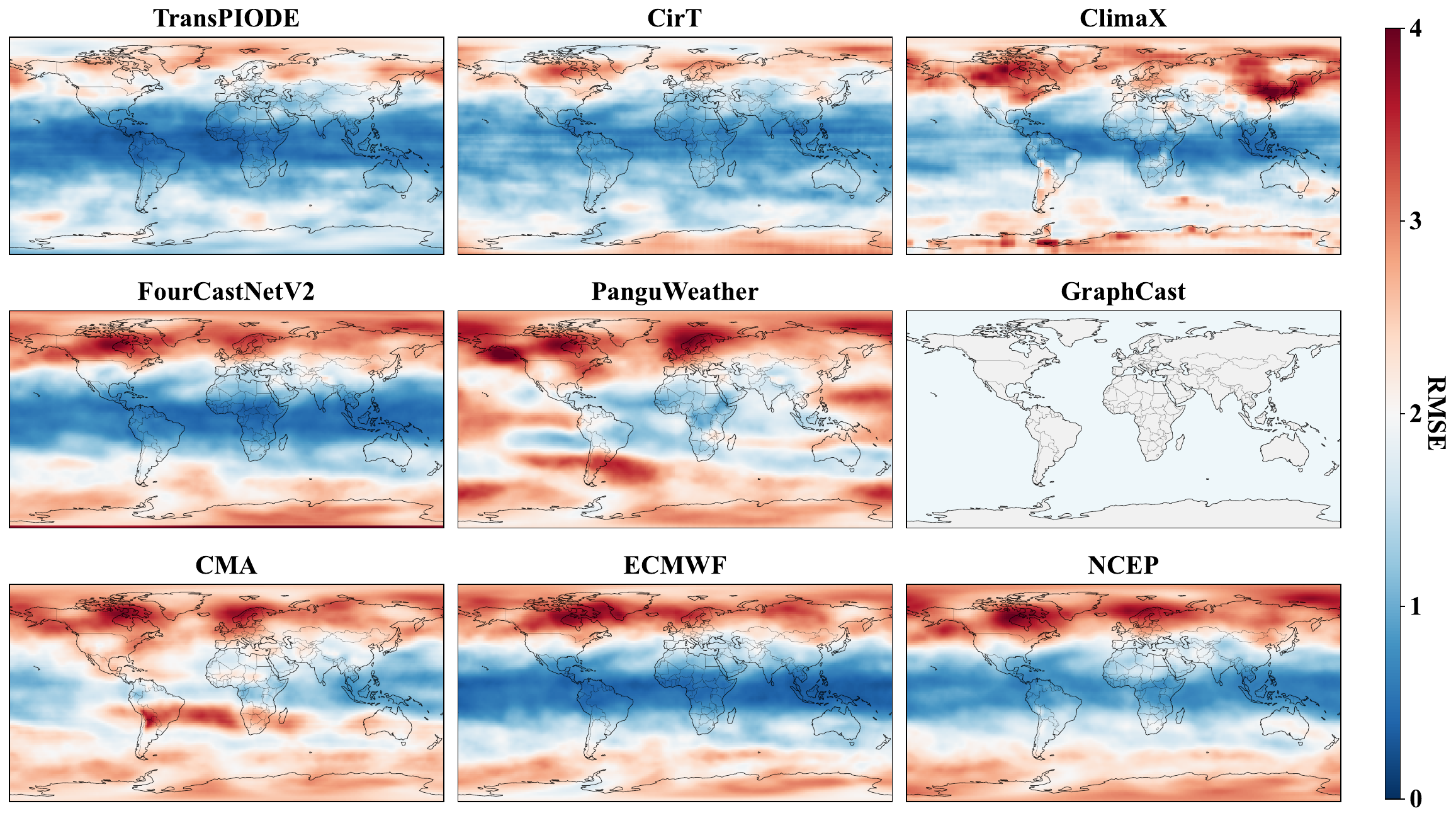}
    \vspace{-0.3cm}
    \caption{Global visualization of RMSE distribution of $t500$ for Weeks 5-6 forecasts across different models. The color scale represents RMSE values, with darker blue indicating better performance (lower error) and darker red indicating worse performance (higher error). Note that GraphCast results are unavailable at Weeks 5-6 due to out-of-memory issues when attempting extended-range forecasts.}
    \label{appendix_fig:map_t500_1}
\end{figure*}

\begin{figure*}[t]
    \centering
    \includegraphics[width=0.93\linewidth]{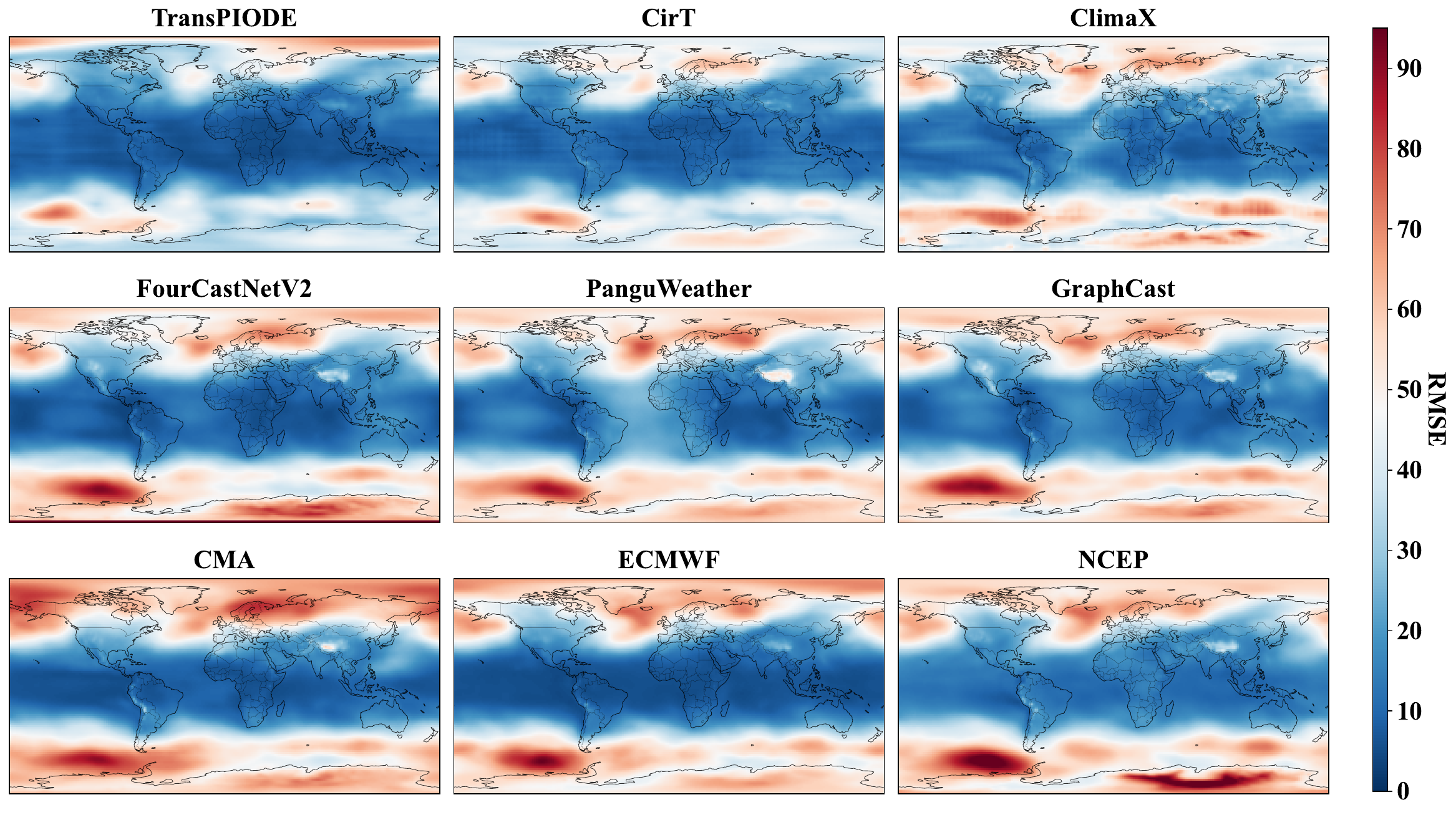}
    \vspace{-0.3cm}
    \caption{Global visualization of RMSE distribution of $z850$ for Weeks 3-4 forecasts across different models. The color scale represents RMSE values, with darker blue indicating better performance (lower error) and darker red indicating worse performance (higher error).}
    \label{appendix_fig:map_z850_0}
\end{figure*}

\begin{figure*}[t]
    \centering
    \includegraphics[width=0.93\linewidth]{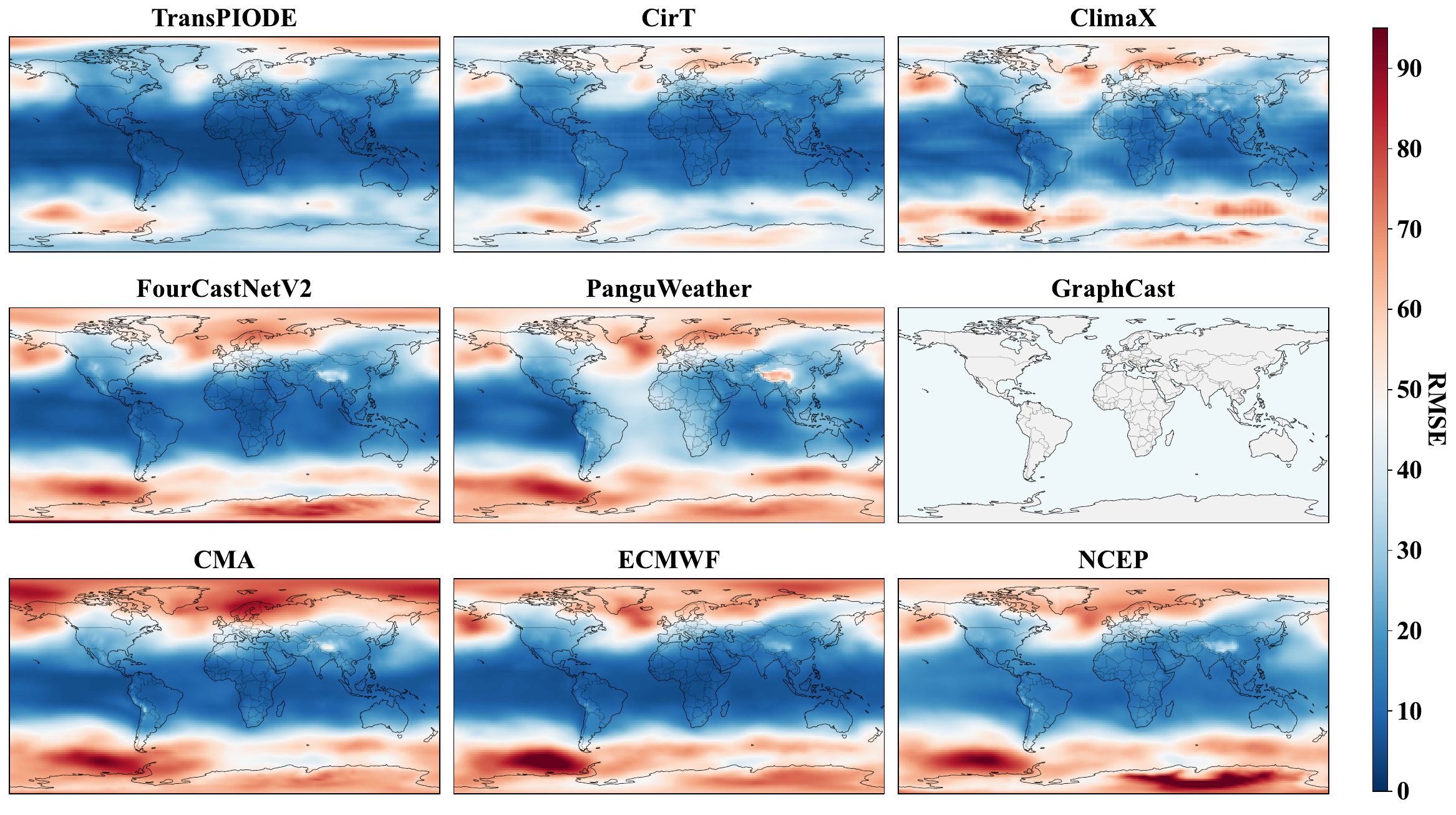}
    \vspace{-0.3cm}
    \caption{Global visualization of RMSE distribution of $z850$ for Weeks 5-6 forecasts across different models. The color scale represents RMSE values, with darker blue indicating better performance (lower error) and darker red indicating worse performance (higher error). Note that GraphCast results are unavailable at Weeks 5-6 due to out-of-memory issues when attempting extended-range forecasts.}
    \label{appendix_fig:map_z850_1}
\end{figure*}

\begin{figure*}[t]
    \centering
    \includegraphics[width=0.93\linewidth]{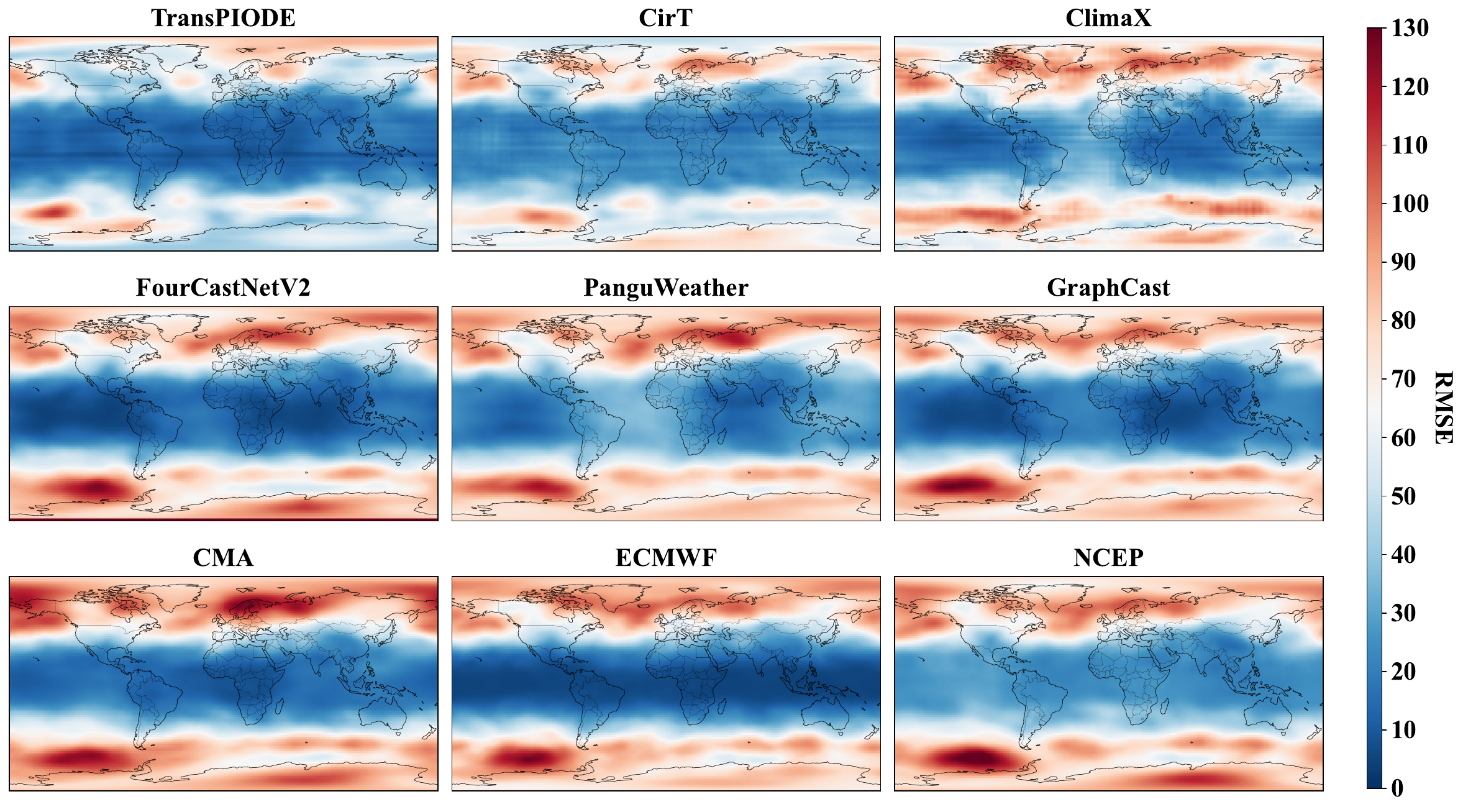}
    \vspace{-0.3cm}
    \caption{Global visualization of RMSE distribution of $z500$ for Weeks 3-4 forecasts across different models. The color scale represents RMSE values, with darker blue indicating better performance (lower error) and darker red indicating worse performance (higher error).}
    \label{appendix_fig:map_z500_0}
\end{figure*}

\begin{figure*}[t]
    \centering
    \includegraphics[width=0.93\linewidth]{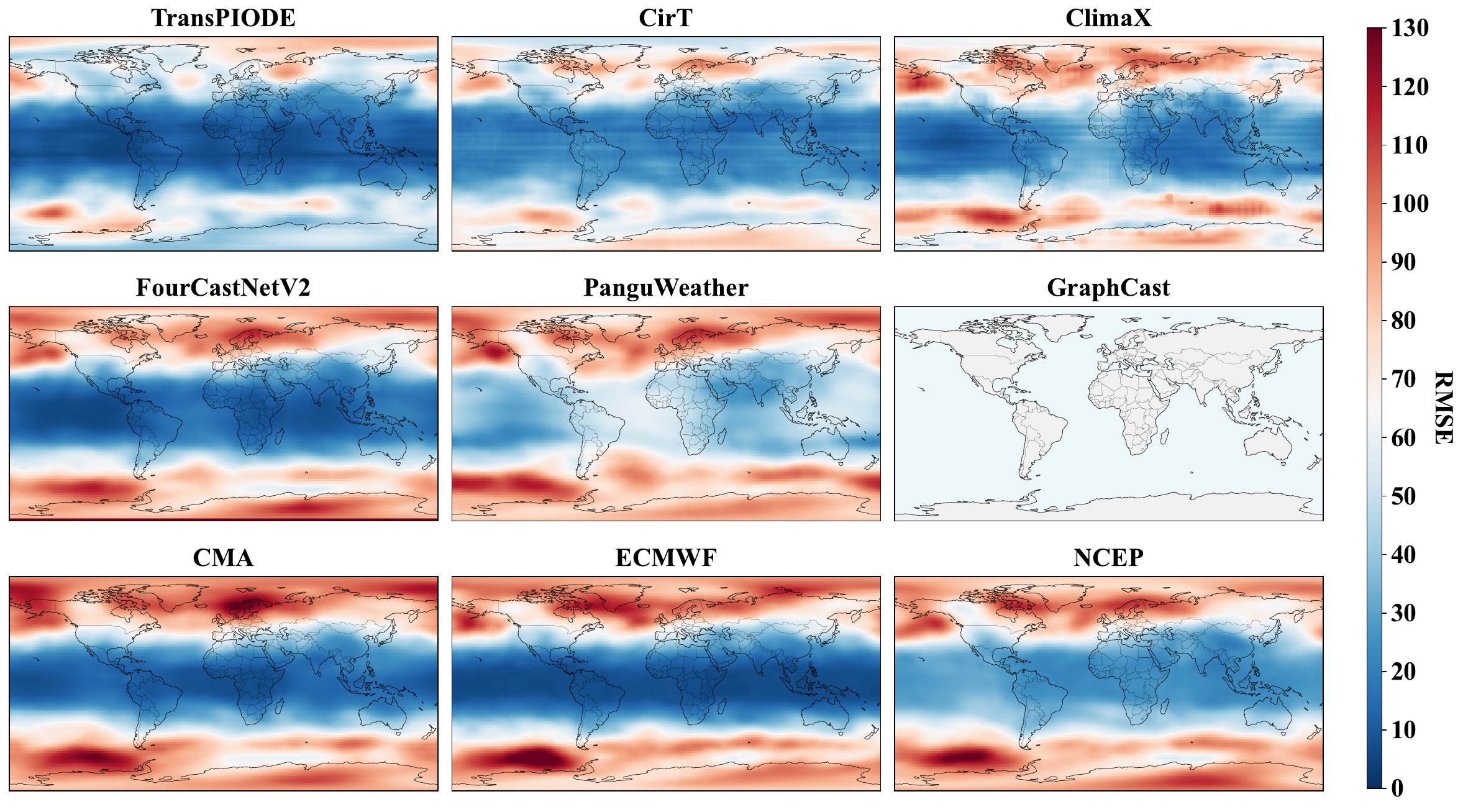}
    \vspace{-0.3cm}
    \caption{Global visualization of RMSE distribution of $z500$ for Weeks 5-6 forecasts across different models. The color scale represents RMSE values, with darker blue indicating better performance (lower error) and darker red indicating worse performance (higher error). Note that GraphCast results are unavailable at Weeks 5-6 due to out-of-memory issues when attempting extended-range forecasts.}
    \label{appendix_fig:map_z500_1}
\end{figure*}


\begin{figure*}[t]
    \centering
    \includegraphics[width=0.85\linewidth]{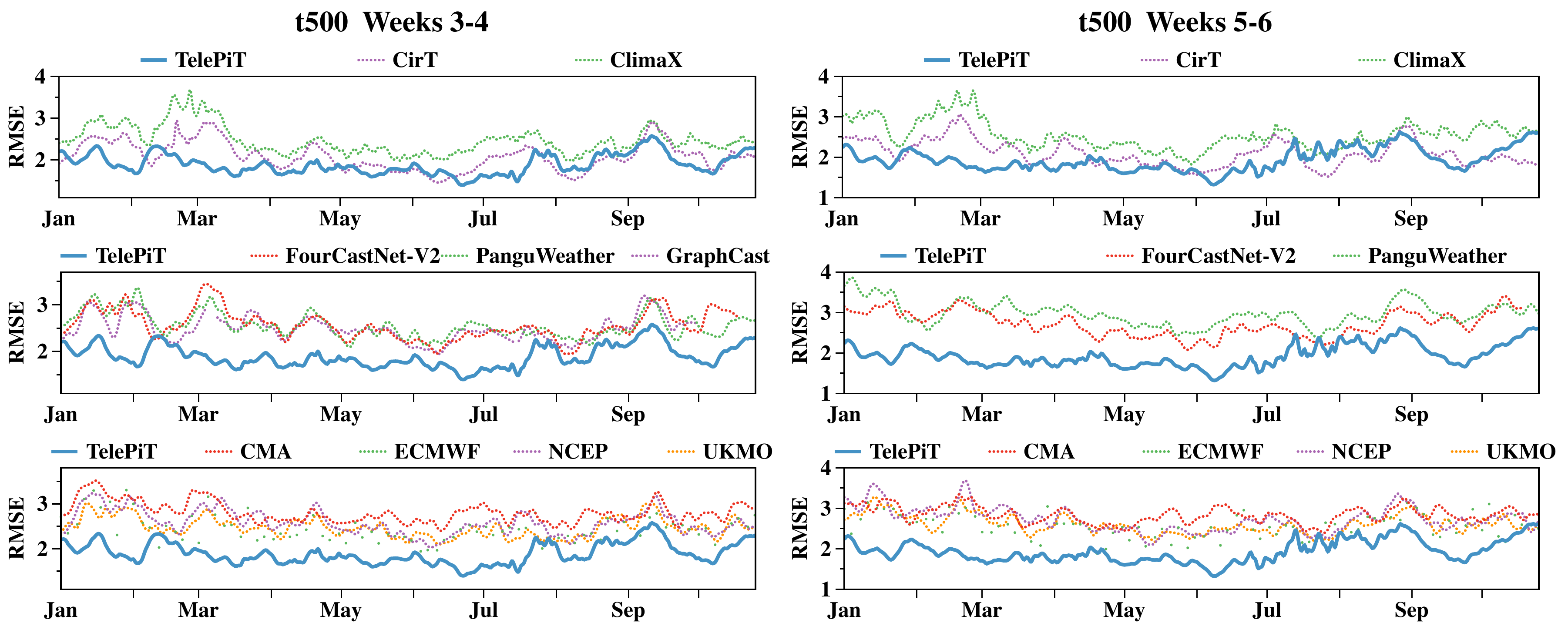}
    \vspace{-0.2cm}
    \caption{Daily forecasting performance (RMSE) of $t500$ in 2018. For improved visualization clarity, a Gaussian smoothing (with $\sigma=1.5$) was applied to the daily values for models exhibiting high variability (Pangu, GraphCast, FourCastNetV2, CMA, ECMWF, NCEP, and UKMO). Note that ECMWF predictions appear as discrete points because their operational system generates forecasts only every 2-3 days, unlike the other models which provide daily predictions.}
    \label{appendix_fig:t500_rmse_2018}
\end{figure*}

\begin{figure*}[t]
    \centering
    \includegraphics[width=0.85\linewidth]{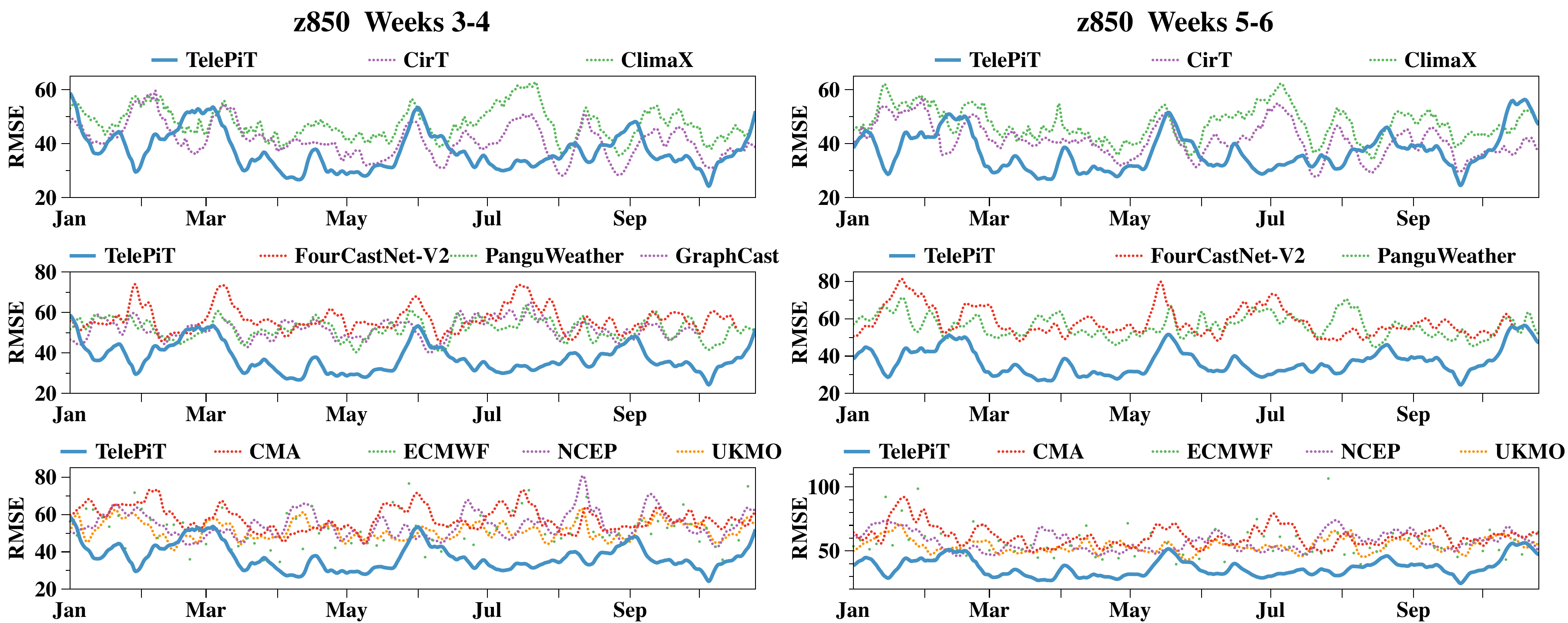}
    \vspace{-0.2cm}
    \caption{Daily forecasting performance (RMSE) of $z850$ in 2018.}
    \label{appendix_fig:z850_rmse_2018}
\end{figure*}

\begin{figure*}[t]
    \centering
    \includegraphics[width=0.85\linewidth]{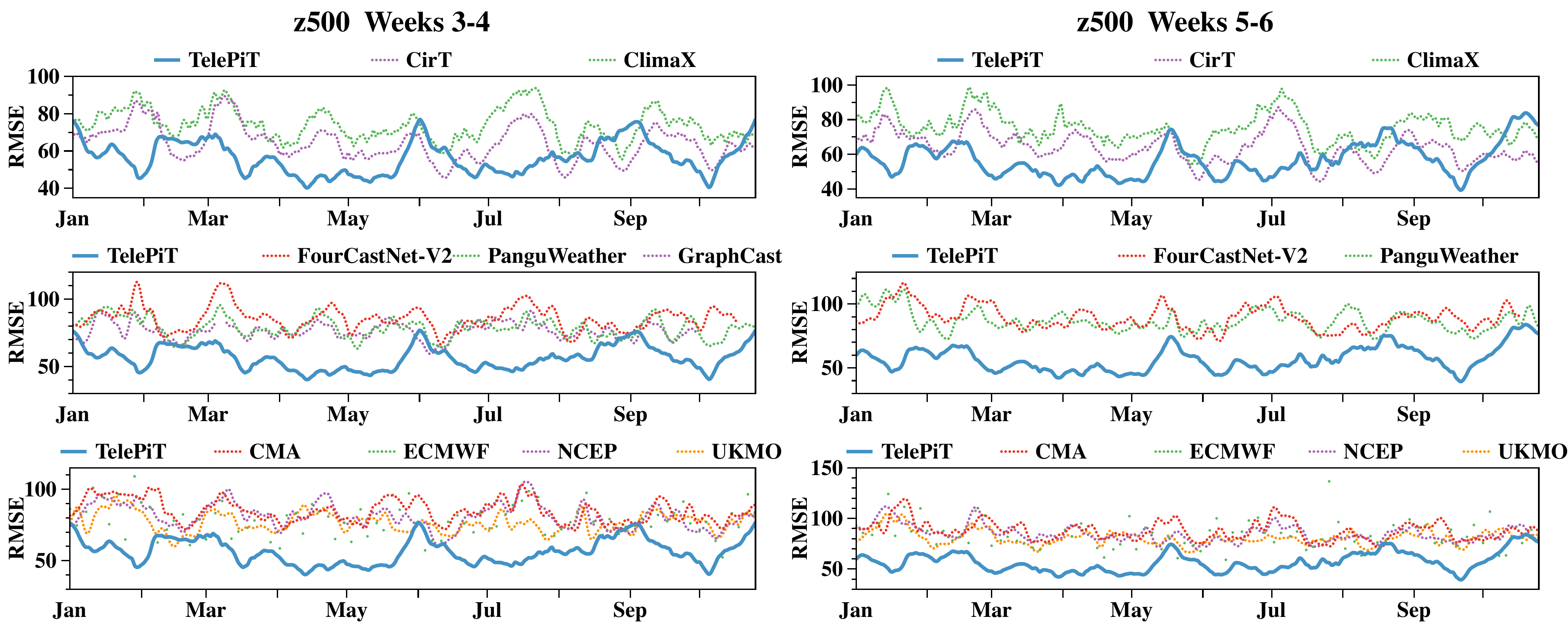}
    \vspace{-0.2cm}
    \caption{Daily forecasting performance (RMSE) of $z500$ in 2018.}
    \label{appendix_fig:z500_rmse_2018}
\end{figure*}

\begin{table*}[t] 
\let\oldeverydisplay\everydisplay
\let\oldeverymath\everymath
\everydisplay{}
\everymath{}
\definecolor{lightblue}{RGB}{240, 249, 255}  
\definecolor{lightred}{RGB}{255, 255, 255}   
\small
\setlength{\tabcolsep}{3pt} 
\caption{Robustness study on 2019 data (out-of-sample) comparing \textbf{TelePiT} against baselines across all metrics, demonstrating consistent performance advantages and strong generalization capabilities. Operational systems (NCEP, UKMO, CMA, ECMWF) results are obtained from ChaosBench and do not include $t2m$, $u10$, $v10$ predictions. GraphCast results are unavailable at Weeks 5-6 due to out-of-memory issues.}

\label{appendix_tab:robustness}
  \begin{tabular}{cc|ccccccc|ccccccc}
    \toprule
    \multirow{2}{*}{\textbf{Metric}} & \multirow{2}{*}{\textbf{Model}} & \multicolumn{7}{c|}{\textbf{Weeks 3-4}} & \multicolumn{7}{c}{\textbf{Weeks 5-6}}\\ 
    & & z500 & z850 & t500 & t850 & t2m & u10 & v10 & z500 & z850 & t500 & t850 & t2m & u10 & v10 \\

    \midrule
    \multirow{10}{*}{\rotatebox{90}{\textbf{RMSE ($\downarrow$)}}} 
    &\multicolumn{1}{>{\columncolor{lightred}}c|}{\text{ NCEP }}  
    &\multicolumn{1}{>{\columncolor{lightred}}c}{66.7175} & \multicolumn{1}{>{\columncolor{lightred}}c}{44.4719} & \multicolumn{1}{>{\columncolor{lightred}}c}{2.1620}& \multicolumn{1}{>{\columncolor{lightred}}c}{2.7271}      
    &\multicolumn{1}{>{\columncolor{lightred}}c}{--} & \multicolumn{1}{>{\columncolor{lightred}}c}{--} & \multicolumn{1}{>{\columncolor{lightred}}c|}{--} 
    &\multicolumn{1}{>{\columncolor{lightred}}c}{69.6268} & \multicolumn{1}{>{\columncolor{lightred}}c}{46.0469} & \multicolumn{1}{>{\columncolor{lightred}}c}{2.3193}& \multicolumn{1}{>{\columncolor{lightred}}c}{2.8611}      
    &\multicolumn{1}{>{\columncolor{lightred}}c}{--} & \multicolumn{1}{>{\columncolor{lightred}}c}{--} & \multicolumn{1}{>{\columncolor{lightred}}c}{--} \\

    &\multicolumn{1}{>{\columncolor{lightblue}}c|}{\text{ UKMO }}  
    &\multicolumn{1}{>{\columncolor{lightblue}}c}{65.6153} & \multicolumn{1}{>{\columncolor{lightblue}}c}{44.8668} & \multicolumn{1}{>{\columncolor{lightblue}}c}{2.1525}& \multicolumn{1}{>{\columncolor{lightblue}}c}{2.6018}      
    &\multicolumn{1}{>{\columncolor{lightblue}}c}{--} & \multicolumn{1}{>{\columncolor{lightblue}}c}{--} & \multicolumn{1}{>{\columncolor{lightblue}}c|}{--} 
    &\multicolumn{1}{>{\columncolor{lightblue}}c}{69.7506} & \multicolumn{1}{>{\columncolor{lightblue}}c}{47.5905} & \multicolumn{1}{>{\columncolor{lightblue}}c}{2.3109}& \multicolumn{1}{>{\columncolor{lightblue}}c}{2.7552}      
    &\multicolumn{1}{>{\columncolor{lightblue}}c}{--} & \multicolumn{1}{>{\columncolor{lightblue}}c}{--} & \multicolumn{1}{>{\columncolor{lightblue}}c}{--} \\
    
    &\multicolumn{1}{>{\columncolor{lightred}}c|}{\text{ CMA }}  
    &\multicolumn{1}{>{\columncolor{lightred}}c}{68.3561} & \multicolumn{1}{>{\columncolor{lightred}}c}{47.2971} & \multicolumn{1}{>{\columncolor{lightred}}c}{2.4489}& \multicolumn{1}{>{\columncolor{lightred}}c}{2.8035}      
    &\multicolumn{1}{>{\columncolor{lightred}}c}{--} & \multicolumn{1}{>{\columncolor{lightred}}c}{--} & \multicolumn{1}{>{\columncolor{lightred}}c|}{--} 
    &\multicolumn{1}{>{\columncolor{lightred}}c}{71.7510} & \multicolumn{1}{>{\columncolor{lightred}}c}{49.6562} & \multicolumn{1}{>{\columncolor{lightred}}c}{2.6457}& \multicolumn{1}{>{\columncolor{lightred}}c}{2.8593}      
    &\multicolumn{1}{>{\columncolor{lightred}}c}{--} & \multicolumn{1}{>{\columncolor{lightred}}c}{--} & \multicolumn{1}{>{\columncolor{lightred}}c}{--} \\
    
    &\multicolumn{1}{>{\columncolor{lightblue}}c|}{\text{ ECMWF }}  
    &\multicolumn{1}{>{\columncolor{lightblue}}c}{62.0551} & \multicolumn{1}{>{\columncolor{lightblue}}c}{42.0964} & \multicolumn{1}{>{\columncolor{lightblue}}c}{2.0025}& \multicolumn{1}{>{\columncolor{lightblue}}c}{2.3296}      
    &\multicolumn{1}{>{\columncolor{lightblue}}c}{--} & \multicolumn{1}{>{\columncolor{lightblue}}c}{--} & \multicolumn{1}{>{\columncolor{lightblue}}c|}{--} 
    &\multicolumn{1}{>{\columncolor{lightblue}}c}{67.0798} & \multicolumn{1}{>{\columncolor{lightblue}}c}{44.9257} & \multicolumn{1}{>{\columncolor{lightblue}}c}{2.1850}& \multicolumn{1}{>{\columncolor{lightblue}}c}{2.5142}      
    &\multicolumn{1}{>{\columncolor{lightblue}}c}{--} & \multicolumn{1}{>{\columncolor{lightblue}}c}{--} & \multicolumn{1}{>{\columncolor{lightblue}}c}{--} \\

    &\multicolumn{1}{>{\columncolor{lightred}}c|}{\text{ GraphCast }}  
    &\multicolumn{1}{>{\columncolor{lightred}}c}{63.4077} & \multicolumn{1}{>{\columncolor{lightred}}c}{42.8811} & \multicolumn{1}{>{\columncolor{lightred}}c}{2.1750}& \multicolumn{1}{>{\columncolor{lightred}}c}{2.3806}      
    &\multicolumn{1}{>{\columncolor{lightred}}c}{--} & \multicolumn{1}{>{\columncolor{lightred}}c}{--} & \multicolumn{1}{>{\columncolor{lightred}}c|}{--} 
    &\multicolumn{1}{>{\columncolor{lightred}}c}{--} & \multicolumn{1}{>{\columncolor{lightred}}c}{--} & \multicolumn{1}{>{\columncolor{lightred}}c}{--}& \multicolumn{1}{>{\columncolor{lightred}}c}{--}      
    &\multicolumn{1}{>{\columncolor{lightred}}c}{--} & \multicolumn{1}{>{\columncolor{lightred}}c}{--} & \multicolumn{1}{>{\columncolor{lightred}}c}{--} \\

    &\multicolumn{1}{>{\columncolor{lightblue}}c|}{\text{ FCNV2 }}  
    &\multicolumn{1}{>{\columncolor{lightblue}}c}{62.4918} & \multicolumn{1}{>{\columncolor{lightblue}}c}{41.6883} & \multicolumn{1}{>{\columncolor{lightblue}}c}{2.0631}& \multicolumn{1}{>{\columncolor{lightblue}}c}{2.3786}      
    &\multicolumn{1}{>{\columncolor{lightblue}}c}{--} & \multicolumn{1}{>{\columncolor{lightblue}}c}{4.1937} & \multicolumn{1}{>{\columncolor{lightblue}}c|}{2.3937} 
    &\multicolumn{1}{>{\columncolor{lightblue}}c}{68.5185} & \multicolumn{1}{>{\columncolor{lightblue}}c}{45.0432} & \multicolumn{1}{>{\columncolor{lightblue}}c}{2.2761}& \multicolumn{1}{>{\columncolor{lightblue}}c}{2.6632}      
    &\multicolumn{1}{>{\columncolor{lightblue}}c}{--} & \multicolumn{1}{>{\columncolor{lightblue}}c}{4.2009} & \multicolumn{1}{>{\columncolor{lightblue}}c}{2.4139} \\
    
    &\multicolumn{1}{>{\columncolor{lightred}}c|}{\text{ Pangu }}  
    &\multicolumn{1}{>{\columncolor{lightred}}c}{65.4311} & \multicolumn{1}{>{\columncolor{lightred}}c}{43.2099} & \multicolumn{1}{>{\columncolor{lightred}}c}{2.2319}& \multicolumn{1}{>{\columncolor{lightred}}c}{2.5877}      
    &\multicolumn{1}{>{\columncolor{lightred}}c}{--} & \multicolumn{1}{>{\columncolor{lightred}}c}{3.9950} & \multicolumn{1}{>{\columncolor{lightred}}c|}{2.3089} 
    &\multicolumn{1}{>{\columncolor{lightred}}c}{76.8535} & \multicolumn{1}{>{\columncolor{lightred}}c}{48.3898} & \multicolumn{1}{>{\columncolor{lightred}}c}{2.7828}& \multicolumn{1}{>{\columncolor{lightred}}c}{3.0028}      
    &\multicolumn{1}{>{\columncolor{lightred}}c}{--} & \multicolumn{1}{>{\columncolor{lightred}}c}{4.0534} & \multicolumn{1}{>{\columncolor{lightred}}c}{2.3238} \\

    &\multicolumn{1}{>{\columncolor{lightblue}}c|}{\text{ ClimaX }}  
    &\multicolumn{1}{>{\columncolor{lightblue}}c}{63.1316} & \multicolumn{1}{>{\columncolor{lightblue}}c}{39.0181} & \multicolumn{1}{>{\columncolor{lightblue}}c}{2.2274}& \multicolumn{1}{>{\columncolor{lightblue}}c}{2.7341}      
    &\multicolumn{1}{>{\columncolor{lightblue}}c}{58.4047} & \multicolumn{1}{>{\columncolor{lightblue}}c}{0.8207} & \multicolumn{1}{>{\columncolor{lightblue}}c|}{0.7755} 
    &\multicolumn{1}{>{\columncolor{lightblue}}c}{64.3506} & \multicolumn{1}{>{\columncolor{lightblue}}c}{39.5042} & \multicolumn{1}{>{\columncolor{lightblue}}c}{2.2819}& \multicolumn{1}{>{\columncolor{lightblue}}c}{2.7875}      
    &\multicolumn{1}{>{\columncolor{lightblue}}c}{58.4240} & \multicolumn{1}{>{\columncolor{lightblue}}c}{0.8210} & \multicolumn{1}{>{\columncolor{lightblue}}c}{0.7772} \\
    
    &\multicolumn{1}{>{\columncolor{lightred}}c|}{\text{ CirT}}  
    &\multicolumn{1}{>{\columncolor{lightred}}c}{53.5886} & \multicolumn{1}{>{\columncolor{lightred}}c}{33.5952} & \multicolumn{1}{>{\columncolor{lightred}}c}{1.8372}& \multicolumn{1}{>{\columncolor{lightred}}c}{2.0697}      
    &\multicolumn{1}{>{\columncolor{lightred}}c}{28.4229} & \multicolumn{1}{>{\columncolor{lightred}}c}{0.5570} & \multicolumn{1}{>{\columncolor{lightred}}c|}{0.5284} 
    &\multicolumn{1}{>{\columncolor{lightred}}c}{53.0649} & \multicolumn{1}{>{\columncolor{lightred}}c}{33.4213} & \multicolumn{1}{>{\columncolor{lightred}}c}{1.8017}& \multicolumn{1}{>{\columncolor{lightred}}c}{2.0853}      
    &\multicolumn{1}{>{\columncolor{lightred}}c}{28.4559} & \multicolumn{1}{>{\columncolor{lightred}}c}{0.5535} & \multicolumn{1}{>{\columncolor{lightred}}c}{0.5256} \\

    &\multicolumn{1}{>{\columncolor{lightblue}}c|}{\textbf{ TelePiT }}  
    &\multicolumn{1}{>{\columncolor{lightblue}}c}{\textbf{49.1308}} & \multicolumn{1}{>{\columncolor{lightblue}}c}{\textbf{31.9932}} & \multicolumn{1}{>{\columncolor{lightblue}}c}{\textbf{1.6652}}& \multicolumn{1}{>{\columncolor{lightblue}}c}{\textbf{1.9258}}      
    &\multicolumn{1}{>{\columncolor{lightblue}}c}{\textbf{12.8109}} & \multicolumn{1}{>{\columncolor{lightblue}}c}{\textbf{0.5073}} & \multicolumn{1}{>{\columncolor{lightblue}}c|}{\textbf{0.4843}} 
    &\multicolumn{1}{>{\columncolor{lightblue}}c}{\textbf{48.6264}} & \multicolumn{1}{>{\columncolor{lightblue}}c}{\textbf{32.0561}} & \multicolumn{1}{>{\columncolor{lightblue}}c}{\textbf{1.6702}}& \multicolumn{1}{>{\columncolor{lightblue}}c}{\textbf{1.9310}}      
    &\multicolumn{1}{>{\columncolor{lightblue}}c}{\textbf{12.8184}} & \multicolumn{1}{>{\columncolor{lightblue}}c}{\textbf{0.5046}} & \multicolumn{1}{>{\columncolor{lightblue}}c}{\textbf{0.4821}} \\

    \midrule
    \multirow{10}{*}{\rotatebox{90}{\textbf{ACC ($\uparrow$)}}}
    &\multicolumn{1}{>{\columncolor{lightred}}c|}{\text{ NCEP }}  
    &\multicolumn{1}{>{\columncolor{lightred}}c}{0.9711} & \multicolumn{1}{>{\columncolor{lightred}}c}{0.9184} & \multicolumn{1}{>{\columncolor{lightred}}c}{0.9811}& \multicolumn{1}{>{\columncolor{lightred}}c}{0.9790}      
    &\multicolumn{1}{>{\columncolor{lightred}}c}{--} & \multicolumn{1}{>{\columncolor{lightred}}c}{--} & \multicolumn{1}{>{\columncolor{lightred}}c|}{--} 
    &\multicolumn{1}{>{\columncolor{lightred}}c}{0.9689} & \multicolumn{1}{>{\columncolor{lightred}}c}{0.9123} & \multicolumn{1}{>{\columncolor{lightred}}c}{0.9783}& \multicolumn{1}{>{\columncolor{lightred}}c}{0.9765}      
    &\multicolumn{1}{>{\columncolor{lightred}}c}{--} & \multicolumn{1}{>{\columncolor{lightred}}c}{--} & \multicolumn{1}{>{\columncolor{lightred}}c}{--} \\

    &\multicolumn{1}{>{\columncolor{lightblue}}c|}{\text{ UKMO }}  
    &\multicolumn{1}{>{\columncolor{lightblue}}c}{0.9705} & \multicolumn{1}{>{\columncolor{lightblue}}c}{0.9181} & \multicolumn{1}{>{\columncolor{lightblue}}c}{0.9808}& \multicolumn{1}{>{\columncolor{lightblue}}c}{0.9761}      
    &\multicolumn{1}{>{\columncolor{lightblue}}c}{--} & \multicolumn{1}{>{\columncolor{lightblue}}c}{--} & \multicolumn{1}{>{\columncolor{lightblue}}c|}{--} 
    &\multicolumn{1}{>{\columncolor{lightblue}}c}{0.9674} & \multicolumn{1}{>{\columncolor{lightblue}}c}{0.9083} & \multicolumn{1}{>{\columncolor{lightblue}}c}{0.9784}& \multicolumn{1}{>{\columncolor{lightblue}}c}{0.9735}      
    &\multicolumn{1}{>{\columncolor{lightblue}}c}{--} & \multicolumn{1}{>{\columncolor{lightblue}}c}{--} & \multicolumn{1}{>{\columncolor{lightblue}}c}{--} \\  
    
    &\multicolumn{1}{>{\columncolor{lightred}}c|}{\text{ CMA }}  
    &\multicolumn{1}{>{\columncolor{lightred}}c}{0.9700} & \multicolumn{1}{>{\columncolor{lightred}}c}{0.9149} & \multicolumn{1}{>{\columncolor{lightred}}c}{0.9781}& \multicolumn{1}{>{\columncolor{lightred}}c}{0.9779}      
    &\multicolumn{1}{>{\columncolor{lightred}}c}{--} & \multicolumn{1}{>{\columncolor{lightred}}c}{--} & \multicolumn{1}{>{\columncolor{lightred}}c|}{--} 
    &\multicolumn{1}{>{\columncolor{lightred}}c}{0.9670} & \multicolumn{1}{>{\columncolor{lightred}}c}{0.9060} & \multicolumn{1}{>{\columncolor{lightred}}c}{0.9754}& \multicolumn{1}{>{\columncolor{lightred}}c}{0.9772}      
    &\multicolumn{1}{>{\columncolor{lightred}}c}{--} & \multicolumn{1}{>{\columncolor{lightred}}c}{--} & \multicolumn{1}{>{\columncolor{lightred}}c}{--} \\
    
    &\multicolumn{1}{>{\columncolor{lightblue}}c|}{\text{ ECMWF }}  
    &\multicolumn{1}{>{\columncolor{lightblue}}c}{0.9735} & \multicolumn{1}{>{\columncolor{lightblue}}c}{0.9246} & \multicolumn{1}{>{\columncolor{lightblue}}c}{0.9834}& \multicolumn{1}{>{\columncolor{lightblue}}c}{0.9841}      
    &\multicolumn{1}{>{\columncolor{lightblue}}c}{--} & \multicolumn{1}{>{\columncolor{lightblue}}c}{--} & \multicolumn{1}{>{\columncolor{lightblue}}c|}{--} 
    &\multicolumn{1}{>{\columncolor{lightblue}}c}{0.9688} & \multicolumn{1}{>{\columncolor{lightblue}}c}{0.9139} & \multicolumn{1}{>{\columncolor{lightblue}}c}{0.9800}& \multicolumn{1}{>{\columncolor{lightblue}}c}{0.9814}      
    &\multicolumn{1}{>{\columncolor{lightblue}}c}{--} & \multicolumn{1}{>{\columncolor{lightblue}}c}{--} & \multicolumn{1}{>{\columncolor{lightblue}}c}{--} \\

    &\multicolumn{1}{>{\columncolor{lightred}}c|}{\text{ GraphCast }}  
    &\multicolumn{1}{>{\columncolor{lightred}}c}{0.9725} & \multicolumn{1}{>{\columncolor{lightred}}c}{0.9252} & \multicolumn{1}{>{\columncolor{lightred}}c}{0.9807}& \multicolumn{1}{>{\columncolor{lightred}}c}{0.9833}      
    &\multicolumn{1}{>{\columncolor{lightred}}c}{--} & \multicolumn{1}{>{\columncolor{lightred}}c}{--} & \multicolumn{1}{>{\columncolor{lightred}}c|}{--} 
    &\multicolumn{1}{>{\columncolor{lightred}}c}{--} & \multicolumn{1}{>{\columncolor{lightred}}c}{--} & \multicolumn{1}{>{\columncolor{lightred}}c}{--}& \multicolumn{1}{>{\columncolor{lightred}}c}{--}      
    &\multicolumn{1}{>{\columncolor{lightred}}c}{--} & \multicolumn{1}{>{\columncolor{lightred}}c}{--} & \multicolumn{1}{>{\columncolor{lightred}}c}{--} \\

    &\multicolumn{1}{>{\columncolor{lightblue}}c|}{\text{ FCNV2 }}  
    &\multicolumn{1}{>{\columncolor{lightblue}}c}{0.9681} & \multicolumn{1}{>{\columncolor{lightblue}}c}{0.9202} & \multicolumn{1}{>{\columncolor{lightblue}}c}{0.9804}& \multicolumn{1}{>{\columncolor{lightblue}}c}{0.9793}      
    &\multicolumn{1}{>{\columncolor{lightblue}}c}{--} & \multicolumn{1}{>{\columncolor{lightblue}}c}{0.3171} & \multicolumn{1}{>{\columncolor{lightblue}}c|}{0.4819} 
    &\multicolumn{1}{>{\columncolor{lightblue}}c}{0.9631} & \multicolumn{1}{>{\columncolor{lightblue}}c}{0.9088} & \multicolumn{1}{>{\columncolor{lightblue}}c}{0.9772}& \multicolumn{1}{>{\columncolor{lightblue}}c}{0.9751}      
    &\multicolumn{1}{>{\columncolor{lightblue}}c}{--} & \multicolumn{1}{>{\columncolor{lightblue}}c}{0.3108} & \multicolumn{1}{>{\columncolor{lightblue}}c}{0.4709} \\
    
    &\multicolumn{1}{>{\columncolor{lightred}}c|}{\text{ Pangu }}  
    &\multicolumn{1}{>{\columncolor{lightred}}c}{0.9713} & \multicolumn{1}{>{\columncolor{lightred}}c}{0.9224} & \multicolumn{1}{>{\columncolor{lightred}}c}{0.9801}& \multicolumn{1}{>{\columncolor{lightred}}c}{0.9806}      
    &\multicolumn{1}{>{\columncolor{lightred}}c}{--} & \multicolumn{1}{>{\columncolor{lightred}}c}{0.3305} & \multicolumn{1}{>{\columncolor{lightred}}c|}{0.4828} 
    &\multicolumn{1}{>{\columncolor{lightred}}c}{0.9636} & \multicolumn{1}{>{\columncolor{lightred}}c}{0.9045} & \multicolumn{1}{>{\columncolor{lightred}}c}{0.9721}& \multicolumn{1}{>{\columncolor{lightred}}c}{0.9753}      
    &\multicolumn{1}{>{\columncolor{lightred}}c}{--} & \multicolumn{1}{>{\columncolor{lightred}}c}{0.3166} & \multicolumn{1}{>{\columncolor{lightred}}c}{0.4684} \\

    &\multicolumn{1}{>{\columncolor{lightblue}}c|}{\text{ ClimaX }}  
    &\multicolumn{1}{>{\columncolor{lightblue}}c}{0.9764} & \multicolumn{1}{>{\columncolor{lightblue}}c}{0.9410} & \multicolumn{1}{>{\columncolor{lightblue}}c}{0.9820}& \multicolumn{1}{>{\columncolor{lightblue}}c}{0.9808}      
    &\multicolumn{1}{>{\columncolor{lightblue}}c}{0.8965} & \multicolumn{1}{>{\columncolor{lightblue}}c}{0.7821} & \multicolumn{1}{>{\columncolor{lightblue}}c|}{0.7781} 
    &\multicolumn{1}{>{\columncolor{lightblue}}c}{0.9747} & \multicolumn{1}{>{\columncolor{lightblue}}c}{0.9380} & \multicolumn{1}{>{\columncolor{lightblue}}c}{0.9808}& \multicolumn{1}{>{\columncolor{lightblue}}c}{0.9797}      
    &\multicolumn{1}{>{\columncolor{lightblue}}c}{0.8965} & \multicolumn{1}{>{\columncolor{lightblue}}c}{0.7808} & \multicolumn{1}{>{\columncolor{lightblue}}c}{0.7776} \\
    
    &\multicolumn{1}{>{\columncolor{lightred}}c|}{\text{ CirT}}  
    &\multicolumn{1}{>{\columncolor{lightred}}c}{0.9812} & \multicolumn{1}{>{\columncolor{lightred}}c}{0.9528} & \multicolumn{1}{>{\columncolor{lightred}}c}{0.9867}& \multicolumn{1}{>{\columncolor{lightred}}c}{0.9877}      
    &\multicolumn{1}{>{\columncolor{lightred}}c}{0.9776} & \multicolumn{1}{>{\columncolor{lightred}}c}{0.9126} & \multicolumn{1}{>{\columncolor{lightred}}c|}{0.9175} 
    &\multicolumn{1}{>{\columncolor{lightred}}c}{0.9812} & \multicolumn{1}{>{\columncolor{lightred}}c}{0.9529} & \multicolumn{1}{>{\columncolor{lightred}}c}{0.9869}& \multicolumn{1}{>{\columncolor{lightred}}c}{0.9874}      
    &\multicolumn{1}{>{\columncolor{lightred}}c}{0.9775} & \multicolumn{1}{>{\columncolor{lightred}}c}{0.9133} & \multicolumn{1}{>{\columncolor{lightred}}c}{0.9185} \\

    &\multicolumn{1}{>{\columncolor{lightblue}}c|}{\textbf{ TelePiT }}  
    &\multicolumn{1}{>{\columncolor{lightblue}}c}{\textbf{0.9840}} & \multicolumn{1}{>{\columncolor{lightblue}}c}{\textbf{0.9569}} & \multicolumn{1}{>{\columncolor{lightblue}}c}{\textbf{0.9886}}& \multicolumn{1}{>{\columncolor{lightblue}}c}{\textbf{0.9891}}      
    &\multicolumn{1}{>{\columncolor{lightblue}}c}{\textbf{0.9954}} & \multicolumn{1}{>{\columncolor{lightblue}}c}{\textbf{0.9239}} & \multicolumn{1}{>{\columncolor{lightblue}}c|}{\textbf{0.9276}} 
    &\multicolumn{1}{>{\columncolor{lightblue}}c}{\textbf{0.9839}} & \multicolumn{1}{>{\columncolor{lightblue}}c}{\textbf{0.9557}} & \multicolumn{1}{>{\columncolor{lightblue}}c}{\textbf{0.9886}}& \multicolumn{1}{>{\columncolor{lightblue}}c}{\textbf{0.9891}}      
    &\multicolumn{1}{>{\columncolor{lightblue}}c}{\textbf{0.9954}} & \multicolumn{1}{>{\columncolor{lightblue}}c}{\textbf{0.9246}} & \multicolumn{1}{>{\columncolor{lightblue}}c}{\textbf{0.9273}} \\

    \midrule
    \multirow{10}{*}{\rotatebox{90}{\textbf{MS-SSIM ($\uparrow$)}}}
    &\multicolumn{1}{>{\columncolor{lightred}}c|}{\text{ NCEP }}  
    &\multicolumn{1}{>{\columncolor{lightred}}c}{0.8655} & \multicolumn{1}{>{\columncolor{lightred}}c}{0.7921} & \multicolumn{1}{>{\columncolor{lightred}}c}{0.8850}& \multicolumn{1}{>{\columncolor{lightred}}c}{0.9120}      
    &\multicolumn{1}{>{\columncolor{lightred}}c}{--} & \multicolumn{1}{>{\columncolor{lightred}}c}{--} & \multicolumn{1}{>{\columncolor{lightred}}c|}{--} 
    &\multicolumn{1}{>{\columncolor{lightred}}c}{0.8554} & \multicolumn{1}{>{\columncolor{lightred}}c}{0.7793} & \multicolumn{1}{>{\columncolor{lightred}}c}{0.8724}& \multicolumn{1}{>{\columncolor{lightred}}c}{0.9037}      
    &\multicolumn{1}{>{\columncolor{lightred}}c}{--} & \multicolumn{1}{>{\columncolor{lightred}}c}{--} & \multicolumn{1}{>{\columncolor{lightred}}c}{--} \\

    &\multicolumn{1}{>{\columncolor{lightblue}}c|}{\text{ UKMO }}  
    &\multicolumn{1}{>{\columncolor{lightblue}}c}{0.8647} & \multicolumn{1}{>{\columncolor{lightblue}}c}{0.7849} & \multicolumn{1}{>{\columncolor{lightblue}}c}{0.8873}& \multicolumn{1}{>{\columncolor{lightblue}}c}{0.7993}      
    &\multicolumn{1}{>{\columncolor{lightblue}}c}{--} & \multicolumn{1}{>{\columncolor{lightblue}}c}{--} & \multicolumn{1}{>{\columncolor{lightblue}}c|}{--} 
    &\multicolumn{1}{>{\columncolor{lightblue}}c}{0.8517} & \multicolumn{1}{>{\columncolor{lightblue}}c}{0.7652} & \multicolumn{1}{>{\columncolor{lightblue}}c}{0.8739}& \multicolumn{1}{>{\columncolor{lightblue}}c}{0.7914}      
    &\multicolumn{1}{>{\columncolor{lightblue}}c}{--} & \multicolumn{1}{>{\columncolor{lightblue}}c}{--} & \multicolumn{1}{>{\columncolor{lightblue}}c}{--} \\
    
    &\multicolumn{1}{>{\columncolor{lightred}}c|}{\text{ CMA }}  
    &\multicolumn{1}{>{\columncolor{lightred}}c}{0.8569} & \multicolumn{1}{>{\columncolor{lightred}}c}{0.7827} & \multicolumn{1}{>{\columncolor{lightred}}c}{0.8781}& \multicolumn{1}{>{\columncolor{lightred}}c}{0.8984}      
    &\multicolumn{1}{>{\columncolor{lightred}}c}{--} & \multicolumn{1}{>{\columncolor{lightred}}c}{--} & \multicolumn{1}{>{\columncolor{lightred}}c|}{--} 
    &\multicolumn{1}{>{\columncolor{lightred}}c}{0.8452} & \multicolumn{1}{>{\columncolor{lightred}}c}{0.7642} & \multicolumn{1}{>{\columncolor{lightred}}c}{0.8681}& \multicolumn{1}{>{\columncolor{lightred}}c}{0.8922}      
    &\multicolumn{1}{>{\columncolor{lightred}}c}{--} & \multicolumn{1}{>{\columncolor{lightred}}c}{--} & \multicolumn{1}{>{\columncolor{lightred}}c}{--} \\
    
    &\multicolumn{1}{>{\columncolor{lightblue}}c|}{\text{ ECMWF }}  
    &\multicolumn{1}{>{\columncolor{lightblue}}c}{0.8744} & \multicolumn{1}{>{\columncolor{lightblue}}c}{0.8145} & \multicolumn{1}{>{\columncolor{lightblue}}c}{0.8961}& \multicolumn{1}{>{\columncolor{lightblue}}c}{0.9196}      
    &\multicolumn{1}{>{\columncolor{lightblue}}c}{--} & \multicolumn{1}{>{\columncolor{lightblue}}c}{--} & \multicolumn{1}{>{\columncolor{lightblue}}c|}{--} 
    &\multicolumn{1}{>{\columncolor{lightblue}}c}{0.8548} & \multicolumn{1}{>{\columncolor{lightblue}}c}{0.7907} & \multicolumn{1}{>{\columncolor{lightblue}}c}{0.8789}& \multicolumn{1}{>{\columncolor{lightblue}}c}{0.9079}      
    &\multicolumn{1}{>{\columncolor{lightblue}}c}{--} & \multicolumn{1}{>{\columncolor{lightblue}}c}{--} & \multicolumn{1}{>{\columncolor{lightblue}}c}{--} \\

    &\multicolumn{1}{>{\columncolor{lightred}}c|}{\text{ GraphCast }}  
    &\multicolumn{1}{>{\columncolor{lightred}}c}{0.8686} & \multicolumn{1}{>{\columncolor{lightred}}c}{0.7927} & \multicolumn{1}{>{\columncolor{lightred}}c}{0.8915}& \multicolumn{1}{>{\columncolor{lightred}}c}{0.9176}      
    &\multicolumn{1}{>{\columncolor{lightred}}c}{--} & \multicolumn{1}{>{\columncolor{lightred}}c}{--} & \multicolumn{1}{>{\columncolor{lightred}}c|}{--} 
    &\multicolumn{1}{>{\columncolor{lightred}}c}{--} & \multicolumn{1}{>{\columncolor{lightred}}c}{--} & \multicolumn{1}{>{\columncolor{lightred}}c}{--}& \multicolumn{1}{>{\columncolor{lightred}}c}{--}      
    &\multicolumn{1}{>{\columncolor{lightred}}c}{--} & \multicolumn{1}{>{\columncolor{lightred}}c}{--} & \multicolumn{1}{>{\columncolor{lightred}}c}{--} \\

    &\multicolumn{1}{>{\columncolor{lightblue}}c|}{\text{ FCNV2 }}  
    &\multicolumn{1}{>{\columncolor{lightblue}}c}{0.8617} & \multicolumn{1}{>{\columncolor{lightblue}}c}{0.8041} & \multicolumn{1}{>{\columncolor{lightblue}}c}{0.8970}& \multicolumn{1}{>{\columncolor{lightblue}}c}{0.9165}      
    &\multicolumn{1}{>{\columncolor{lightblue}}c}{--} & \multicolumn{1}{>{\columncolor{lightblue}}c}{0.2204} & \multicolumn{1}{>{\columncolor{lightblue}}c|}{0.3258} 
    &\multicolumn{1}{>{\columncolor{lightblue}}c}{0.8505} & \multicolumn{1}{>{\columncolor{lightblue}}c}{0.7833} & \multicolumn{1}{>{\columncolor{lightblue}}c}{0.8844}& \multicolumn{1}{>{\columncolor{lightblue}}c}{0.9064}      
    &\multicolumn{1}{>{\columncolor{lightblue}}c}{--} & \multicolumn{1}{>{\columncolor{lightblue}}c}{0.2162} & \multicolumn{1}{>{\columncolor{lightblue}}c}{0.3195} \\
    
    &\multicolumn{1}{>{\columncolor{lightred}}c|}{\text{ Pangu }}  
    &\multicolumn{1}{>{\columncolor{lightred}}c}{0.8662} & \multicolumn{1}{>{\columncolor{lightred}}c}{0.7829} & \multicolumn{1}{>{\columncolor{lightred}}c}{0.8841}& \multicolumn{1}{>{\columncolor{lightred}}c}{0.9108}      
    &\multicolumn{1}{>{\columncolor{lightred}}c}{--} & \multicolumn{1}{>{\columncolor{lightred}}c}{0.2241} & \multicolumn{1}{>{\columncolor{lightred}}c|}{0.3193} 
    &\multicolumn{1}{>{\columncolor{lightred}}c}{0.8431} & \multicolumn{1}{>{\columncolor{lightred}}c}{0.7466} & \multicolumn{1}{>{\columncolor{lightred}}c}{0.8575}& \multicolumn{1}{>{\columncolor{lightred}}c}{0.8942}      
    &\multicolumn{1}{>{\columncolor{lightred}}c}{--} & \multicolumn{1}{>{\columncolor{lightred}}c}{0.2147} & \multicolumn{1}{>{\columncolor{lightred}}c}{0.3127} \\

    &\multicolumn{1}{>{\columncolor{lightblue}}c|}{\text{ ClimaX }}  
    &\multicolumn{1}{>{\columncolor{lightblue}}c}{0.8660} & \multicolumn{1}{>{\columncolor{lightblue}}c}{0.7890} & \multicolumn{1}{>{\columncolor{lightblue}}c}{0.8782}& \multicolumn{1}{>{\columncolor{lightblue}}c}{0.8948}      
    &\multicolumn{1}{>{\columncolor{lightblue}}c}{0.7193} & \multicolumn{1}{>{\columncolor{lightblue}}c}{0.8714} & \multicolumn{1}{>{\columncolor{lightblue}}c|}{0.8650} 
    &\multicolumn{1}{>{\columncolor{lightblue}}c}{0.8597} & \multicolumn{1}{>{\columncolor{lightblue}}c}{0.7862} & \multicolumn{1}{>{\columncolor{lightblue}}c}{0.8725}& \multicolumn{1}{>{\columncolor{lightblue}}c}{0.8886}      
    &\multicolumn{1}{>{\columncolor{lightblue}}c}{0.7191} & \multicolumn{1}{>{\columncolor{lightblue}}c}{0.8696} & \multicolumn{1}{>{\columncolor{lightblue}}c}{0.8639} \\
    
    &\multicolumn{1}{>{\columncolor{lightred}}c|}{\text{ CirT}}  
    &\multicolumn{1}{>{\columncolor{lightred}}c}{0.9077} & \multicolumn{1}{>{\columncolor{lightred}}c}{0.8633} & \multicolumn{1}{>{\columncolor{lightred}}c}{0.9269}& \multicolumn{1}{>{\columncolor{lightred}}c}{0.9447}      
    &\multicolumn{1}{>{\columncolor{lightred}}c}{0.8159} & \multicolumn{1}{>{\columncolor{lightred}}c}{0.9382} & \multicolumn{1}{>{\columncolor{lightred}}c|}{0.9314} 
    &\multicolumn{1}{>{\columncolor{lightred}}c}{0.9083} & \multicolumn{1}{>{\columncolor{lightred}}c}{0.8637} & \multicolumn{1}{>{\columncolor{lightred}}c}{\textbf{0.9273}}& \multicolumn{1}{>{\columncolor{lightred}}c}{0.9438}      
    &\multicolumn{1}{>{\columncolor{lightred}}c}{0.8137} & \multicolumn{1}{>{\columncolor{lightred}}c}{0.9390} & \multicolumn{1}{>{\columncolor{lightred}}c}{0.9329} \\

    &\multicolumn{1}{>{\columncolor{lightblue}}c|}{\textbf{ TelePiT }}  
    &\multicolumn{1}{>{\columncolor{lightblue}}c}{\textbf{0.9093}} & \multicolumn{1}{>{\columncolor{lightblue}}c}{\textbf{0.8665}} & \multicolumn{1}{>{\columncolor{lightblue}}c}{\textbf{0.9276}}& \multicolumn{1}{>{\columncolor{lightblue}}c}{\textbf{0.9459}}      
    &\multicolumn{1}{>{\columncolor{lightblue}}c}{\textbf{0.8494}} & \multicolumn{1}{>{\columncolor{lightblue}}c}{\textbf{0.9459}} & \multicolumn{1}{>{\columncolor{lightblue}}c|}{\textbf{0.9390}} 
    &\multicolumn{1}{>{\columncolor{lightblue}}c}{\textbf{0.9101}} & \multicolumn{1}{>{\columncolor{lightblue}}c}{\textbf{0.8660}} & \multicolumn{1}{>{\columncolor{lightblue}}c}{\text{0.9271}}& \multicolumn{1}{>{\columncolor{lightblue}}c}{\textbf{0.9454}}      
    &\multicolumn{1}{>{\columncolor{lightblue}}c}{\textbf{0.8494}} & \multicolumn{1}{>{\columncolor{lightblue}}c}{\textbf{0.9454}} & \multicolumn{1}{>{\columncolor{lightblue}}c}{\textbf{0.9396}} \\

    \midrule
    \multirow{10}{*}{\rotatebox{90}{\textbf{SpecDiv ($\downarrow$)}}}
    &\multicolumn{1}{>{\columncolor{lightred}}c|}{\text{ NCEP }}  
    &\multicolumn{1}{>{\columncolor{lightred}}c}{0.5451} & \multicolumn{1}{>{\columncolor{lightred}}c}{0.4890} & \multicolumn{1}{>{\columncolor{lightred}}c}{4.3145}& \multicolumn{1}{>{\columncolor{lightred}}c}{0.6117}      
    &\multicolumn{1}{>{\columncolor{lightred}}c}{--} & \multicolumn{1}{>{\columncolor{lightred}}c}{--} & \multicolumn{1}{>{\columncolor{lightred}}c|}{--} 
    &\multicolumn{1}{>{\columncolor{lightred}}c}{0.5486} & \multicolumn{1}{>{\columncolor{lightred}}c}{0.4704} & \multicolumn{1}{>{\columncolor{lightred}}c}{4.2580}& \multicolumn{1}{>{\columncolor{lightred}}c}{0.5819}      
    &\multicolumn{1}{>{\columncolor{lightred}}c}{--} & \multicolumn{1}{>{\columncolor{lightred}}c}{--} & \multicolumn{1}{>{\columncolor{lightred}}c}{--} \\

    &\multicolumn{1}{>{\columncolor{lightblue}}c|}{\text{ UKMO }}  
    &\multicolumn{1}{>{\columncolor{lightblue}}c}{0.1281} & \multicolumn{1}{>{\columncolor{lightblue}}c}{0.1151} & \multicolumn{1}{>{\columncolor{lightblue}}c}{0.2576}& \multicolumn{1}{>{\columncolor{lightblue}}c}{0.1022}      
    &\multicolumn{1}{>{\columncolor{lightblue}}c}{--} & \multicolumn{1}{>{\columncolor{lightblue}}c}{--} & \multicolumn{1}{>{\columncolor{lightblue}}c|}{--} 
    &\multicolumn{1}{>{\columncolor{lightblue}}c}{0.1194} & \multicolumn{1}{>{\columncolor{lightblue}}c}{0.1220} & \multicolumn{1}{>{\columncolor{lightblue}}c}{0.2465}& \multicolumn{1}{>{\columncolor{lightblue}}c}{0.0966}      
    &\multicolumn{1}{>{\columncolor{lightblue}}c}{--} & \multicolumn{1}{>{\columncolor{lightblue}}c}{--} & \multicolumn{1}{>{\columncolor{lightblue}}c}{--} \\
    
    &\multicolumn{1}{>{\columncolor{lightred}}c|}{\text{ CMA }}  
    &\multicolumn{1}{>{\columncolor{lightred}}c}{0.1655} & \multicolumn{1}{>{\columncolor{lightred}}c}{0.1787} & \multicolumn{1}{>{\columncolor{lightred}}c}{0.2090}& \multicolumn{1}{>{\columncolor{lightred}}c}{0.0194}      
    &\multicolumn{1}{>{\columncolor{lightred}}c}{--} & \multicolumn{1}{>{\columncolor{lightred}}c}{--} & \multicolumn{1}{>{\columncolor{lightred}}c|}{--} 
    &\multicolumn{1}{>{\columncolor{lightred}}c}{0.1678} & \multicolumn{1}{>{\columncolor{lightred}}c}{0.1591} & \multicolumn{1}{>{\columncolor{lightred}}c}{0.1346}& \multicolumn{1}{>{\columncolor{lightred}}c}{0.0258}      
    &\multicolumn{1}{>{\columncolor{lightred}}c}{--} & \multicolumn{1}{>{\columncolor{lightred}}c}{--} & \multicolumn{1}{>{\columncolor{lightred}}c}{--} \\
    
    &\multicolumn{1}{>{\columncolor{lightblue}}c|}{\text{ ECMWF }}  
    &\multicolumn{1}{>{\columncolor{lightblue}}c}{0.6792} & \multicolumn{1}{>{\columncolor{lightblue}}c}{0.0719} & \multicolumn{1}{>{\columncolor{lightblue}}c}{0.1650}& \multicolumn{1}{>{\columncolor{lightblue}}c}{0.3280}      
    &\multicolumn{1}{>{\columncolor{lightblue}}c}{--} & \multicolumn{1}{>{\columncolor{lightblue}}c}{--} & \multicolumn{1}{>{\columncolor{lightblue}}c|}{--} 
    &\multicolumn{1}{>{\columncolor{lightblue}}c}{0.2564} & \multicolumn{1}{>{\columncolor{lightblue}}c}{0.0619} & \multicolumn{1}{>{\columncolor{lightblue}}c}{0.1250}& \multicolumn{1}{>{\columncolor{lightblue}}c}{0.1378}      
    &\multicolumn{1}{>{\columncolor{lightblue}}c}{--} & \multicolumn{1}{>{\columncolor{lightblue}}c}{--} & \multicolumn{1}{>{\columncolor{lightblue}}c}{--} \\

    &\multicolumn{1}{>{\columncolor{lightred}}c|}{\text{ GraphCast }}  
    &\multicolumn{1}{>{\columncolor{lightred}}c}{0.0819} & \multicolumn{1}{>{\columncolor{lightred}}c}{0.3114} & \multicolumn{1}{>{\columncolor{lightred}}c}{0.0671}& \multicolumn{1}{>{\columncolor{lightred}}c}{0.0631}      
    &\multicolumn{1}{>{\columncolor{lightred}}c}{--} & \multicolumn{1}{>{\columncolor{lightred}}c}{--} & \multicolumn{1}{>{\columncolor{lightred}}c|}{--} 
    &\multicolumn{1}{>{\columncolor{lightred}}c}{--} & \multicolumn{1}{>{\columncolor{lightred}}c}{--} & \multicolumn{1}{>{\columncolor{lightred}}c}{--}& \multicolumn{1}{>{\columncolor{lightred}}c}{--}      
    &\multicolumn{1}{>{\columncolor{lightred}}c}{--} & \multicolumn{1}{>{\columncolor{lightred}}c}{--} & \multicolumn{1}{>{\columncolor{lightred}}c}{--} \\

    &\multicolumn{1}{>{\columncolor{lightblue}}c|}{\text{ FCNV2 }}  
    &\multicolumn{1}{>{\columncolor{lightblue}}c}{0.1288} & \multicolumn{1}{>{\columncolor{lightblue}}c}{0.2452} & \multicolumn{1}{>{\columncolor{lightblue}}c}{0.4943}& \multicolumn{1}{>{\columncolor{lightblue}}c}{0.3097}      
    &\multicolumn{1}{>{\columncolor{lightblue}}c}{--} & \multicolumn{1}{>{\columncolor{lightblue}}c}{0.1986} & \multicolumn{1}{>{\columncolor{lightblue}}c|}{0.4253} 
    &\multicolumn{1}{>{\columncolor{lightblue}}c}{0.1346} & \multicolumn{1}{>{\columncolor{lightblue}}c}{0.2757} & \multicolumn{1}{>{\columncolor{lightblue}}c}{0.5108}& \multicolumn{1}{>{\columncolor{lightblue}}c}{0.3141}      
    &\multicolumn{1}{>{\columncolor{lightblue}}c}{--} & \multicolumn{1}{>{\columncolor{lightblue}}c}{0.1735} & \multicolumn{1}{>{\columncolor{lightblue}}c}{0.4174} \\
    
    &\multicolumn{1}{>{\columncolor{lightred}}c|}{\text{ Pangu }}  
    &\multicolumn{1}{>{\columncolor{lightred}}c}{0.6510} & \multicolumn{1}{>{\columncolor{lightred}}c}{0.1458} & \multicolumn{1}{>{\columncolor{lightred}}c}{0.2283}& \multicolumn{1}{>{\columncolor{lightred}}c}{0.2155}      
    &\multicolumn{1}{>{\columncolor{lightred}}c}{--} & \multicolumn{1}{>{\columncolor{lightred}}c}{0.0557} & \multicolumn{1}{>{\columncolor{lightred}}c|}{0.2216} 
    &\multicolumn{1}{>{\columncolor{lightred}}c}{0.6812} & \multicolumn{1}{>{\columncolor{lightred}}c}{0.1448} & \multicolumn{1}{>{\columncolor{lightred}}c}{0.2486}& \multicolumn{1}{>{\columncolor{lightred}}c}{0.2381}      
    &\multicolumn{1}{>{\columncolor{lightred}}c}{--} & \multicolumn{1}{>{\columncolor{lightred}}c}{0.0574} & \multicolumn{1}{>{\columncolor{lightred}}c}{0.2138} \\

    &\multicolumn{1}{>{\columncolor{lightblue}}c|}{\text{ ClimaX }}  
    &\multicolumn{1}{>{\columncolor{lightblue}}c}{0.4051} & \multicolumn{1}{>{\columncolor{lightblue}}c}{0.0317} & \multicolumn{1}{>{\columncolor{lightblue}}c}{0.3999}& \multicolumn{1}{>{\columncolor{lightblue}}c}{0.1342}      
    &\multicolumn{1}{>{\columncolor{lightblue}}c}{0.4534} & \multicolumn{1}{>{\columncolor{lightblue}}c}{0.2494} & \multicolumn{1}{>{\columncolor{lightblue}}c|}{0.0935} 
    &\multicolumn{1}{>{\columncolor{lightblue}}c}{0.3881} & \multicolumn{1}{>{\columncolor{lightblue}}c}{0.0666} & \multicolumn{1}{>{\columncolor{lightblue}}c}{0.2905}& \multicolumn{1}{>{\columncolor{lightblue}}c}{0.1370}      
    &\multicolumn{1}{>{\columncolor{lightblue}}c}{0.4558} & \multicolumn{1}{>{\columncolor{lightblue}}c}{0.2683} & \multicolumn{1}{>{\columncolor{lightblue}}c}{0.1209} \\
    
    &\multicolumn{1}{>{\columncolor{lightred}}c|}{\text{ CirT}}  
    &\multicolumn{1}{>{\columncolor{lightred}}c}{0.0404} & \multicolumn{1}{>{\columncolor{lightred}}c}{0.1356} & \multicolumn{1}{>{\columncolor{lightred}}c}{0.0884}& \multicolumn{1}{>{\columncolor{lightred}}c}{0.0295}      
    &\multicolumn{1}{>{\columncolor{lightred}}c}{0.0952} & \multicolumn{1}{>{\columncolor{lightred}}c}{0.0404} & \multicolumn{1}{>{\columncolor{lightred}}c|}{0.0599} 
    &\multicolumn{1}{>{\columncolor{lightred}}c}{0.0592} & \multicolumn{1}{>{\columncolor{lightred}}c}{0.2005} & \multicolumn{1}{>{\columncolor{lightred}}c}{0.2462}& \multicolumn{1}{>{\columncolor{lightred}}c}{0.0189}      
    &\multicolumn{1}{>{\columncolor{lightred}}c}{0.0902} & \multicolumn{1}{>{\columncolor{lightred}}c}{0.0398} & \multicolumn{1}{>{\columncolor{lightred}}c}{0.0583} \\

    &\multicolumn{1}{>{\columncolor{lightblue}}c|}{\textbf{ TelePiT }}  
    &\multicolumn{1}{>{\columncolor{lightblue}}c}{\textbf{0.0315}} & \multicolumn{1}{>{\columncolor{lightblue}}c}{\textbf{0.0246}} & \multicolumn{1}{>{\columncolor{lightblue}}c}{\textbf{0.0195}}& \multicolumn{1}{>{\columncolor{lightblue}}c}{\textbf{0.0047}}      
    &\multicolumn{1}{>{\columncolor{lightblue}}c}{\textbf{0.0019}} & \multicolumn{1}{>{\columncolor{lightblue}}c}{\textbf{0.0015}} & \multicolumn{1}{>{\columncolor{lightblue}}c|}{\textbf{0.0072}} 
    &\multicolumn{1}{>{\columncolor{lightblue}}c}{\textbf{0.0295}} & \multicolumn{1}{>{\columncolor{lightblue}}c}{\textbf{0.0212}} & \multicolumn{1}{>{\columncolor{lightblue}}c}{\textbf{0.0191}}& \multicolumn{1}{>{\columncolor{lightblue}}c}{\textbf{0.0041}}      
    &\multicolumn{1}{>{\columncolor{lightblue}}c}{\textbf{0.0019}} & \multicolumn{1}{>{\columncolor{lightblue}}c}{\textbf{0.0021}} & \multicolumn{1}{>{\columncolor{lightblue}}c}{\textbf{0.0069}} \\

    \midrule
    \multirow{10}{*}{\rotatebox{90}{\textbf{SpecRec ($\downarrow$)}}}
    &\multicolumn{1}{>{\columncolor{lightred}}c|}{\text{ NCEP }}  
    &\multicolumn{1}{>{\columncolor{lightred}}c}{0.0670} & \multicolumn{1}{>{\columncolor{lightred}}c}{0.0617} & \multicolumn{1}{>{\columncolor{lightred}}c}{0.2431}& \multicolumn{1}{>{\columncolor{lightred}}c}{0.1405}      
    &\multicolumn{1}{>{\columncolor{lightred}}c}{--} & \multicolumn{1}{>{\columncolor{lightred}}c}{--} & \multicolumn{1}{>{\columncolor{lightred}}c|}{--} 
    &\multicolumn{1}{>{\columncolor{lightred}}c}{0.0672} & \multicolumn{1}{>{\columncolor{lightred}}c}{0.0604} & \multicolumn{1}{>{\columncolor{lightred}}c}{0.2426}& \multicolumn{1}{>{\columncolor{lightred}}c}{0.1369}      
    &\multicolumn{1}{>{\columncolor{lightred}}c}{--} & \multicolumn{1}{>{\columncolor{lightred}}c}{--} & \multicolumn{1}{>{\columncolor{lightred}}c}{--} \\

    &\multicolumn{1}{>{\columncolor{lightblue}}c|}{\text{ UKMO }}  
    &\multicolumn{1}{>{\columncolor{lightblue}}c}{0.0410} & \multicolumn{1}{>{\columncolor{lightblue}}c}{0.0382} & \multicolumn{1}{>{\columncolor{lightblue}}c}{0.0486}& \multicolumn{1}{>{\columncolor{lightblue}}c}{0.0308}      
    &\multicolumn{1}{>{\columncolor{lightblue}}c}{--} & \multicolumn{1}{>{\columncolor{lightblue}}c}{--} & \multicolumn{1}{>{\columncolor{lightblue}}c|}{--} 
    &\multicolumn{1}{>{\columncolor{lightblue}}c}{0.0399} & \multicolumn{1}{>{\columncolor{lightblue}}c}{0.0393} & \multicolumn{1}{>{\columncolor{lightblue}}c}{0.0471}& \multicolumn{1}{>{\columncolor{lightblue}}c}{0.0304}      
    &\multicolumn{1}{>{\columncolor{lightblue}}c}{--} & \multicolumn{1}{>{\columncolor{lightblue}}c}{--} & \multicolumn{1}{>{\columncolor{lightblue}}c}{--} \\
    
    &\multicolumn{1}{>{\columncolor{lightred}}c|}{\text{ CMA }}  
    &\multicolumn{1}{>{\columncolor{lightred}}c}{0.0393} & \multicolumn{1}{>{\columncolor{lightred}}c}{0.0517} & \multicolumn{1}{>{\columncolor{lightred}}c}{0.0487}& \multicolumn{1}{>{\columncolor{lightred}}c}{0.0155}      
    &\multicolumn{1}{>{\columncolor{lightred}}c}{--} & \multicolumn{1}{>{\columncolor{lightred}}c}{--} & \multicolumn{1}{>{\columncolor{lightred}}c|}{--} 
    &\multicolumn{1}{>{\columncolor{lightred}}c}{0.0395} & \multicolumn{1}{>{\columncolor{lightred}}c}{0.0501} & \multicolumn{1}{>{\columncolor{lightred}}c}{0.0388}& \multicolumn{1}{>{\columncolor{lightred}}c}{0.0183}      
    &\multicolumn{1}{>{\columncolor{lightred}}c}{--} & \multicolumn{1}{>{\columncolor{lightred}}c}{--} & \multicolumn{1}{>{\columncolor{lightred}}c}{--} \\
    
    &\multicolumn{1}{>{\columncolor{lightblue}}c|}{\text{ ECMWF }}  
    &\multicolumn{1}{>{\columncolor{lightblue}}c}{0.1266} & \multicolumn{1}{>{\columncolor{lightblue}}c}{0.0266} & \multicolumn{1}{>{\columncolor{lightblue}}c}{0.0458}& \multicolumn{1}{>{\columncolor{lightblue}}c}{0.0791}      
    &\multicolumn{1}{>{\columncolor{lightblue}}c}{--} & \multicolumn{1}{>{\columncolor{lightblue}}c}{--} & \multicolumn{1}{>{\columncolor{lightblue}}c|}{--} 
    &\multicolumn{1}{>{\columncolor{lightblue}}c}{0.0532} & \multicolumn{1}{>{\columncolor{lightblue}}c}{0.0248} & \multicolumn{1}{>{\columncolor{lightblue}}c}{0.0461}& \multicolumn{1}{>{\columncolor{lightblue}}c}{0.0393}      
    &\multicolumn{1}{>{\columncolor{lightblue}}c}{--} & \multicolumn{1}{>{\columncolor{lightblue}}c}{--} & \multicolumn{1}{>{\columncolor{lightblue}}c}{--} \\

    &\multicolumn{1}{>{\columncolor{lightred}}c|}{\text{ GraphCast }}  
    &\multicolumn{1}{>{\columncolor{lightred}}c}{0.0281} & \multicolumn{1}{>{\columncolor{lightred}}c}{0.0553} & \multicolumn{1}{>{\columncolor{lightred}}c}{0.0260}& \multicolumn{1}{>{\columncolor{lightred}}c}{0.0269}      
    &\multicolumn{1}{>{\columncolor{lightred}}c}{--} & \multicolumn{1}{>{\columncolor{lightred}}c}{--} & \multicolumn{1}{>{\columncolor{lightred}}c|}{--} 
    &\multicolumn{1}{>{\columncolor{lightred}}c}{--} & \multicolumn{1}{>{\columncolor{lightred}}c}{--} & \multicolumn{1}{>{\columncolor{lightred}}c}{--}& \multicolumn{1}{>{\columncolor{lightred}}c}{--}      
    &\multicolumn{1}{>{\columncolor{lightred}}c}{--} & \multicolumn{1}{>{\columncolor{lightred}}c}{--} & \multicolumn{1}{>{\columncolor{lightred}}c}{--} \\

    &\multicolumn{1}{>{\columncolor{lightblue}}c|}{\text{ FCNV2 }}  
    &\multicolumn{1}{>{\columncolor{lightblue}}c}{0.0410} & \multicolumn{1}{>{\columncolor{lightblue}}c}{0.0555} & \multicolumn{1}{>{\columncolor{lightblue}}c}{0.0760}& \multicolumn{1}{>{\columncolor{lightblue}}c}{0.0645}      
    &\multicolumn{1}{>{\columncolor{lightblue}}c}{--} & \multicolumn{1}{>{\columncolor{lightblue}}c}{0.0468} & \multicolumn{1}{>{\columncolor{lightblue}}c|}{0.0719} 
    &\multicolumn{1}{>{\columncolor{lightblue}}c}{0.0425} & \multicolumn{1}{>{\columncolor{lightblue}}c}{0.0576} & \multicolumn{1}{>{\columncolor{lightblue}}c}{0.0753}& \multicolumn{1}{>{\columncolor{lightblue}}c}{0.0643}      
    &\multicolumn{1}{>{\columncolor{lightblue}}c}{--} & \multicolumn{1}{>{\columncolor{lightblue}}c}{0.0445} & \multicolumn{1}{>{\columncolor{lightblue}}c}{0.0708} \\
    
    &\multicolumn{1}{>{\columncolor{lightred}}c|}{\text{ Pangu }}  
    &\multicolumn{1}{>{\columncolor{lightred}}c}{0.1129} & \multicolumn{1}{>{\columncolor{lightred}}c}{0.0399} & \multicolumn{1}{>{\columncolor{lightred}}c}{0.0534}& \multicolumn{1}{>{\columncolor{lightred}}c}{0.0521}      
    &\multicolumn{1}{>{\columncolor{lightred}}c}{--} & \multicolumn{1}{>{\columncolor{lightred}}c}{0.0255} & \multicolumn{1}{>{\columncolor{lightred}}c|}{0.0565} 
    &\multicolumn{1}{>{\columncolor{lightred}}c}{0.1196} & \multicolumn{1}{>{\columncolor{lightred}}c}{0.0395} & \multicolumn{1}{>{\columncolor{lightred}}c}{0.0567}& \multicolumn{1}{>{\columncolor{lightred}}c}{0.0558}      
    &\multicolumn{1}{>{\columncolor{lightred}}c}{--} & \multicolumn{1}{>{\columncolor{lightred}}c}{0.0262} & \multicolumn{1}{>{\columncolor{lightred}}c}{0.0559} \\

    &\multicolumn{1}{>{\columncolor{lightblue}}c|}{\text{ ClimaX }}  
    &\multicolumn{1}{>{\columncolor{lightblue}}c}{0.0862} & \multicolumn{1}{>{\columncolor{lightblue}}c}{0.0194} & \multicolumn{1}{>{\columncolor{lightblue}}c}{0.1033}& \multicolumn{1}{>{\columncolor{lightblue}}c}{0.0528}      
    &\multicolumn{1}{>{\columncolor{lightblue}}c}{0.0794} & \multicolumn{1}{>{\columncolor{lightblue}}c}{0.0603} & \multicolumn{1}{>{\columncolor{lightblue}}c|}{0.0336} 
    &\multicolumn{1}{>{\columncolor{lightblue}}c}{0.0860} & \multicolumn{1}{>{\columncolor{lightblue}}c}{0.0346} & \multicolumn{1}{>{\columncolor{lightblue}}c}{0.0790}& \multicolumn{1}{>{\columncolor{lightblue}}c}{0.0523}      
    &\multicolumn{1}{>{\columncolor{lightblue}}c}{0.0796} & \multicolumn{1}{>{\columncolor{lightblue}}c}{0.0628} & \multicolumn{1}{>{\columncolor{lightblue}}c}{0.0406} \\
    
    &\multicolumn{1}{>{\columncolor{lightred}}c|}{\text{ CirT}}  
    &\multicolumn{1}{>{\columncolor{lightred}}c}{0.0189} & \multicolumn{1}{>{\columncolor{lightred}}c}{0.0320} & \multicolumn{1}{>{\columncolor{lightred}}c}{0.0290}& \multicolumn{1}{>{\columncolor{lightred}}c}{0.0175}      
    &\multicolumn{1}{>{\columncolor{lightred}}c}{0.0339} & \multicolumn{1}{>{\columncolor{lightred}}c}{0.0205} & \multicolumn{1}{>{\columncolor{lightred}}c|}{0.0233} 
    &\multicolumn{1}{>{\columncolor{lightred}}c}{0.0241} & \multicolumn{1}{>{\columncolor{lightred}}c}{0.0361} & \multicolumn{1}{>{\columncolor{lightred}}c}{0.0445}& \multicolumn{1}{>{\columncolor{lightred}}c}{0.0141}      
    &\multicolumn{1}{>{\columncolor{lightred}}c}{0.0330} & \multicolumn{1}{>{\columncolor{lightred}}c}{0.0202} & \multicolumn{1}{>{\columncolor{lightred}}c}{0.0231} \\

    &\multicolumn{1}{>{\columncolor{lightblue}}c|}{\textbf{ TelePiT }}  
    &\multicolumn{1}{>{\columncolor{lightblue}}c}{\textbf{0.0157}} & \multicolumn{1}{>{\columncolor{lightblue}}c}{\textbf{0.0159}} & \multicolumn{1}{>{\columncolor{lightblue}}c}{\textbf{0.0162}}& \multicolumn{1}{>{\columncolor{lightblue}}c}{\textbf{0.0073}}      
    &\multicolumn{1}{>{\columncolor{lightblue}}c}{\textbf{0.0044}} & \multicolumn{1}{>{\columncolor{lightblue}}c}{\textbf{0.0039}} & \multicolumn{1}{>{\columncolor{lightblue}}c|}{\textbf{0.0077}} 
    &\multicolumn{1}{>{\columncolor{lightblue}}c}{\textbf{0.0153}} & \multicolumn{1}{>{\columncolor{lightblue}}c}{\textbf{0.0147}} & \multicolumn{1}{>{\columncolor{lightblue}}c}{\textbf{0.0162}}& \multicolumn{1}{>{\columncolor{lightblue}}c}{\textbf{0.0069}}      
    &\multicolumn{1}{>{\columncolor{lightblue}}c}{\textbf{0.0044}} & \multicolumn{1}{>{\columncolor{lightblue}}c}{\textbf{0.0044}} & \multicolumn{1}{>{\columncolor{lightblue}}c}{\textbf{0.0078}} \\
    
    \bottomrule
  \end{tabular}
\let\everydisplay\oldeverydisplay
\let\everymath\oldeverymath
\end{table*}

\begin{figure*}[t]
    \centering
    \includegraphics[width=1\linewidth]{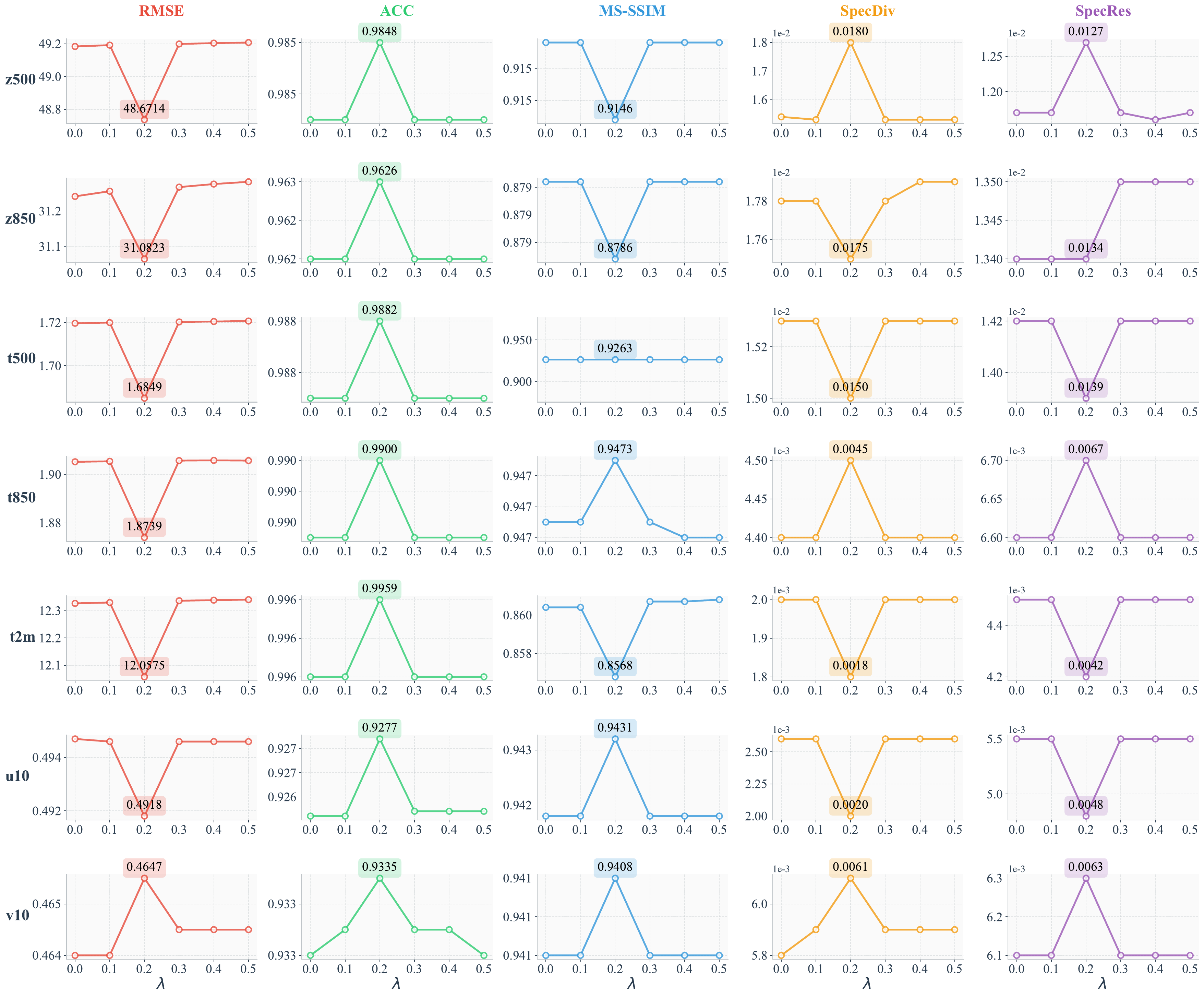}
    \caption{Parameter sensitivity analysis of teleconnection coefficient $\lambda$ across different key variables and all evaluation metrics on Weeks 3-4.
    The parameter $\lambda$ controls the influence of teleconnection patterns in the attention mechanism (Equation 8). Most variables achieve optimal performance at $\lambda$ = 0.2, indicating that moderate incorporation of atmospheric teleconnection patterns enhances model performance. RMSE values generally reach minimum at $\lambda$ = 0.2, while ACC values peak at the same point, suggesting an optimal balance between local atmospheric dynamics and large-scale teleconnection patterns for S2S forecasting accuracy.}
    \label{appendix_fig:lamda}
\end{figure*}

\vspace{0.2cm}



\end{document}